%% file: main.tex
\documentclass[final,12pt]{styles/arxiv/colt2021}

\input{header}
\title[Policy Optimization for Stochastic Shortest Path]{Policy Optimization for Stochastic Shortest Path}
\usepackage{times}
\coltauthor{%
 \Name{Liyu Chen} \Email{liyuc@usc.edu}\\
 \addr University of Southern California
 \AND
 \Name{Haipeng Luo} \Email{haipengl@usc.edu}\\
 \addr University of Southern California
 \AND
 \Name{Aviv Rosenberg} \Email{avivros007@gmail.com}\\
 \addr Tel-Aviv University
}

\begin{document}
\SetAlgoVlined
\DontPrintSemicolon
\maketitle

\begin{abstract}
Policy optimization is among the most popular and successful reinforcement learning algorithms, and there is increasing interest in understanding its theoretical guarantees.
In this work, we initiate the study of policy optimization for the stochastic shortest path (SSP) problem, a goal-oriented reinforcement learning model that strictly generalizes the finite-horizon model and better captures many applications.
We consider a wide range of settings, including stochastic and adversarial environments under full information or bandit feedback,
and propose a policy optimization algorithm for each setting that makes use of novel correction terms and/or variants of dilated bonuses~\citep{luo2021policy}.
For most settings, our algorithm is shown to achieve a near-optimal regret bound.

One key technical contribution of this work is a new approximation scheme to tackle SSP problems that we call \textit{stacked discounted approximation} and use in all our proposed algorithms.
Unlike the finite-horizon approximation that is heavily used in recent SSP algorithms,
our new approximation enables us to learn a near-stationary policy with only logarithmic changes during an episode and could lead to an exponential improvement in space complexity.
	
%
\end{abstract}

\tolerance 1414
\hbadness 1414
\emergencystretch 1.5em
\hfuzz 0.3pt
\widowpenalty=10000
\vfuzz \hfuzz
\raggedbottom

\sloppy

\section{Introduction}
\input{intro}

\section{Preliminaries}
\label{sec:pre}
\input{pre}

\section{Stacked Discounted Approximation and Algorithm Template}
\label{sec:sda}
\input{sda}
\section{Algorithms and Results for Stochastic Environments}
\label{sec:sto}
\input{sto}

\section{Algorithms and Results for Adversarial Environments}
\label{sec:adv}
\input{adv}

\section{Conclusion}
Our work initiates the study of policy optimization for SSP and systematically develops a set of novel algorithms suitable for different settings.
Many questions remain open, such as closing the gap between some of our results and the best known results achieved by other types of methods. 
Moreover, as mentioned, one of the reasons to study PO for SSP is that PO usually works well when combined with function approximation.
Although our work is only for the tabular setting, we believe that our results lay a solid foundation for future studies on SSP with function approximation.

\paragraph{Acknowledgements.}
HL is supported by NSF Award IIS-1943607 and a Google Faculty Research Award.



\bibliography{ref}
\newpage

\appendix

\section{Preliminary for Appendix}
\label{app:pre}
\input{app-pre}

\section{Omitted Details for \pref{sec:sda}}
\label{app:sda}
\input{app-sda}

\section{Omitted Details for \pref{sec:sto}}
\input{app-sto}

\section{Omitted Details for \pref{sec:adv}}
\input{app-adv}

\section{Learning without Some Parameters}
\label{app:pf}
\input{app-pf}

\section{Auxiliary Lemma}
\input{auxlm}

\end{document}

%% file: header.tex
\usepackage{times}
\usepackage{booktabs}       
\usepackage{amsfonts}       
\usepackage{nicefrac}       
\usepackage{enumitem}
\usepackage{nicefrac}
\PassOptionsToPackage{linktocpage}{hyperref} 

\usepackage{amsmath,amssymb,mathrsfs}

\usepackage{footnote}
\usepackage{algorithm}
\usepackage{algorithmic}

\allowdisplaybreaks

\makeatletter
\def\set@curr@file#1{\def\@curr@file{#1}} 
\makeatother

\usepackage{graphicx}
\usepackage{mathtools}
\usepackage{footnote}
\usepackage{float}
\usepackage{xspace}
\usepackage{multirow}
\usepackage{wrapfig}
\usepackage{framed}
\usepackage{bbm}
\usepackage{footnote}
\usepackage{nicefrac}
\usepackage{makecell}
\usepackage[T1]{fontenc}
\makesavenoteenv{tabular}
\makesavenoteenv{table}

\definecolor{Green}{rgb}{0.13, 0.65, 0.3}
\usepackage[]{color-edits}

\addauthor{HL}{red}
\addauthor{LC}{cyan}

\newcommand{\calA}{{\mathcal{A}}}
\newcommand{\calB}{{\mathcal{B}}}
\newcommand{\calC}{\mathcal{C}}
\newcommand{\calX}{{\mathcal{X}}}
\newcommand{\calS}{{\mathcal{S}}}
\newcommand{\calF}{{\mathcal{F}}}

\newcommand{\calT}{{\mathcal{T}}}
\newcommand{\calP}{{\mathcal{P}}}

\newcommand{\calM}{{\mathcal{M}}}

\newcommand{\KL}{\text{\rm KL}}
\newcommand{\V}{\text{\rm Var}}

\DeclareMathOperator*{\argmin}{argmin}
\DeclareMathOperator*{\argmax}{argmax}

\newcommand{\eat}[1]{}

\newcommand{\inner}[2]{\left\langle #1,#2 \right\rangle}
\newcommand{\inners}[2]{\langle #1,#2\rangle}

\newcommand{\rbr}[1]{\left(#1\right)}
\newcommand{\sbr}[1]{\left[#1\right]}
\newcommand{\cbr}[1]{\left\{#1\right\}}

\newcommand{\abr}[1]{\left|#1\right|}
\newcommand{\bigO}[1]{\order\left( #1 \right)}
\newcommand{\tilO}[1]{\otil\left( #1 \right)}

\newcommand{\bigo}[1]{\order( #1 )}
\newcommand{\tilo}[1]{\otil( #1 )}
\newcommand{\lowo}[1]{\lorder( #1 )}
\DeclarePairedDelimiter\ceil{\lceil}{\rceil}

\renewcommand{\P}{\bar{P}}

\newcommand{\Np}{N^+}

\newcommand{\sumi}{\sum_{i=1}^{I_k}}
\newcommand{\sumit}{\sum_{i=1}^{J_k}}
\newcommand{\sumitp}{\sum_{i=1}^{J_k+1}}
\newcommand{\sumib}{\sum_{i=1}^{\J_k}}
\newcommand{\hatQ}{\widehat{Q}}

\newcommand{\tilQ}{\widetilde{Q}}
\newcommand{\tilV}{\widetilde{V}}

\newcommand{\Tmax}{\ensuremath{T_{\max}}}

\newcommand{\T}{\ensuremath{T_\star}}
\newcommand{\cmin}{\ensuremath{c_{\min}}}

\newcommand{\reg}{\textsc{Reg}}
\newcommand{\bias}{\textsc{Bias}}

\newcommand{\SA}{\calS\times\calA}

\newcommand{\rs}{\mathring{s}}
\newcommand{\sinit}{s_{\text{\rm init}}}
\newcommand{\frn}{\mathfrak{n}}
\newcommand{\frN}{\mathfrak{N}}
\newcommand{\frNp}{\frN^+}
\newcommand{\barc}{\bar{c}}

\newcommand{\B}{B_{\star}}
\newcommand{\tilh}{\widetilde{h}}
\newcommand{\ux}{\overline{x}}
\newcommand{\lx}{\underline{x}}
\newcommand{\uq}{\overline{q}}
\renewcommand{\lq}{\underline{q}}
\newcommand{\N}{N} 
\newcommand{\M}{M}
\newcommand{\Mp}{\M^+}
\newcommand{\J}{\bar{J}}

\newcommand{\barP}{\bar{P}}

\newcommand{\epsilonstar}{\epsilon^{\star}}
\newcommand{\barQ}{\bar{Q}}
\newcommand{\rcalS}{\mathring{\calS}}
\newcommand{\rcalM}{\mathring{\calM}}
\newcommand{\rsinit}{\mathring{s}_{\text{init}}}
\newcommand{\rP}{\mathring{P}}
\newcommand{\rc}{\mathring{c}}
\newcommand{\roptpi}{\mathring{\pi}^{\star}}
\newcommand{\rR}{\mathring{R}}

\newcommand{\suma}{\sum_{a\in\calA}}

\newcommand{\sumsa}[1][s, a]{\sum_{#1}}

\newcommand{\sumh}{\sum_{h=1}^H}

\newcommand{\sumk}{\sum_{k=1}^K}

\newcommand{\hatc}{\widehat{c}}

\newcommand{\hatq}{\widehat{q}}
\newcommand{\optpi}{\pi^\star}
\newcommand{\optq}{q^{\star}}

\newcommand{\hatP}{\widehat{P}}
\newcommand{\tils}{\widetilde{s}}
\newcommand{\tila}{\widetilde{a}}
\newcommand{\tilc}{\widetilde{c}}

\newcommand{\tilq}{\widetilde{q}}

\newcommand{\tilP}{\widetilde{P}}
\newcommand{\tilA}{\widetilde{A}}

\newcommand{\field}[1]{\mathbb{#1}}

\newcommand{\fR}{\field{R}}

\newcommand{\fN}{\field{N}}
\newcommand{\E}{\field{E}}
\newcommand{\fV}{\field{V}}

\newcommand{\Ind}{\field{I}}

\newcommand{\norm}[1]{\left\|{#1}\right\|}

\newcommand{\order}{\ensuremath{\mathcal{O}}}
\newcommand{\lorder}{\ensuremath{\Omega}}
\newcommand{\otil}{\ensuremath{\tilde{\mathcal{O}}}}

\usepackage{prettyref}
\newcommand{\pref}[1]{\prettyref{#1}}
\newcommand{\pfref}[1]{Proof of \prettyref{#1}}
\newcommand{\savehyperref}[2]{\texorpdfstring{\hyperref[#1]{#2}}{#2}}
\newrefformat{eq}{\savehyperref{#1}{Eq.~\textup{(\ref*{#1})}}}
\newrefformat{eqn}{\savehyperref{#1}{Equation~\ref*{#1}}}
\newrefformat{lem}{\savehyperref{#1}{Lemma~\ref*{#1}}}
\newrefformat{def}{\savehyperref{#1}{Definition~\ref*{#1}}}
\newrefformat{line}{\savehyperref{#1}{Line~\ref*{#1}}}
\newrefformat{thm}{\savehyperref{#1}{Theorem~\ref*{#1}}}
\newrefformat{corr}{\savehyperref{#1}{Corollary~\ref*{#1}}}
\newrefformat{cor}{\savehyperref{#1}{Corollary~\ref*{#1}}}
\newrefformat{sec}{\savehyperref{#1}{Section~\ref*{#1}}}
\newrefformat{subsec}{\savehyperref{#1}{Section~\ref*{#1}}}
\newrefformat{app}{\savehyperref{#1}{Appendix~\ref*{#1}}}
\newrefformat{assum}{\savehyperref{#1}{Assumption~\ref*{#1}}}
\newrefformat{ex}{\savehyperref{#1}{Example~\ref*{#1}}}
\newrefformat{fig}{\savehyperref{#1}{Figure~\ref*{#1}}}
\newrefformat{alg}{\savehyperref{#1}{Algorithm~\ref*{#1}}}
\newrefformat{rem}{\savehyperref{#1}{Remark~\ref*{#1}}}
\newrefformat{conj}{\savehyperref{#1}{Conjecture~\ref*{#1}}}
\newrefformat{prop}{\savehyperref{#1}{Proposition~\ref*{#1}}}
\newrefformat{proto}{\savehyperref{#1}{Protocol~\ref*{#1}}}
\newrefformat{prob}{\savehyperref{#1}{Problem~\ref*{#1}}}
\newrefformat{claim}{\savehyperref{#1}{Claim~\ref*{#1}}}
\newrefformat{que}{\savehyperref{#1}{Question~\ref*{#1}}}
\newrefformat{op}{\savehyperref{#1}{Open Problem~\ref*{#1}}}
\newrefformat{fn}{\savehyperref{#1}{Footnote~\ref*{#1}}}
\newrefformat{tab}{\savehyperref{#1}{Table~\ref*{#1}}}
\newrefformat{fig}{\savehyperref{#1}{Figure~\ref*{#1}}}

\DeclareOldFontCommand{\rm}{\normalfont\rmfamily}{\mathrm}
\DeclareOldFontCommand{\it}{\normalfont\itshape}{\mathit}

%% file: intro.tex

Stochastic Shortest Path (SSP) is a goal-oriented reinforcement learning setting, where a learner tries to reach a goal state with minimum total cost.
Compared to the heavily studied finite-horizon setting, SSP is often a better model for capturing many real-world applications such as games, car navigation, robotic manipulations, and others.
We study the online learning problem in SSP, where the learner interacts with an environment with unknown cost and transition function for multiple episodes.
In each episode, the learner starts from an initial state, sequentially takes an action, incurs a cost, and transits to the next state until the goal state is reached.
The goal of the learner is to achieve low regret, defined as the difference between her total cost and the expected cost of the optimal policy.
A unique challenge of learning SSP is to trade off between two objectives: reaching the goal state and minimizing the cost.
Indeed, neither reaching the goal as fast as possible nor minimizing the cost alone solves the problem.

Policy Optimization (PO) is among the most popular methods in reinforcement learning due to its strong empirical performance and favorable theoretical properties.
Unlike value-based approaches such as Q learning, PO-type methods directly optimize the policy in an incremental manner.
Many widely used practical algorithms fall into this category, such as REINFORCE~\citep{williams1992simple}, NPG~\citep{kakade2001natural}, and TRPO~\citep{schulman2015trust}.
They are also easy to implement and computationally efficient compared to other methods such as those operating over the occupancy measure space (e.g.,~\citep{zimin2013online}).
From a theoretical perspective, PO is a general framework that works for different types of environments, including stochastic costs or even adversarial costs~\citep{efroni2020optimistic}, function approximation~\citep{cai2020provably}, and non-stationary environments~\citep{fei2020dynamic}.
Despite its popularity in applications, most theoretical works on PO focus on simple models such as finite-horizon models~\citep{cai2020provably,efroni2020optimistic,luo2021policy} and discounted models~\citep{liu2019neural,wang2019neural,agarwal2021theory}, which are often oversimplifications of real-life applications.
In particular, PO methods have not been applied to regret minimization in SSP as far as we know.

Motivated by this gap, in this work, we systematically study policy optimization in SSP.
We consider a wide range of different settings and for each of them discuss how to design a policy optimization algorithm with a strong regret bound.
Specifically, our main results are as follows:
\begin{itemize}[leftmargin=*]
  \setlength\itemsep{0em}
	\item In \pref{sec:sda}, we first propose an important technique used in all our algorithms: \textit{stacked discounted approximation}.
	It reduces any SSP instance to a special Markov Decision Process (MDP) with a stack of $\bigo{\ln K}$ layers ($K$ is the total number of episodes), each of which contains a discounted MDP (hence the name) such that the learner stays in the same layer with a certain probability $\gamma$ and proceeds to the next layer with probability $1-\gamma$.
	This approximation not only resolves the difficulty of having dynamic and potentially unbounded episode lengths in the PO analysis, but more importantly leads to a near-stationary policies with only $\bigo{\ln K}$ changes within an episode.
	Compared to the commonly used finite-horizon approximation~\citep{chen2021minimax,chen2021finding,cohen2021minimax} which changes the policy at every step of an episode, our approach could lead to an exponential improvement in space complexity and is also more natural since the optimal policy for SSP is indeed stationary.
	
	\item Building on the stacked discounted approximation, in \pref{sec:sto}, we design PO algorithms for two types of stochastic environments considered in the literature.
	In the first type (called stochastic costs), the cost for each visit of a state-action pair is an i.i.d. sample of an unknown distribution and is revealed to the learner immediately after the visit.
	Our algorithm achieves $\tilo{\B S\sqrt{AK}}$ regret in this case, close to the minimax bound $\tilo{\B \sqrt{SAK}}$~\citep{cohen2021minimax}, where $S$ is the number of states, $A$ is the number of actions, and $\B$ is the maximum expected cost of the optimal policy starting from any states.
	In the second type (called stochastic adversary following~\citep{chen2021finding}), the cost function for each episode is fixed and an i.i.d. sample of an unknown distribution,
	and only at the end of the episode, the learner observes the entire cost function (full-information feedback) or the costs for all visited state-action pairs (bandit feedback).
	Our algorithm achieves $\tilo{\sqrt{D\T K} + DS\sqrt{AK}}$ regret with full information and $\tilo{\sqrt{D\T SAK} + DS\sqrt{AK}}$ regret with bandit feedback,
	where $D$ is the diameter of the MDP and $\T$ is the expected hitting time of the optimal policy starting from the initial state.
	These bounds match the best existing results from~\citep{chen2021finding} (and exhibit a $\sqrt{S}$ gap in the second term $DS\sqrt{AK}$ compared to their lower bounds).
	
	\item Finally, in \pref{sec:adv}, we further study SSP with adversarial costs and design PO algorithms that achieve $\tilo{\T\sqrt{DK} + \sqrt{D\T S^2AK}}$ regret with full information and $\tilo{\sqrt{\Tmax^5S^2AK}}$ regret with bandit feedback, where $\Tmax$ is the maximum expected hitting time of the optimal policy over all states.
	The best existing bounds for these settings are $\tilo{\sqrt{D\T S^2AK}}$ and $\tilo{\sqrt{D\T S^3A^2K}}$ respectively~\citep{chen2021finding}.
\end{itemize}

While our regret bounds do not always match the state-of-the-art,
we emphasize again that our algorithms are more space-efficient due to the stacked discounted approximation (and could also be more time-efficient in some cases).
We also note that in the analysis of stacked discounted approximation, a regret bound starting from any state (not just the initial state) is important, and PO indeed provides such a guarantee while other methods based on occupancy measure do not.
In other words, PO is especially compatible with our stacked discounted approximation.
Moreover, our results also significantly improve our theoretical understanding on PO,
and pave the way for future study on more challenging problems such as SSP with function approximation, where in some cases PO is the only method known to be computationally and statistically efficient~\citep{luo2021policy}.

\paragraph{Other Techniques} 

To achieve our results for stochastic environments, we make two other technical contributions. First, in order to control the cost estimation error optimally, we derive a set of novel correction terms fed to the PO algorithm, which resolves some technical difficulties brought by PO due to its lack of optimism and also greatly simplifies the analysis.
Second, due to the soft policy updates, the standard PO analysis leads to an undesirable dominating term related to $\T$ or even $\Tmax$ in the regret, and we develop a refined analysis on the value difference between learner's policies and the optimal policy to reduce this to a lower order term.

To achieve our results for adversarial environments, we develop a tighter variance-aware bound for the stability term in the PO analysis, which plays a key role in removing the $\Tmax$ dependency in the dominating term of the regret bound in the full information setting.
We further extend the dilated bonuses of~\citep{luo2021policy} (for the finite-horizon setting) to the stacked discounted MDPs, which is essential for both the full information setting and the bandit feedback setting.

\paragraph{Related Work} Regret minimization in SSP has received much attention recently for both stochastic environment~\citep{tarbouriech2020no,cohen2020near,cohen2021minimax,tarbouriech2021stochastic,chen2021implicit,chen2021improved,jafarnia2021online} and adversarial environment~\citep{rosenberg2020adversarial,chen2021minimax,chen2021finding}. 
All previous approaches are either value-based (e.g. Q learning) or occupancy-measure-based, while we take the first step in studying the more practical and versatile PO methods.
Among numerous studies on PO, the closest to our work are the recent ones by~\citet{efroni2020optimistic} and~\citet{luo2021policy} for the special case of finite-horizon MDPs.


The use of variance information~\citep{lattimore2012pac,azar2017minimax,zhou2021nearly,zhang2021variance,kim2021improved} and correction terms~\citep{steinhardt2014adaptivity,wei2018more,chen2021impossible} is crucial for achieving optimal and adaptive regret bound in online learning.
In this work we heavily make use of these ideas as mentioned.

%% file: pre.tex

An SSP instance is defined by a Markov Decision Process (MDP) $\calM=(\calS, \sinit, g, \calA, P)$.
Here, $\calS$ is the state space, $\sinit\in\calS$ is the initial state, $g\notin\calS$ is the goal state, $\calA$ is the action space, and $P=\{P_{s, a}\}_{(s, a)\in\SA}$ with $P_{s, a}\in\Delta_{\calS_+}$ is the transition function, where $\calS_+=\calS\cup\{g\}$ and $\Delta_{\calS_+}$ is the simplex over $\calS_+$.

The learning protocol is as follows: the learner interacts with the environment for $K$ episodes.
In episode $k$, the learner starts in initial state $\sinit$, sequentially takes an action, incurs a cost (which might not be observed immediately), and transits to the next state until the goal state $g$ is reached.
Formally, at the $i$-th step of episode $k$, the learner observes state $s^k_i$ (with $s^k_1=\sinit$), takes action $a^k_i$, suffers cost $c^k_i$, and transits to the next state $s^k_{i+1}\sim P_{s^k_i, a^k_i}$.
Denote by $I_k$ the length of episode $k$, such that $s^k_{I_k+1}=g$ when $I_k$ is finite. 
Note that the heavily studied finite-horizon setting is a special case of SSP where $I_k$ is always guaranteed to be some fixed number.

\paragraph{Proper Policies and Related Concepts} At a high level, the learner's goal is to reach the goal state with minimum cost.
Thus, we focus on \textit{proper policies}: a stationary policy $\pi:\calS\rightarrow\Delta_{\calA}$ is a mapping that assigns to each state a distribution over actions, and it is proper if following $\pi$ from any initial state reaches the goal state with probability $1$.
Denote by $\Pi$ the set of proper policies (assumed to be non-empty).
Given a proper policy $\pi$, a transition function $P$, and a cost function $c:\SA\rightarrow[0,1]$,
we define its value function and action-value function as follows:
$V^{\pi, P, c}(s) = \E\sbr{\left.\sum_{i=1}^Ic(s_i, a_i)\right|\pi, P, s_1=s}$ and 
$Q^{\pi, P, c}(s, a) = c(s, a) + \E_{s'\sim P_{s, a}}[V^{\pi, P, c}(s')]$,
where the expectation in $V^{\pi, P, c}$ is over the randomness of action $a_i\sim\pi(\cdot|s_i)$, next state $s_{i+1}\sim P_{s_i, a_i}$, and the number of steps $I$ before reaching $g$.
Also define the \emph{advantage function} $A^{\pi,P,c}(s, a)=Q^{\pi,P,c}(s, a) - V^{\pi,P,c}(s)$.

We consider two types of environments: stochastic environments and adversarial environments, which differ in the way costs are generated (and revealed), discussed in detail below.

\paragraph{Stochastic Environments} We start with the simpler environment with a fixed ``ground truth'' cost: there exists an unknown mean cost function $c:\SA\rightarrow[\cmin, 1]$, and the costs incurred by the learner are i.i.d samples from some distribution with support $[\cmin, 1]$ and mean $c$.
Here, $\cmin\in[0, 1]$ is a global lower bound.\footnote{Unlike many previous works for stochastic costs that require $\cmin>0$ in their analysis, our methods allow $\cmin=0$. 
} 
We consider the following three types of cost feedback.
\begin{enumerate}
	\item \textbf{Stochastic costs:} whenever the learner visits state-action pair $(s, a)$, she immediately observes (and incurs) an i.i.d cost sampled from some unknown distribution with mean $c(s, a)$.
	\item \textbf{Stochastic adversary, full information:} before learning starts, an adversary samples $K$ i.i.d. cost functions $\{c_k\}_{k=1}^K$ from some unknown distribution with mean $c$.
	At the $i$-th step of episode $k$, the learner incurs cost $c^k_i=c_k(s^k_i, a^k_i)$.
	Only at the end of this episode (after the goal state is reached), the learner observes the entire cost function $c_k$.
	\item \textbf{Stochastic adversary, bandit feedback:} this is the same as above, except that at the end of episode $k$, the learner only observes the costs of all visited state-action pairs: $\{c_k(s^k_i, a^k_i)\}_{i=1}^{I_k}$.
\end{enumerate}
The learner's objective is to minimize her regret, defined as the difference between her total incurred cost and the total expected cost of the best proper policy:
$
	R_K = \sumk\sum_{i=1}^{I_k}c^k_i - K\cdot V^{\optpi, P, c}(\sinit),
$
where $\optpi$ is the optimal proper policy satisfying $\optpi\in\argmin_{\pi\in\Pi}V^{\pi, P, c}(s)$ for all $s\in\calS$.

\paragraph{Adversarial Environments} We also consider the more challenging environment that adapts to learner's behavior in a possibly malicious manner.
Specifically, in episode $k$, the environment decides an arbitrary cost function $c_k: \SA\rightarrow[\cmin, 1]$ which could depend on the learner's algorithm as well as her randomness before episode $k$.
The learner then suffers cost $c^k_i=c_k(s^k_i, a^k_i)$ at the $i$-th step of episode $k$.
Similarly to the stochastic adversary case, the learner observes information on $c_k$ only after she reaches the goal state in episode $k$, and she observes the entire $c_k$ in the full-information setting or just the cost of visited state-action pairs $\{c_k(s^k_i, a^k_i)\}_{i=1}^{I_k}$ in the bandit setting.
The objective is again to minimize her regret against the optimal proper policy in hindsight:
$
	R_K = \sumk\rbr{\sum_{i=1}^{I_k}c^k_i - V^{\optpi, P, c_k}(\sinit)},
$
where we overload the notation $\optpi$ to denote the overall optimal proper policy such that $\optpi \in\argmin_{\pi\in\Pi}\sumk V^{\pi, P, c_k}(s)$ for all $s\in\calS$.

\paragraph{Key Parameters and Notations} 
Let $T^{\pi}(s)$ be one plus the expected number of steps to reach the goal if one follows policy $\pi$ starting from state $s$.
Four parameters play a key role in our analysis and regret bounds: $\B=\max_sV^{\optpi,P,c}(s)$, the maximum expected cost of the optimal policy starting from any state; $\T=T^{\optpi}(\sinit)$, the hitting time of the optimal policy starting from the initial state; $\Tmax=\max_sT^{\optpi}(s)$, the maximum hitting time of the optimal policy starting from any state; and $D=\max_s\min_{\pi}T^{\pi}(s)$, the SSP-diameter. 
We also define the \textit{fast policy} $\pi_f$ such that $\pi_f\in\argmin_{\pi}T^{\pi}(s)$ for all state $s$.
Similarly to previous works, in most discussions we assume the knowledge of all four parameters and the fast policy, and defer to \pref{app:pf} what we can achieve when some of these are unknown.
We also assume $\B\geq1$ for simplicity.

For $n\in\fN_+$, we define $[n]=\{1,\ldots,n\}$.
$\E_k[\cdot]$ denotes the conditional expectation given everything before episode $k$.
The notation $\tilo{\cdot}$ hides all logarithmic terms including $\ln K$ and $\ln\frac{1}{\delta}$ for some confidence level $\delta\in(0, 1)$.
For a distribution $\tilP\in\Delta_{\calS_+}$ and a function $V:\calS_+\rightarrow\fR$, define $\tilP V=\E_{s\sim \tilP}[V(s)]$.

%% file: sda.tex

Policy optimization algorithm have been naturally derived in many MDP models.
In the finite-horizon setting, one can update the policy at the end of each episode using the cost for this episode that is always bounded.
In the discounted setting or average reward setting with some ergodic assumption, one can also update the policy after a certain fixed number of steps since the short-term information is enough to predict the long-term behavior reasonably well.
However, this is not possible in SSP: the hitting time of an arbitrary policy can be arbitrarily large in SSP, and only looking at a fixed number of steps can not always provide accurate information.


A natural solution would be to approximate SSP by other MDP models, and then apply PO in the reduced model.
Approximating SSP instances by finite-horizon MDPs~\citep{chen2021implicit,chen2021minimax,cohen2021minimax} or discounted MDPs~\citep{tarbouriech2021stochastic,min2021learning} is a common practice in the literature, but both have their pros and cons.
Finite-horizon approximation shrinks the estimation error exponentially fast and usually leads to optimal regret~\citep{chen2021minimax,cohen2021minimax}.
However, it greatly increases the space complexity of the algorithm as it needs to store non-stationary policies with horizon of order $\tilo{\Tmax}$ or $\tilo{\frac{\B}{\cmin}}$.
Discounted approximation, on the other hand, produces stationary policies,
but the estimation error decreases only linearly in the effective horizon $(1-\gamma)^{-1}$, where $\gamma$ is the discounted factor.
This often leads to sub-optimal regret bounds 
and large time complexity~\citep{tarbouriech2021stochastic}.
These issues greatly limit the practical potential of these methods, and PO methods built on top of them would be less interesting.

To address these issues and achieve optimal regret with small space complexity, we introduce a new approximation scheme called \textit{Stacked Discounted Approximation}, which is a hybrid of finite-horizon and discounted approximations.
The key idea is as follows: the finite-horizon approximation requires a horizon of order $\bigo{\Tmax\ln K}$, but one can imagine that policies at nearby layers are close to each other and can be approximated by one stationary policy. 
Thus, we propose to achieve the best of both worlds by dividing the layers into $\bigo{\ln K}$ parts and performing discounted approximation within each part with an effective horizon $\bigo{\Tmax}$.
Formally, we define the following.
\begin{definition}
	\label{def:sda}
	For an SSP instance $\calM=(\calS, \sinit, g, \calA, P)$, we define, for number of layers $H$, discounted factor $\gamma$, and terminal cost $c_f$, another SSP instance $\rcalM=(\rcalS, \rsinit, g, \calA, \rP)$ as follows:
	\begin{enumerate}
		\item $\rcalS=\calS\times[H+1]$,  $\rsinit=(\sinit, 1)$, and the goal state $g$ remains the same.
		\item Transition from $(s, h)$ to $(s', h')$ is only possible for $h'\in\{h, h+1\}$: for any $h\leq H$ and $(s,a,s')\in\calS\times\calA\times\calS$, we have $\rP_{(s, h), a}(s', h)=\gamma P_{s, a}(s')$ (stay in the same layer with probability $\gamma$), $\rP_{(s, h), a}(s', h+1)=(1-\gamma)P_{s, a}(s')$ (proceed to the next layer with probability $1-\gamma$), and $\rP_{(s, h), a}(g)=P_{s, a}(g)$;
		for $h=H+1$, we have $\rP_{(s, H+1), a}(g)=1$ for any $(s,a)$ (immediately reach the goal if at layer $H+1$).
		For notational convenience, we also write $\rP_{(s, h), a}(s', h')$ as $P_{(s, h), a}(s', h')$ or $P_{s, a, h}(s', h')$, and $\rP_{(s, h), a}(g)$ as $P_{(s, h), a}(g)$ or $P_{s, a, h}(g)$. 
		\item For any cost function $c: \SA\rightarrow[0, 1]$ in $\calM$, we define a cost function $\rc$ for $\rcalM$ such that $\rc((s, h), a)=c(s, a)$ for $h\in[H]$ and $\rc((s, H+1), a)=c_f$ (terminal cost).
		For notational convenience, we also write $\rc((s, h), a)$ as $c((s, h), a)$ or $c(s, a, h)$.
	\end{enumerate}
\end{definition}

For any stationary policy $\pi$ in $\rcalM$, we write $\pi(a|(s, h))$ as $\pi(a|s, h)$, and we often abuse the notation $Q^{\pi, P, c}$ and $V^{\pi, P, c}$ to represent the value functions with respect to policy $\pi$, transition $\rP$, and cost function $\rc$.
We also often use $(s, a, h)$ in place of $((s, h), a)$ for function input, that is, we write $f((s, h), a)$ as $f(s, a, h)$.

Define $\roptpi$ for $\rcalM$ that mimics the behavior of $\optpi$, in the sense that $\roptpi(\cdot|s, h)=\optpi(\cdot|s)$.
If we set $\gamma=1-\frac{1}{2\Tmax}$, 
by the definition of $\Tmax$, it can be shown that the probability of $\roptpi$ transiting to the next layer before reaching $g$ is upper bounded by $1/2$.
If we further set $H=\bigo{\ln K}$, then the probability of transiting to the $(H+1)$-th layer before reaching $g$ is at most $\frac{1}{2^H}=\tilo{1/K}$.
As a result, the estimation error decreases exponentially in the number of layers while the policy only changes for $\bigo{\ln K}$ many times.
More importantly, due to the discounted factor, the expected hitting time of any policy is of order $\bigo{\frac{H}{1-\gamma}}=\bigo{\Tmax\ln K}$, which controls the cost of exploration and enables the learner to only update its policy at the end of an episode.
We summarize the intuition above in the following lemma.
\begin{lemma}
	\label{lem:sda}
	For any cost function $c:\SA\rightarrow[0, 1]$ and terminal cost $c_f$, 
	we have $V^{\pi, P, c}(s, h)\leq \frac{H-h+1}{1-\gamma} + c_f$ for any $h\in[H], s \in\calS$, and policy $\pi$ in $\rcalM$.
	Moreover, if $\gamma=1-\frac{1}{2\Tmax}$, we further have $Q^{\roptpi, P, c}(s, a, h) \leq Q^{\optpi, P, c}(s, a) + \frac{c_f}{2^{H-h+1}}$ for any $h\in[H]$ and $(s,a) \in\calS\times\calA$.
\end{lemma}
\begin{proof}
	The first statement is because in expectation it takes any policy $\frac{1}{1-\gamma}$ steps to transit from one layer to the next and each step incurs at most $1$ cost (except for the terminal cost).
	For the second statement, note that $V^{\pi, P, c}(s, H+1)=Q^{\pi, P, c}(s, a, H+1)=c_f$ for any $(s, a)\in\SA$, and for any $h\in[H]$, $V^{\pi, P, c}(s, h) = \suma\pi(a|s, h) Q^{\pi, P, c}(s, a, h)$ and
	$$Q^{\pi, P, c}(s, a, h)= c(s, a) + \gamma P_{s, a}V^{\pi, P, c}(\cdot, h) + (1-\gamma)P_{s, a}V^{\pi, P, c}(\cdot, h+1),$$
	where we abuse the notation and define $V^{\pi, P, c}(g, h)=0$ for all $h\in[H+1]$.
	Now we prove the second statement by induction for $h=H+1,\ldots,1$.
	The base case $h=H+1$ is clearly true.
	For $h\leq H$, we bound $Q^{\roptpi, P, c}(s, a, h) - Q^{\optpi, P, c}(s, a)$ as follows:
	\begin{align*}
		&\gamma P_{s, a}V^{\roptpi, P, c}(\cdot, h) + (1-\gamma)P_{s, a}V^{\roptpi, P, c}(\cdot, h+1) - P_{s, a}V^{\optpi, P, c}\\
		&\leq \gamma P_{s, a}(V^{\roptpi, P, c}(\cdot, h) - V^{\optpi, P, c}) + (1-\gamma)\frac{c_f}{2^{H-h}} \tag{$V^{\roptpi, P, c}(s, h+1) - V^{\optpi, P, c}(s)\leq \frac{c_f}{2^{H-h}}$ by induction}\\
		&= \gamma \E_{s' \sim P_{s,a}, a'\sim \optpi(s')}\sbr{Q^{\roptpi, P, c}(s', a', h) - Q^{\optpi, P, c}(s', a')} + (1-\gamma)\frac{c_f}{2^{H-h}}	.	
	\end{align*}
	By repeating the arguments above, we arrive at
	\[
	Q^{\roptpi, P, c}(s, a, h) - Q^{\optpi, P, c}(s, a) \leq \E\sbr{\left. \sum_{t=1}^I\gamma^{t-1}(1-\gamma)\frac{c_f}{2^{H-h}} \right| \optpi, P, s_1=s, a_1=a},
	\]
	where $I$ is the (random) number of steps it takes for $\optpi$ to reach the goal in $\calM$ starting from $(s,a)$.
	Bounding $\gamma^{t-1}$ by $1$ and $\E[I]$ by $\Tmax$, we then obtain the upper bound $\frac{(1-\gamma)\Tmax c_f}{2^{H-h}} = \frac{c_f}{2^{H-h+1}}$, which finishes the induction.
\end{proof}

\begin{remark}
	\label{rem:T}
	Applying the first statement of \pref{lem:sda} with $c(s, a)=1$ and $c_f=1$, we have the expected hitting time of any policy in $\rcalM$ bounded by $\frac{H}{1-\gamma}+1$ starting from any state in any layer.
\end{remark}

Now we complete the approximation by showing how to solve the original problem via solving its stacked discounted version.
Given a policy $\pi$ for $\rcalM$, define a non-stationary randomized policy $\sigma(\pi)$ for $\calM$ as follows: it maintains an internal counter $h$ initialized as $1$.
In each time step before reaching the goal, it first follows $\pi(\cdot|s, h)$ for one step, where $s$ is the current state.
Then, it samples a Bernoulli random variable $X$ with mean $\gamma$, and it increases $h$ by $1$ if $X=0$.
When $h=H+1$, it executes the fast policy $\pi_f$ until reaching the goal state.
Clearly, the trajectory of $\sigma(\pi)$ indeed follows the same distribution of the trajectory of $\pi$ in $\rcalM$.
We show that as long as $H$ is large enough and $c_f$ is of order $\tilo{D}$, this reduction makes sure that the regret between these two problems are similar.
The proof is deferred to \pref{app:sda}.
\begin{lemma}
	\label{lem:approx}
	Let $\gamma=1-\frac{1}{2\Tmax}$, $H=\ceil{\log_2(c_fK)}$, $c_f=\ceil{4D\ln\frac{2K}{\delta}}$ for some $\delta\in(0, 1)$, and $\pi_1,\ldots,\pi_K$ be policies for $\rcalM$.
	Then the regret of executing $\sigma(\pi_1),\ldots,\sigma(\pi_K)$ in $\calM$ satisfies
	$R_K \leq \rR_K + \tilo{1}$ with probability at least $1-\delta$, where $\rR_K = \sumk\rbr{\sumit c^k_i + \rc^k_{J_k+1} - V^{\roptpi,P,c}(s^k_1, 1)}$ for stochastic environments, and $\rR_K = \sumk\rbr{\sumit c^k_i + \rc^k_{J_k+1} - V^{\roptpi,P,c_k}(s^k_1, 1)}$ for adversarial environments. 
	Here, $J_k$ is the number of time steps in episode $k$ before the learner reaching $g$ or the counter of $\sigma(\pi_k)$ reaching $H+1$, and $\rc^k_{J_k+1}=c_f\Ind\{s^k_{J_k+1}\neq g\}$.
\end{lemma}

\paragraph{Computing Fast Policy and Estimating Diameter} For simplicity, we assume knowledge of the diameter and the fast policy above.
When these are unknown, one can follow the ideas in~\citep{chen2021finding} for estimating the fast policy with constant overhead and then adopt their template for learning without knowing the diameter; see~\citep[Lemma~1, Appendix E]{chen2021finding}.

\paragraph{Policy Optimization in Stacked Discounted MDPs} Now we describe a template of performing policy optimization with the stacked discounted approximation.
The pseudocode is shown in \pref{alg:po}.
To handle unknown transition, we maintain standard Bernstein-style transition confidence sets $\{\calP_k\}_{k=1}^K$ whose definition is deferred to \pref{app:conf}.
In episode $k$, the algorithm first computes policy $\pi_k$ in $\rcalM$ following the multiplicative weights update with some learning rate $\eta > 0$, such that $\pi_k(a|s, h) \propto e^{-\eta\sum_{j=1}^{k-1}(\tilQ_j(s, a, h)-B_j(s, a, h)) }$ for some optimistic action-value estimator $\tilQ_j$ and exploration bonus function $B_j$ (computed from past observations and confidence sets).
Then, it executes $\sigma(\pi_k)$ for this episode.
Finally, it computes confidence set $\calP_{k+1}$.
All algorithms introduced in this work follow this template and differ from each other in the definition of $\tilQ_k$ and $B_k$.
Ideally, $\tilQ_k-B_k$ should be the action-value function with respect to the true transition, the true cost function, and policy $\pi_k$,
but since the transition and cost functions are unknown, the key challenge lies in constructing accurate estimators that simultaneously encourage sufficient exploration.

\begin{algorithm}[t]
	\caption{Template for Policy Optimization with Stacked Discounted Approximation}
	\label{alg:po}
	
	\textbf{Initialize:} $\calP_1$, the set of all possible transition functions in $\rcalM$ (\pref{eq:all P}); $\eta>0$, some learning rate.
	
	
	\For{$k=1,\ldots,K$}{
	    
	    Compute $\pi_k(a|s, h) \propto \exp\rbr{-\eta\sum_{j=1}^{k-1}(\tilQ_j(s, a, h)-B_j(s, a, h)) }$.
	    
	    Execute $\sigma(\pi_k)$ for one episode (see the paragraph before \pref{lem:approx}).
		
		Compute some optimistic action-value estimator $\tilQ_k$ and exploration bonus function $B_k$ using $\calP_k$ and observations from episode $k$.
		
		Compute transition confidence set $\calP_{k+1}$, as defined in \pref{eq:conf}.
	}
\end{algorithm}

%
%
%
%

\paragraph{Optimistic Transitions}
Our algorithms require using some optimistic transitions.
Specifically, for a policy $\pi$, a confidence set $\calP$, and a cost function $c$, let 
$\Gamma(\pi, \calP, c)$ be the corresponding optimistic transition such that $\Gamma(\pi, \calP, c) \in \argmin_{P\in\calP}V^{\pi, P, c}(s,h)$ for all state $(s,h)$.
The existence of such an optimistic transition and how it can be efficiently approximated via Extended Value Iteration (in at most $\tilo{\Tmax}$ iterations) are deferred to \pref{app:compute Gamma}.
We abuse the notation and denote by $V^{\pi, \calP, c}$ and $Q^{\pi, \calP, c}$ the value function $V^{\pi, \Gamma(\pi, \calP, c), c}$ and action-value function $Q^{\pi, \Gamma(\pi, \calP, c), c}$.

\paragraph{Occupancy Measure} Another important concept for subsequent discussions is \emph{occupancy measure}.
Given a policy $\pi: \rcalS\rightarrow\Delta_{\calA}$ and a transition function $P=\{P_{s,a,h}\}_{(s,h)\in\rcalS,a\in\calA}$ with $P_{s,a,h}\in\Delta_{\rcalS_+}$ and $\rcalS_+=\rcalS\cup\{g\}$, define $q_{\pi, P}:\rcalS\times\calA\times\rcalS_+ \rightarrow \fR_+$ such that $q_{\pi, P}(\rs, a, \rs')=\E[\sum_{i=1}^I\Ind\{s_i=\rs, a_i=a, s_{i+1}=\rs'\}|\pi, P, s_1=\rsinit]$ is the expected number of visits to $(\rs, a, \rs')$ following policy $\pi$ in a stacked discounted MDP with transition $P$.
We also let $q_{\pi,P}(s,a,h)=\sum_{\rs'}q_{\pi,P}((s,h), a, \rs')$ be the expected number of visits to $((s, h), a)$ and $q_{\pi,P}(s, h)=\sum_aq_{\pi,P}(s,a,h)$ be the number of visits to $(s, h)$.
Note that if a function $q:\rcalS\times\calA\times\rcalS_+ \rightarrow \fR_+$ is an occupancy measure, then the corresponding policy $\pi_q$ satisfies $\pi_q(a|s,h)\propto q(s, a, h)$ and the corresponding transition function $P_q$ satisfies $P_{q,s,a,h}(s',h')\propto q((s,h), a, (s', h'))$.
Moreover, $V^{\pi, P, c}(\rsinit)=\inner{q_{\pi, P}}{c}$ holds for any policy $\pi$, transition function $P$ and cost function $c$.

\paragraph{Other Notations} In the rest of the paper, following \pref{lem:approx} we set $\gamma=1-\frac{1}{2\Tmax}$, $H=\ceil{\log_2(c_fK)}$, and $c_f=\ceil{4D\ln\frac{2K}{\delta}}$ for some failure probability $\delta\in(0, 1)$.
Define $\chi=2H\Tmax+c_f$ as the value function upper bound in $\rcalM$ (according to the first statement of \pref{lem:sda}).
Also define $q_k=q_{\pi_k,P}$, $\optq=q_{\roptpi,P}$, and $L=\ceil{\frac{8H}{1-\gamma}\ln(2\Tmax K/\delta)}$.

%% file: sto.tex

In this section, we consider policy optimization in stochastic environments with three types of feedback introduced in \pref{sec:pre}.
We show that a simple policy optimization framework can be used to achieve near-optimal regret for all three settings.
In contrast, previous works treat stochastic costs and stochastic adversaries as different problems and solve them via different approaches.
Below, we start by describing the algorithm and its guarantees, followed by some explanation behind the algorithm design and then some key ideas and novelty in the analysis.

\paragraph{Algorithm}
As mentioned, the only elements left to be specified in \pref{alg:po} are $\tilQ_k$ and $B_k$.
For stochastic environments, we simply set $B_k(s,a,h)=0$ for all $(s,a,h)$ since exploration is relatively easier in this case.
We now discuss how to construct $\tilQ_k$.

\begin{itemize}[leftmargin=*]
\setlength\itemsep{0em}
\item
\textbf{Action-value estimator $\tilQ_k$} is defined as $Q^{\pi_k, \calP_k, \tilc_k}$ for some corrected cost estimator $\tilc_k$:
\begin{align}
	\label{eq:corrected_c}
	\tilc_k(s, a, h)=(1+\lambda\hatQ_k(s, a, h))\hatc_k(s, a, h) + e_k(s, a, h),
\end{align}
where $\lambda$ is some parameter, $\hatQ_k=Q^{\pi_k, \calP_k, \hatc_k}$ is another action-value estimator with respect to some optimistic cost estimator $\hatc_k$, and $e_k$ is some correction term (all to be specified below).

\item
\textbf{Optimistic cost estimator $\hatc_k$} is defined as 
\begin{align*}
\hatc_k(s,a,h) &=\hatc_k(s,a)\Ind\{h\leq H\}+c_f\Ind\{h=H+1\}, \\
\hatc_k(s,a) &=\max\big\{0, \barc_k(s, a) - 2\sqrt{\barc_k(s, a)\alpha_k(s, a)} - 7\alpha_k(s, a)\big\},
\end{align*}
where $\barc_k(s, a)$ is the average of all costs that are observed for $(s,a)$ in episode $j=1, \ldots, k-1$ before $\sigma(\pi_j)$ switches to the fast policy,
and $\alpha_k(s,a)$ is $\iota=\ln(2SALK/\delta)$ divided by the number of samples used in computing $\barc_k(s, a)$, such that $2\sqrt{\barc_k(s, a)\alpha_k(s, a)} + 7\alpha_k(s, a)$ is a standard Bernstein-style deviation term (thus making $\hatc_k(s,a)$ an optimistic underestimator).  
We note that naturally, the way to compute $\barc_k(s, a)$ is different for different types of feedback --- for stochastic costs, we might have multiple samples for $(s,a)$ in one episode, while for stochastic adversaries, we have exactly one sample in each episode in the full-information setting, and one or zero samples in the bandit setting.

\item
\textbf{Correction term $e_k(s, a, h)$} is defined as $0$ for stochastic costs; 
$(8\iota\sqrt{\nicefrac{\hatc_k(s, a, h)}{k}}+\beta'\hatQ_k(s,a,h))\Ind\{h\leq H\}$ with $\beta'=\min\{\nicefrac{1}{\Tmax}, \nicefrac{1}{\sqrt{D\T K}}\}$ for stochastic adversary with full information;
and $\beta\hatQ_k(s, a, h)\Ind\{h\leq H\}$ with $\beta=\min\{\nicefrac{1}{\Tmax}, \sqrt{\nicefrac{SA}{D\T K}}\}$ for stochastic adversary with bandit feedback.

\item
\textbf{Parameter tuning}: 
learning rate $\eta$ (for the multiplicative weights update) is set to $\min\{\nicefrac{1}{3\Tmax(8\iota+\nicefrac{\chi}{\Tmax})^2}, \nicefrac{1}{\sqrt{\lambda \Tmax^4K}}\}$, and the parameter $\lambda$ is set to $\min\{\nicefrac{1}{\Tmax}, \sqrt{\nicefrac{S^2A}{\square^2K}}\}$ where $\square$ is $\B$ for stochastic costs and $D$ for stochastic adversaries.
\end{itemize}  

We now state the regret guarantees of our algorithm for each of the three settings (proofs are defered to \pref{app:proof_SF-SC} to \pref{app:proof_SA-B}).

\begin{theorem}
	\label{thm:SF-SC}
	For stochastic costs, \pref{alg:po} with the instantiation above achieves $R_K = \tilo{\B S\sqrt{AK} + \Tmax^3(S^2AK)^{1/4} + S^4A^{2.5}\Tmax^4}$ with probability at least $1-32\delta$.
\end{theorem}

Ignoring lower-order terms, our bound almost matches the minimax bound $\tilo{\B \sqrt{SAK}}$ of~\citep{cohen2021minimax}, with a $\sqrt{S}$ factor gap.

\begin{theorem}
	\label{thm:SA-F}
	For stochastic adversary with full information, \pref{alg:po} with the instantiation above achieves $R_K = \tilo{\sqrt{D\T K} + DS\sqrt{AK} + \Tmax^3(S^2A^3K)^{1/4} + S^4A^{2.5}\Tmax^4 }$ with probability at least $1-50\delta$. 
\end{theorem}
\begin{theorem}
	\label{thm:SA-B}
	For stochastic adversary with bandit feedback, \pref{alg:po} with the instantiation above achieves $R_K = \tilo{\sqrt{SAD\T K} + DS\sqrt{AK} + \Tmax^3SA^{5/4}K^{1/4} + S^4A^{2.5}\Tmax^4}$ with probability at least $1-50\delta$.
\end{theorem}

Ignoring lower-order terms again, these bounds for stochastic adversary match the best known results from~\citep{chen2021finding}, and they all exhibit a $\sqrt{S}$ gap in the term $DS\sqrt{AK}$ compared to the best existing lower bounds~\citep{chen2021finding}.

We emphasize again that besides the simplicity of PO, one algorithmic advantage of our method compared to those based on finite-horizon approximation is its low space complexity to store policies --- the horizon $H$ for our method is only $\bigo{\ln K}$, while the horizon for other works~\citep{chen2021finding,cohen2021minimax} is $\tilo{\Tmax}$ when $\Tmax$ is known or otherwise $\tilo{\nicefrac{\B}{\cmin}}$.
Note that when $\cmin = 0$, a common technique is to perturb the cost and deal with a modified problem with $\cmin = \nicefrac{1}{\text{poly}(K)}$, in which case our space complexity is exponentially better.
In fact, even for time complexity, although our method requires calculating optimistic transition and might need $\tilo{\Tmax}$ rounds of Extended Value Iteration, this procedure could terminate much earlier, while the finite-horizon approximation approaches always need at least $\lowo{\Tmax}$ time complexity since that is the horizon of the MDP they are dealing with.

\paragraph{Analysis highlights}
We start by explaining the design of the corrected cost estimator \pref{eq:corrected_c}.
Roughly speaking, standard analysis of PO leads to a term of order $\lambda\sumk\inners{q_k}{c\circ \hatQ_k}$ due to the transition estimation error, which can be prohibitively large (for functions $f$ and $g$ with the same domain, we define $(f\circ g)(x) = f(x)g(x)$).
Introducing the correction bias $\lambda\hatQ_k(s, a, h)\hatc_k(s, a, h)$ in \pref{eq:corrected_c}, on the other hand, has the effect of transforming this problematic term into its counterpart $\lambda\sumk\inners{\optq}{c\circ \hatQ_k}$ in terms of $\optq$ instead of $q_k$.
Bounding the latter term, however, requires a property that PO enjoys, that is, a regret bound for any initial state-action pair: $\sumk(\hatQ_k-Q^{\roptpi, P, c})(s, a, h)=\tilo{\sqrt{K}}$ for any $(s, a, h)$.
In contrast, approaches based on occupancy measure~\citep{chen2021finding} only guarantee a regret bound starting from $\sinit$.
This makes PO especially compatible with our stacked discounted approximation.
Based on this observation, we further have $\lambda\sumk\inners{\optq}{c\circ \hatQ_k}\approx\lambda\sumk\inner{\optq}{c\circ Q^{\roptpi, P, c}}$, where the latter term is only about the behavior of the optimal policy and is thus nicely bounded (see e.g. \pref{lem:qcQ}).
To sum up, the correction term $\lambda\hatQ_k(s, a, h)\hatc_k(s, a, h)$  in \pref{eq:corrected_c} together with a favorable property of PO helps us control the transition estimation error in a near-optimal way.


For stochastic adversaries, an extra complication arises due to the cost estimation error $\sumk \inner{q_k}{c - \hatc_k}$, which results in the extra $\sqrt{D\T K}$ or $\sqrt{SAD\T K}$ term in the minimax regret bound (depending on the feedback type). 
Obtaining this optimal cost estimation error requires us to add yet another correction term $e_k$ in \pref{eq:corrected_c}.
Specifically, we show that $\sumk \inner{q_k}{c - \hatc_k} \approx \sumk \inner{q_k}{e_k}$ for $e_k$ defined as in our algorithm description.
Then, the role of adding $e_k$ in \pref{eq:corrected_c} is again to turn the term above to its counterpart $\sumk \inner{\optq}{e_k}$ in terms of the optimal policy's behavior, which can then be nicely bounded.
As a side product, we note that this also provides a much cleaner analysis on bounding the cost estimation error compared to~\citep{chen2021finding}, where they require explicitly forcing the expected hitting time of the learner's policy to be bounded.


Finally, we point out another novelty in our analysis.
Compared to other approaches that act according to the exact optimal policy of an estimated MDP, PO incurs an additional cost due to only updating the policy incrementally in each episode.
This cost is often of order $\tilo{\sqrt{K}}$ and is one of the dominating terms in the regret bound; see e.g.~\citep{efroni2020optimistic,wu2021nearly} for the finite-horizon case.
For SSP, this is undesirable because it also depends on $\T$ or even $\Tmax$.
Reducing this cost has been studied from the optimization perspective ---
for example, an improved $\tilo{\nicefrac{1}{K}}$ convergence rate of PO has been established recently by~\citep{agarwal2021theory}.
However, adopting their analysis to regret minimization requires additional efforts.
Specifically, we need to carefully bound the bias from using an action-value estimator in the policy's update, which can be shown to be approximately bounded by $\sumk (\tilQ_{k+1}-\tilQ_k)(s, a, h)$.
In \pref{lem:diff}, we show that this term is of lower order by carefully analyzing the drift $(\tilQ_{k+1}-\tilQ_k)(s, a, h)$ in each episode.

\begin{remark}
We remark that our algorithm can be applied to finite-horizon MDPs with inhomogeneous transition and gives a $\tilo{\sqrt{S^2AH^3K}}$ regret bound, improving over that of~\citep{efroni2020optimistic} by a factor of $\sqrt{H}$ where $H$ is the horizon.
	We omit the details but only mention that the improvement comes from two sources: first, the aforementioned improved PO analysis turns a $\tilo{H^2\sqrt{K}}$ regret term into a lower order term; second, we use Bernstein-style transition confidence set to obtain an improved $\tilo{\sqrt{S^2AH^3K}}$ transition estimation error.
\end{remark}

%% file: adv.tex

We move on to consider the more challenging environments with adversarial costs, where the extra exploration bonus function $B_k$ in \pref{alg:po} now plays an important role.
Even in the finite-horizon setting, developing efficient PO methods in this case can be challenging, and \citet{luo2021policy} proposed the so-called ``dilated bonuses'' to guide better exploration, which we also adopt and extend to SSP.
Specifically, for a policy $\pi$, a transition confidence set $\calP$, and some bonus function $b: \SA\times[H+1]\rightarrow\fR$, we define the corresponding dilated bonus function $B^{\pi, \calP, b}: \SA\times[H+1]\rightarrow\fR$ as: $B^{\pi, \calP, b}(s, a, H+1)=b(s, a,H+1)$ and for $h \in [H]$,
\begin{equation}
	B^{\pi, \calP, b}(s, a, h) = b(s, a, h) + \rbr{1 + \frac{1}{H'}}\max_{\hatP\in\calP}\hatP_{s, a, h}\rbr{\sum_{a'}\pi(a'|\cdot, \cdot)B^{\pi, \calP, b}(\cdot, a', \cdot)},
\end{equation}
where $H'=\frac{8(H+1)\ln(2K)}{1-\gamma}$ is the dilated coefficient.
Intuitively, $B^{\pi, \calP, b}$ is the dilated (by a factor of $1 + \nicefrac{1}{H'}$) and optimistic (by maximizing over $\calP$) version of the action-value function with respect to $\pi$ and $b$.
In the finite-horizon setting~\citep{luo2021policy}, this can be computed directly via dynamic programming,
but how to compute it in a stacked discounted MDP (or even why it exists) is less clear.
Fortunately, we show that this can indeed be computed efficiently via a combination of dynamic programming and Extended Value Iteration; see \pref{app:compute B}.



%
%
%
%
%

\paragraph{Algorithm (full information)}
We now describe our algorithm for the adversarial full-information case (where $c_k$ is revealed at the end of episode $k$).
It suffices to specify $\tilQ_k$ and $B_k$ in \pref{alg:po}.

\begin{itemize}[leftmargin=*]
\setlength\itemsep{0em}
\item \textbf{Action-value estimator} $\tilQ_k$ is defined as $\tilQ_k=Q^{\pi_k, \calP_k, \tilc_k}$, where $\tilc_k(s,a,h)=(1+\lambda\hatQ_k(s, a, h))c_k(s, a, h)$ for some parameter $\lambda$ and $\hatQ_k=Q^{\pi_k, \calP_k, c_k}$.
\item \textbf{Dilated bonus $B_k$} is defined as $B^{\pi_k, \calP_k, b_k}$ with $b_k(s, a, h) = 2\eta\suma\pi_k(a|s, h)\tilA_k(s, a, h)^2$,
where $\tilA_k(s, a, h)=\tilQ_k(s, a, h) - \tilV_k(s, h)$ (advantage function) and $\tilV_k=V^{\pi_k, \calP_k, \tilc_k}$.
\item \textbf{Parameter tuning}: $\eta=\min\{\nicefrac{1}{(64\chi^2\sqrt{HH'})}, \nicefrac{1}{\sqrt{DK}}\}$ and $\lambda=\min\{\nicefrac{1}{\chi}, 48\eta + \sqrt{\nicefrac{S^2A}{D\T K}}\}$.
\end{itemize}

Our algorithm enjoys the following guarantee (whose proof can be found in \pref{app:proof_adv_full}).
\begin{theorem}
	\label{thm:full}
	For adversarial costs with full information, \pref{alg:po} with the instantiation above achieves
	$R_K = \tilo{\T\sqrt{DK} + \sqrt{S^2AD\T K} + S^4A^2\Tmax^5}$
	with probability at least $1-20\delta$.
\end{theorem}

The best existing bound is $\tilo{\sqrt{S^2AD\T K}}$ from~\citep{chen2021finding}.
Ignoring the lower order term,
our result matches theirs when $\T\leq S^2A$ (and is worse by a $\sqrt{\nicefrac{\T}{S^2A}}$ factor otherwise).
Our algorithm enjoys better time and space complexity though, similar to earlier discussions.

\paragraph{Analysis highlights}
For simplicity we assume that the true transition is known, in which case our bound is only $\tilo{\T\sqrt{DK}}$ (the other term $\sqrt{S^2AD\T K}$ is only due to transition estimation error).
A naive way to implement PO would lead to a penalty term $\nicefrac{\T}{\eta}$ plus a stability term $\eta\sumk\sum_{s, h}\optq(s, h)\sum_a\pi_k(a|s, h)Q^{\pi_k,P,c_k}(s, a, h)^2$, which eventually leads to a bound of order $\tilo{\T\Tmax\sqrt{K}}$ if one bounds $Q^{\pi_k,P,c_k}(s, a, h)$ by $\tilo{\Tmax}$.
Our improvement comes from the following five steps:
1) first, through a careful shifting argument, we show that the stability term can be improved to $\eta\sumk\sum_{s, h}\optq(s, h)\sum_a\pi_k(a|s, h)A^{\pi_k,P,c_k}(s, a, h)^2$ (recall that $A$ is the advantage function);
2) second, similarly to~\citep{luo2021policy}, the dilated bonus $B_k$ helps transform $\optq$ to $q_k$ in the term above, leading to $\eta\sumk\inner{q_k}{(A^{\pi_k,P,c_k})^2}$;
3) third, in \pref{lem:var} we show that the previous term is bounded by the variance of the learner's cost, which in turn is at most $\eta\sumk\inner{q_k}{c_k\circ Q^{\pi_k,P,c_k}}$;
4) fourth, similarly to \pref{sec:sto}, the correction term $\lambda c_k\circ\hatQ_k$ in the definition of $\tilc_k$ helps transform $q_k$ back to $\optq$, resulting in $\eta\sumk\inner{\optq}{c_k\circ Q^{\pi_k,P,c_k}}$;
5) finally, since PO guarantees a regret bound for any initial state (as mentioned in  \pref{sec:sto}), the previous term is close to $\eta\sumk\inner{\optq}{c_k\circ Q^{\roptpi_k,P,c_k}}$, which is now only related to the optimal policy and can be shown to be at most $\tilo{\eta D\T K}$.
Combining this with the penalty term $\nicefrac{\T}{\eta}$ and picking the best $\eta$ then results in the claimed $\tilo{\T\sqrt{DK}}$ regret bound.

\paragraph{Algorithm (Bandit Feedback)}
Finally, we describe our algorithm for the adversarial setting with bandit feedback, starting with the instantiation of $B_k$ followed by that of $\tilQ_k$.

\begin{itemize}[leftmargin=*]
\setlength\itemsep{0em}
\item \textbf{Dilated bonus $B_k$} is again defined as $B^{\pi_k, \calP_k, b_k}$, but with a different $b_k$ function similar to that of~\citep{luo2021policy}:
$
		b_k(s, a, h) = L'\Ind\{h\leq H\}\sum_{a'}\pi_k(a'|s, h)\frac{\ux_k(s, a', h)- \lx_k(s, a', h) + 4\theta}{\ux_k(s, a', h)+\theta},
$
for some parameters $L'$ and $\theta$.
Here, $\ux_k(s, a, h)$ and $\lx_k(s, a, h)$ are respectively the largest and smallest possible probability that $((s, h), a)$ is ever visited in episode $k$ following policy $\pi_k$ if the transition lies in $\calP_k$, and they can be computed efficiently as shown in \pref{app:compute ux}.

\item \textbf{Action-value estimator} $\tilQ_k$ is defined as
$
	\tilQ_k(s, a, h) = \frac{G_{k,s,a,h}}{\ux_k(s,a,h)+\theta}\Ind\{h\leq H\} + c_f\Ind\{h=H+1\},
$
where $G_{k,s,a,h}$ is the learner's total cost in $\rcalM$ starting from the first visit to $((s, h), a)$ during the first $L+1$ steps of episode $k$. Recall the definition of $L$ stated at the end of \pref{sec:sda}, which is a high-probability upper bound on the number of steps any policy in $\rcalM$ takes to reach the last layer (so counting only the first $L+1$ steps is simply to make sure that $G_{k,s,a,h}$ is always bounded).



\item \textbf{Parameter tuning:} $\eta=\min\cbr{\nicefrac{1}{(300HH'\Tmax L')}, \sqrt{\nicefrac{1}{\Tmax^2SAK}}}$, $\theta=2\eta L'$, and $L'=L+c_f$.
\end{itemize}

We note that this algorithm is in spirit very similar to that of~\citep{luo2021policy} for the finite-horizon case.
Unfortunately, the correction terms we use throughout other algorithms in this work do not work here for technical reasons, resulting in the following sub-optimal  guarantee which still has $\Tmax$ dependency in the dominating term 
(see \pref{app:proof_adv_bandit} for the proof).
We remark that the best existing bound is $\tilo{\sqrt{D\T S^3A^2K}}$ from~\citep{chen2021finding}.

\begin{theorem}
	\label{thm:bandit}
	For adversarial costs with bandit feedback, \pref{alg:po} with the instantiation above achieves
	$R_K = \tilo{\sqrt{S^2A\Tmax^5K} + S^{5.5}A^{3.5}\Tmax^5}$ with probability at least $1-28\delta$.
\end{theorem}


%% file: app-pre.tex

\paragraph{Extra Notations}
Define $\rs^k_i=(s^k_i, h^k_i)$ as the $i$-th step in $\rcalM$ in episode $k$.
Define $n_k(s, a, h)$ as the number of visits to $((s, h), a)$ in $\rcalM$ in episode $k$, and $n_k(s, a)=\sum_{h\leq H}n_k(s, a, h)$ (excluding layer $H+1$).
Define $\J_k=\min\{L, J_k\}$, $\bar{n}_k(s, a)=\min\{L, n_k(s, a)\}$, and $\bar{n}_k(s, a, h)=\min\{L, n_k(s, a, h)\}$.
For any sequence of scalars or functions $\{z_k\}_k$, define $dz_k=z_{k+1}-z_k$.
By default we assume $\sum_h=\sum_{h=1}^{H+1}$.
For inner product $\inner{u}{v}$, if $u(s, a)$, $u(s, a, h)$, $v(s,a)$, and $v(s, a, h)$ are all defined, we let $\inner{u}{v}=\sum_{s, a, h}u(s, a, h)v(s, a, h)$.
For functions $f$ and $g$ with the same domain, define function $(f\circ g)(x) = f(x)g(x)$.
For any random variable $X$, define conditional variance $\V_k[X]=\E_k[(X-\E_k[X])^2]$.

For an occupancy measure $q$ w.r.t policy $\pi$ and transition $P$, define $q_{(s, h)}$ as the occupancy measure w.r.t policy $\pi$, transition $P$, and initial state $(s, h)$, and $q_{(s, a, h)}$ as the occupancy measure w.r.t policy $\pi$, transition $P$, initial state $(s, h)$, and initial action $a$.
Denote by $x_k(s, a, h)$ the probability that $((s, h), a)$ is ever visited in episode $k$, $x_k(s, a)=\sum_{h=1}^Hx_k(s, a, h)$ the probability that $(s, a)$ is ever visited before layer $H+1$ in episode $k$, and $y_k(s, a, h)$ the probability of visiting $((s, h), a)$ again if the agent starts from $((s, h), a)$.
For any occupancy measure $q(s, a, h)$, we define $q(s, a)=\sum_{h\leq H}q(s, a, h)$ (excluding layer $H+1$).
Note that $q_k(s, a, h)=\frac{x_k(s, a, h)}{1-y_k(s, a, h)}$ and $y_k(s, a, h)\leq\gamma$.
Thus, we have $q_k(s, a, h)=\bigo{\Tmax x_k(s, a, h)}$.

Define $\Lambda_{\calM}$ as the set of possible transition functions of $\rcalM$:
\begin{align}
	\Lambda_{\calM} = \Big\{ &P=\{P_{s,a,h}\}_{(s,h)\in\rcalS,a\in\calA}, P_{s,a,h}\in\Delta_{\rcalS_+}: P_{s,a,H+1}(g) = 1,\notag\\
	&\sum_{s'\in\calS}P_{s,a,h}(s', h) \leq \gamma, \sum_{s'\in\calS}P_{s,a,h}(s',h+1)\leq 1-\gamma,\notag\\ 
	&P_{s,a,h}(s',h')=0, \forall (s,a)\in\SA, h \in [H], h'\notin\{h,h+1\} \Big\},\label{eq:all P}
\end{align}
where $\gamma\cdot\calX=\{\gamma x: x\in\calX\}$ for some set $\calX$.
By definition, the expected hitting time of any stationary policy in an MDP with transition $P\in\Lambda_{\calM}$ is upper bounded by $(H+1)(1-\gamma)^{-1}$ starting from any state. 
Therefore, for any occupancy measure $q$ with $P_q\in\Lambda_{\calM}$ (for example, $q_k$ and $\optq$), we have $\sum_{s, a, h}q(s, a, h)\leq (H+1)(1-\gamma)^{-1}=\tilo{\Tmax}$.

Finally define $\calC_{\calM}$ as the set of possible cost functions of $\rcalM$: 
$$\calC_{\calM}=\cbr{c: \rcalS\rightarrow \fR_+: c(s,a,h)=\tilo{1}, \forall h\leq H, \text{ and } \exists C_0=\tilo{\Tmax}, c(s, a, H+1) = C_0,\forall a }.$$

\subsection{Transition Estimation}
\label{app:conf}

In this section, we present important lemmas regarding the transition confidence sets $\{\calP_k\}_{k=1}^K$.
We first prove an auxiliary lemma saying that the number of steps taken by the learner before reaching $g$ or switching to fast policy is well bounded with high probability.

\begin{lemma}
	\label{lem:bound J}
	With probability at least $1-\delta$, we have $J_k=\J_k$ for all $k\in[K]$.
\end{lemma}
\begin{proof}
    We want to show that $J_k \le L = \lceil \frac{8 H}{1 - \gamma} \ln (2 \Tmax K / \delta) \rceil$ for all $k \in [K]$ with probability at least $1 - \delta$.
    Let $k \in [K]$, it suffices to show that the expected hitting time of $\pi_k$ is upper bounded by $\frac{H}{1-\gamma}$ starting from any $(s, h)$, because then we can apply \pref{lem:hitting} and take a union bound over all $K$ episodes.
    
    Note that the expected hitting time (w.r.t $J_k$) is simply the value function with respect to a cost function that is $1$ for all state-action pairs except for $0$ cost in the goal state $g$ and layer $H+1$ (i.e., $c_f = 0$).
    Thus, by \pref{lem:sda}, the expected hitting time starting from $(s,h)$ is bounded by $\frac{H - h + 1}{1 - \gamma} \le \frac{H}{1 - \gamma}$.
\end{proof}

\paragraph{Definition of $\calP_k$} We define $\calP_k=\bigcap_{s,a,h\leq H}\calP_{k,s,a,h}$, where:
\begin{align}
	\calP_{k,s,a,h} &= \left\{P'\in\Lambda_{\calM}: \abr{\barP_{k, s, a}(s') - P'_{s, a, h}(s', h)/\gamma} \leq\epsilon_k(s, a, s'),\right.\notag\\ 
	&\quad\abr{\P_{k, s, a}(s') - P'_{s, a, h}(s', h+1)/(1-\gamma)} \leq\epsilon_k(s, a, s'),\notag\\
	&\quad\left. \abr{\P_{k, s, a}(g) - P'_{s, a, h}(g)} \leq\epsilon_k(s, a, g), \forall s'\in\calS\right\},\label{eq:conf}
\end{align}
where $\epsilon_k(s, a, s')=4\sqrt{\barP_{k, s, a}(s')\alpha'_k(s, a)} + 28\alpha'_k(s, a)$, $\alpha'_k(s,a)=\frac{\iota}{\Np_k(s, a)}$, $\barP_{k, s, a}(s')=\frac{N_k(s, a, s')}{\Np_k(s, a)}$ is the empirical transition, $\Np_k(s,a)=\max\{1,\N_k(s,a)\}$, $\N_k(s, a)$ is the number of visits to $(s, a)$ in episode $j=1, \ldots, k-1$ before $\sigma(\pi_j)$ switches to the fast policy, and $N_k(s, a, s')$ is the number of visits to $(s, a, s')$ in episode $j=1, \ldots, k-1$ before $\sigma(\pi_j)$ switches to the fast policy.

\begin{lemma}
	\label{lem:Pin}
	Under the event of \pref{lem:bound J}, we have $\rP\in \calP_k$ for any $k\in[K]$ with probability at least $1-\delta$.
\end{lemma}
\begin{proof}
	Clearly $\rP\in\Lambda_{\calM}$.
	Moreover, for any $(s, a)\in\SA, s'\in\calS_+$ by \pref{lem:bernstein} and $N_{K+1}(s, a)\leq LK$ under the event of \pref{lem:bound J}, we have with probability at least $1-\frac{\delta}{2S^2A}$,
	\begin{align}
		\abr{P_{s,a}(s') - \P_{k,s,a}(s')} \leq \epsilon_k(s, a, s').\label{eq:P eps}
	\end{align}
	By a union bound, we have \pref{eq:P eps} holds for any $(s, a)\in\SA, s'\in\calS_+$ with probability at least $1-\delta$.
	Then the statement is proved by $\rP_{s,a,h}(s',h)=\gamma P_{s,a}(s')$, $\rP_{s,a,h}(s',h+1)=(1-\gamma)P_{s,a}(s')$, and $\rP_{s,a,h}(g)=P_{s,a}(g)$.
\end{proof}

\begin{lemma}
	\label{lem:conf bound}
	Under the event of \pref{lem:Pin}, for any $P'\in\calP_k$, we have for any $\rs'\in\rcalS_+$:
	\begin{align*}
		\abr{P'_{s,a,h}(\rs') - P_{s, a, h}(\rs')}\leq 8\sqrt{P_{s, a, h}(\rs')\alpha'_k(s, a)} + 136\alpha'_k(s, a) \triangleq \epsilon^{\star}_k(s, a, h, \rs').
	\end{align*}
	For simplicity, we also write $\epsilon^{\star}_k(s, a, h, (s', h'))$ as $\epsilon^{\star}_k(s, a, h, s', h')$ for $(s', h')\in\rcalS$.
\end{lemma}
\begin{proof}
	Under the event of \pref{lem:Pin} and by \pref{eq:P eps}, we have for all $(s, a)\in\SA$, and $s'\in\calS_+$:
	\begin{align*}
		\P_{k, s, a}(s') \leq P_{s, a}(s') + 4\sqrt{\P_{k, s, a}(s')\alpha'_k(s, a)} + 28\alpha'_k(s, a).
	\end{align*}
	Applying $x^2\leq a x+b\implies x\leq a+\sqrt{b}$ with $a=4\sqrt{\alpha'_k(s, a)}$ and $b=P_{s, a}(s')+28\alpha'_k(s, a)$, we have
	$$\sqrt{\P_{k, s, a}(s')} \leq 4\sqrt{\alpha'_k(s, a)}+\sqrt{P_{s, a}(s')+28\alpha'_k(s, a)} \leq \sqrt{P_{s, a}(s')} + 10\sqrt{\alpha'_k(s, a)}.$$
	Substituting this back to the definition of $\epsilon_k$, we have 
	$$\epsilon_k(s, a, s') = 4\sqrt{\P_{k,s,a}(s')\alpha'_k(s, a) } + 28\alpha'_k(s, a)\leq 4\sqrt{P_{s, a}(s')\alpha'_k(s, a)} + 68\alpha'_k(s, a).$$
	Now we start to prove the statement.
	The statement is clearly true for $\rs'=(s', h')$ with $h'\notin\{h, h+1\}$ since the left-hand side equals to $0$.
	Moreover, by the definition of $\calP_k$, \pref{lem:Pin}, and $x\leq\sqrt{x}$ for $x\in(0, 1)$,
	\begin{align*}
		\abr{P'_{s,a,h}(s',h) - P_{s,a,h}(s',h)} &\leq \abr{P'_{s,a,h}(s',h)  - \gamma \P_{k,s,a}(s')} + \abr{\gamma \P_{k,s,a}(s') - P_{s,a,h}(s',h)}\\
		&\leq 2\gamma\epsilon_k(s, a, s') \leq \epsilonstar_k(s, a, h, s', h),
	\end{align*}
	\begin{align*}
		&\abr{P'_{s,a,h}(s',h+1) - P_{s,a,h}(s',h+1)}\\
		&\leq \abr{P'_{s,a,h}(s',h+1) - (1-\gamma) \P_{k,s,a}(s')} + \abr{(1-\gamma) \P_{k,s,a}(s') - P_{s,a,h}(s',h+1)}\\
		&\leq 2(1-\gamma)\epsilon_k(s, a, s') \leq \epsilonstar_k(s, a, h, s', h),
	\end{align*}
	\begin{align*}
		\abr{P'_{s,a,h}(g) - P_{s,a,h}(g)} &\leq \abr{P'_{s,a,h}(g) - \P_{k,s,a}(g)} + \abr{\P_{k,s,a}(g) - P_{s,a,h}(g)}\\ 
		&\leq 2\epsilon_k(s, a, g) \leq 2\epsilonstar_k(s, a, h, g).
	\end{align*}
	This completes the proof.
\end{proof}

\subsection{Approximation of $Q^{\pi,\Gamma(\pi,\calP,c),c}$}
\label{app:compute Gamma}
We show that $Q^{\pi,\Gamma(\pi,\calP,c),c}$ can be approximated efficiently by Extended Value Iteration similar to \citep{jaksch2010near}.
Note that finding $\Gamma(\pi,\calP,c)$ is equivalent to computing the optimal policy in an augmented MDP $\rcalM$ with state space $\rcalS$ and extended action space $\calP$, such that for any extended action $P\in\calP$, the cost at $((s, h), P)$ is $\sum_a\pi(a|s, h)c(s, a, h)$, and the transition probability to $\rs'\in\rcalS_+$ is $\sum_a\pi(a|s, h)P_{s, a, h}(\rs')$.
In this work, we have $\calP\in\{\calP_k\}_{k=1}^K$, and $\calP_k=\bigcap_{s, a, h}\calP_{k,s,a,h}$, where $\calP_{k,s,a,h}$ is a convex set that specifies constraints on $((s, h), a)$.
In other words, $\calP_k$ is a product of constraints on each $((s, h), a)$ (note that $\Lambda_{\calM}$ can also be decomposed into shared constraints on $P_{s,a,H+1}$ and independent constraints on each $s, a, h\leq H$).
Thus, any policy in $\rcalM$ can be represented by an element $P\in\calP$.
We can now perform value iteration in $\rcalM$ to approximate $Q^{\pi,\Gamma(\pi,\calP,c),c}$.
The Bellman operator of $\rcalM$ is $\calT_0$ defined in \pref{eq:EVI} with $\min$ operator replaced by $\max$ operator. 
Also note that $\rcalM$ is an SSP instance where all policies are proper.
Thus, $V^{\pi,\Gamma(\pi,\calP,c),c}$ is the unique fixed point of $\calT_0$~\citep{bertsekas2013stochastic}.
It is straightforward to show that \pref{lem:EVI} still holds with $\min$ operator replaced by $\max$ operator in \pref{eq:EVI} and let $V^0(s, H+1)=\max_ac(s,a,H+1)$.
Thus, we can approximate $V^{\pi,\Gamma(\pi,\calP,c),c}$ efficiently.

Now suppose after $n$ iterations of modified \pref{eq:EVI}, we obtain $V^n$ such that $\norm{V^n - V^{\pi,\Gamma(\pi,\calP,c),c}}_{\infty}\leq\epsilon$.
Then we can simply use $Q(s, a, h)=c(s, a, h) + \min_{P\in\calP}P_{s,a,h}V^n$ to approximate $Q^{\pi,\Gamma(\pi,\calP,c),c}$, since
\begin{align*}
	&\abr{Q(s, a, h) - Q^{\pi,\Gamma(\pi,\calP,c),c}(s, a, h)} \overset{\text{(i)}}{=} \abr{\min_{P\in\calP}P_{s,a,h}V^n - \min_{P\in\calP}P_{s,a,h}V^{\pi,\Gamma(\pi,\calP,c),c} }\\
	&\leq \max_{P\in\calP}\abr{ P_{s,a,h}(V^n - V^{\pi,\Gamma(\pi,\calP,c),c}) } \leq \norm{V^n - V^{\pi,\Gamma(\pi,\calP,c),c}}_{\infty} \leq \epsilon,
\end{align*}
where (i) is by the definition of $\Gamma(\pi, \calP, c)$.
In this work, setting $\epsilon=1/K$ is enough for obtaining the desired regret bounds.
\pref{lem:EVI} (modified) then implies that $\tilo{\Tmax}$ iterations of modified \pref{eq:EVI} suffices.

%% file: app-sda.tex

\subsection{\pfref{lem:approx}}
\begin{proof}
	We only prove the statement for adversarial environment, and the statement for stochastic environment follows directly from setting $c_1=\cdots c_K=c$.
	By \pref{lem:sda}, we have $V^{\roptpi, P, c_k}(s, 1)\leq V^{\optpi, P, c_k}(s) + \frac{1}{K}$ for any $k\in[K]$.
	Now by \pref{lem:hitting} and the fact that the expected hitting time of fast policy is upper bounded by $D$, we have with probability at least $1-\delta$, the learner reaches the goal within $J_k+c_f$ steps for each episode $k$.
	Thus by a union bound, we have with probability at least $1-\delta$, $\sumk\sum_{i=1}^{I_k}c^k_i\leq\sumk\rbr{\sumit c^k_i + \rc^k_{J_k+1}}$.
	Putting everything together, we get:
	\begin{align*}
		R_K &= \sumk\rbr{\sumi c^k_i - V^{\optpi,P,c_k}(s^k_1)} \leq \sumk\rbr{\sumit c^k_i + \rc^k_{J_k+1} - V^{\roptpi, P, c_k}(s^k_1, 1)} + \tilO{1}\\
		&= \rR_K + \tilO{1}.
	\end{align*}
	This completes the proof.
\end{proof}

%% file: app-sto.tex

\paragraph{Extra Notations} Define optimistic transitions $\tilP_k=\Gamma(\pi_k, \calP_k, \tilc_k)$ and $P_k=\Gamma(\pi_k, \calP_k, \hatc_k)$, such that $\tilQ_k=Q^{\pi_k,\tilP_k,\tilc_k}$ and $\hatQ_k=Q^{\pi_k,P_k,\hatc_k}$.
Also define $\tilq_k=q_{\pi_k,\tilP_k}$ and $Q_k=Q^{\pi_k, P, \hatc_k}$.

\subsection{Cost Estimation}
We provide more details on the definition of $\hatc_k$ for the subsequent analysis.
Recall that $\hatc_k(s,a) =\max\{0, \barc_k(s, a) - 2\sqrt{\barc_k(s, a)\alpha_k(s, a)} - 7\alpha_k(s, a)\}$.
Here, $\barc_k(s, a)=\frac{C_k(s, a)}{\frNp_k(s, a)}$, where $C_k(s, a)$ is the accumulated costs that are observed at $(s,a)$ in episode $j=1, \ldots, k-1$ before $\sigma(\pi_j)$ switches to the fast policy, $\alpha_k(s, a)=\frac{\iota}{\frN^+_k(s, a)}$ (recall $\iota=\ln(2SALK/\delta)$), $\frN^+_k(s, a)=\max\{1, \frN_k(s, a)\}$, and $\frN_k$ is the number of times the learner observes cost at $(s, a)$ in episode $j=1, \ldots, k-1$ before $\sigma(\pi_j)$ switches to the fast policy.
The definition of $C_k$ and $\frN_k$ depends on the type of cost feedback.
For stochastic costs, $C_k(s, a)=\sum_{j=1}^{k-1}\sum_{i=1}^{J_k}c^j_i\Ind\{s^j_i=s, a^j_i=a\}$ and $\frN_k=N_k(s, a)$.
For stochastic adversary, $C_k(s, a)=\sum_{j=1}^{k-1}m_j(s, a)c_j(s, a)$, where $m_k(s, a)$ is the indicator of whether $c_k(s, a)$ is observed in episode $k$ before $\sigma(\pi_k)$ switches to the fast policy, and $\frN_k(s, a)=M_k(s, a)\triangleq\sum_{j=1}^{k-1}m_j(s, a)$.

Below we show a lemma quantifying the cost estimation error.
\begin{lemma}
	\label{lem:cost bound}
	Under the event of \pref{lem:bound J}, we have with probability at least $1-\delta$,
	$$0 \leq c(s, a) - \hatc_k(s, a) \leq 4\sqrt{\hatc_k(s, a)\alpha_k(s, a)} + 34\alpha_k(s, a),$$
	for all definitions of $\hatc_k$.
\end{lemma}
\begin{proof}
	Only prove the stochastic cost case and the stochastic adversary case follows similarly.
	Note that under the event of \pref{lem:bound J}, $\N_{k+1}(s, a)\leq LK$.
	Applying \pref{lem:bernstein} with $X_k=c_k(s, a)$ for each $(s, a)\in\SA$ and then by a union bound over all $(s, a)\in\SA$, we have with probability at least $1-\delta$, for all $k\in[K]$:
	\begin{align*}
		|\bar{c}_k(s, a) - c(s, a)| \leq 2\sqrt{\alpha_k(s, a)\bar{c}_k(s, a)} + 7\alpha_k(s, a).
	\end{align*}
	Hence, $c(s, a) \geq \hatc_k(s, a)$ by the definition of $\hatc_k$.
	Applying $x^2\leq ax+b \implies x\leq a+\sqrt{b}$ with $x=\sqrt{\bar{c}_k(s, a)}$ to the inequality above (ignoring the absolute value operator), we obtain
	\begin{align*}
		\sqrt{\bar{c}_k(s, a)} \leq 2\sqrt{\alpha_k(s, a)} + \sqrt{c(s, a) + 7\alpha_k(s, a)} \leq \sqrt{c(s, a)} + 5\sqrt{\alpha_k(s, a)},
	\end{align*}
	Therefore, $2\sqrt{\alpha_k(s, a)\bar{c}_k(s, a)}+7\alpha_k(s, a)\leq 2\sqrt{\alpha_k(s, a)c(s, a)}+17\alpha_k(s, a)$, and
	\begin{align*}
		 c(s, a) - \hatc_k(s, a) &= c(s, a) - \bar{c}_k(s, a) + \bar{c}_k(s, a) - \hatc_k(s, a)\\ 
		 &\leq 2\cdot (2\sqrt{\alpha_k(s, a)\bar{c}_k(s, a)}+7\alpha_k(s, a)) \leq 4\sqrt{\alpha_k(s, a)c(s, a)} + 34\alpha_k(s, a).
	\end{align*}
	This completes the proof.
\end{proof}

\subsection{Main Results for Stochastic Costs and Stochastic Adversary}

We first show a general regret bound agnostic to the feedback type.
\begin{theorem}
	\label{thm:SF}
	Assuming that there exists a constant $G$ such that for any $s, h$:
	\begin{align*}
		\sum_{k=1}^{K-1}\inner{\pi_{k+1}(\cdot|s, h)}{d\tilQ_k(s, \cdot, h)} \leq G.
	\end{align*}
	Then, \pref{alg:po} in stochastic environments with $\lambda\leq \min\{1/\Tmax, \sqrt{S^2A/K}\}$ ensures with probability at least $1-22\delta$,
	\begin{align*}
		\rR_K &= \tilO{ \sumk\sumib (c^k_i - \hatc_k(\rs^k_i, a^k_i)) - \sumk\inner{q_k}{e_k} + \sumk\inner{\optq}{e_k} + \frac{S^2A}{\lambda} }\\
		&\qquad + \tilO{ \sqrt{S^2A\sumk\inner{q_k}{e_k\circ Q^{\pi_k, P, e_k}}} + S^4A^{2.5}\Tmax^3 + \lambda\sumk\inner{\optq}{c\circ Q^{\roptpi, P, \hatc_k}} }\\
		&\qquad + \tilO{\frac{\Tmax}{\eta} + \Tmax G + \lambda\sumk\inner{\optq}{Q^{\roptpi, P,e_k}} }.
	\end{align*}
\end{theorem}
\begin{proof}
	For notational convenience, define $\omega=S^4A^{2.5}\Tmax^3$.
	By $\inner{\optq}{\hatc_k}\leq\inner{\optq}{c}$ (\pref{lem:cost bound}) and \pref{lem:bound J} (under which $n_k=\bar{n}_k$), we have with probability at least $1-2\delta$,
	\begin{align*}
		\rR_K &= \sumk\rbr{\sumib c^k_i + \rc^k_{J_k+1} - V^{\roptpi, P, c}(\rs^k_1)} \leq \sumk\rbr{\sumib c^k_i + \rc^k_{J_k+1} - V^{\roptpi,P,\hatc_k}(\rs^k_1)}\\
		&\leq \sumk\sumib(c^k_i - \hatc_k(\rs^k_i, a^k_i)) + \sumk\inner{\bar{n}_k - \optq}{\hatc_k}.
	\end{align*}
	For the second term, by the definition of $\tilc_k$,
	\begin{align*}
		&\sumk\inner{\bar{n}_k - \optq}{\hatc_k} = \sumk\inner{\bar{n}_k - q_k}{\hatc_k} + \sumk\inner{q_k - \tilq_k}{\tilc_k} + \sumk\inner{\tilq_k - \optq}{\tilc_k} \\
		&\qquad - \sumk\inner{q_k}{e_k} + \sumk\inner{\optq}{e_k} - \lambda\sumk\inner{q_k}{\hatc_k\circ \hatQ_k} + \lambda\sumk\inner{\optq}{\hatc_k\circ \hatQ_k}  \\
		&\leq \underbrace{\sumk\inner{\bar{n}_k-q_k}{\hatc_k} + \sumk\inner{q_k - \tilq_k}{\tilc_k} - \lambda\sumk\inner{q_k}{\hatc_k\circ Q_k}}_{\xi_1}\\
		&\qquad - \sumk\inner{q_k}{e_k} + \sumk\inner{\optq}{e_k}+ \underbrace{\lambda\sumk\inner{q_k}{\hatc_k\circ(Q_k-\hatQ_k)} }_{\xi_2}\\
		&\qquad + \lambda\sumk\inner{\optq}{c\circ Q^{\roptpi, P, \hatc_k}} + \underbrace{\sumk\inner{\tilq_k - \optq}{\tilc_k} + \lambda\sumk\inner{\optq}{c\circ(\hatQ_k - Q^{\roptpi, P, \hatc_k})} }_{\xi_3}. \tag{$\hatc_k(s, a, h)\leq c(s, a, h)$} 
	\end{align*}
	For $\xi_1$, with probability at least $1-17\delta$:
	\begin{align*}
		&\sumk\inner{\bar{n}_k-q_k}{\hatc_k} + \sumk\inner{q_k - \tilq_k}{\tilc_k} - \lambda\sumk\inner{q_k}{\hatc_k\circ Q_k} \leq \tilO{\sqrt{\sumk\inner{q_k}{\hatc_k\circ Q_k}} + SA\Tmax}\\
		&+ \sumk\inner{q_k - \tilq_k}{(1+\lambda\hatQ_k)\circ\hatc_k} + \sumk\inner{q_k - \tilq_k}{e_k} - \lambda\sumk\inner{q_k}{\hatc_k\circ Q_k} \tag{$\E_k[\bar{n}_k(s,a,h)]\leq q_k(s,a,h)$, \pref{lem:freedman}, \pref{lem:var}, and $\bar{n}_k(s, a, h)\leq L=\tilo{\Tmax}$}\\
		&= \tilO{ \sqrt{S^2A\sumk\inner{q_k}{\hatc_k\circ Q_k}} + \sqrt{S^2A\sumk\inner{q_k}{e_k\circ Q^{\pi_k, P, e_k}}} + \omega} - \lambda\sumk\inner{q_k}{\hatc_k\circ Q_k} \tag{\pref{lem:q-qk} and $(1+\lambda\hatQ_k(s, a, h))\hatc_k(s, a, h)=\tilo{\hatc_k(s, a, h)}$}\\
		&= \tilO{\frac{S^2A}{\lambda} + \sqrt{S^2A\sumk\inner{q_k}{e_k\circ Q^{\pi_k, P, e_k}}} + \omega}. \tag{AM-GM inequality}
	\end{align*}
	For $\xi_2$, by \pref{lem:value diff} and \pref{lem:conf bound}, with probability at least $1-2\delta$,
	\begin{align}
		&Q_k(s, a, h) - \hatQ_k(s, a, h) = \sum_{s', a', h'}q_{k,(s,a,h)}(s', a', h')(P_{s', a', h'} - P_{k, s', a', h'})V^{\pi_k,P_k,\hatc_k} \notag\\
		&=\tilO{ \sum_{s', a'}q_{k, (s, a, h)}(s', a')\rbr{\frac{\sqrt{S}\Tmax}{\sqrt{\Np_k(s', a')}} + \frac{S\Tmax}{\Np_k(s', a')} } }. \label{eq:Qk-hatQ}
	\end{align}
	By $q_k(s,a,h)=\frac{x_k(s,a,h)}{1-y_k(s,a,h)}$ and $y_k(s,a,h)\leq\gamma=1-\frac{1}{2\Tmax}$, we have
	\begin{align}
		\sum_{s, a, h\leq H}q_k(s, a, h)q_{k,(s, a, h)}(s', a') &\leq 2\Tmax\sum_{s, a, h\leq H}x_k(s, a, h)q_{k,(s, a, h)}(s', a')\notag\\
		&\leq 2\Tmax\sum_{s, a, h\leq H}q_k(s', a')=2\Tmax SAHq_k(s', a').\label{eq:qq}
	\end{align}
	Therefore, with probability at least $1-\delta$,
	\begin{align*}
		\xi_2&=\lambda\sumk\inner{q_k}{\hatc_k\circ(Q_k-\hatQ_k)}\\
		&= \tilO{\lambda\sumk\sum_{s,a,h}q_k(s,a,h)\sum_{s', a'}q_{k, (s, a, h)}(s', a')\rbr{\frac{\sqrt{S}\Tmax}{\sqrt{\Np_k(s', a')}} + \frac{S\Tmax}{\Np_k(s', a')} }  }\\
		&= \tilO{\lambda \Tmax^2S^{3/2}A\sum_{s', a'}\sumk \frac{q_k(s', a')}{\sqrt{\Np_k(s', a')}} + \lambda\Tmax^2S^2A\sum_{s', a'}\sumk \frac{q_k(s', a')}{\Np_k(s', a')} } \tag{\pref{eq:qq}}\\
		&= \tilO{ \lambda \Tmax^2S^{3/2}A\sqrt{SA\Tmax K} + \lambda S^3A^2\Tmax^3 } = \tilO{ \omega }. \tag{\pref{lem:n sum} and $\sum_{s, a}q_k(s, a)=\tilo{\Tmax}$}
	\end{align*}
	For $\xi_3$, first note that $\norm{\tilQ_1}_{\infty}= \tilo{\Tmax}$ under all definitions of $\tilc_k$, and by \pref{lem:value diff}:
	\begin{align*}
		&\sumk\inner{\tilq_k - \optq}{\tilc_k} = \sumk\sum_{s, h}\optq(s, h)\sum_a\rbr{ \pi_k(a|s, h) - \optpi(a|s, h) }\tilQ_k(s, a, h)\\ 
		&\qquad + \sumk\sum_{s, a, h}\optq(s, a, h)\rbr{ \tilQ_k(s, a, h) - \tilc_k(s, a, h) - P_{s, a, h}V^{\pi_k,\tilP_k,\tilc_k} }\\
		&= \tilO{\frac{\T}{\eta} + \T G + \Tmax^2 }. \tag{\pref{lem:po}, the definition of $\tilP_k$ and $\sum_{s,a,h}\optq(s,a,h)=\tilo{\T}$ by \pref{lem:sda}}
	\end{align*}
	Next, note that
	\begin{align*}
		&\sumk (\hatQ_k(s, a, h) - Q^{\roptpi, P, \hatc_k}(s, a, h)) \leq \sumk (Q^{\pi_k,\tilP_k, \hatc_k}(s, a, h) - Q^{\roptpi, P, \hatc_k}(s, a, h)) \tag{$P_k, \tilP_k \in\calP_k$} \\
		&\leq \sumk \rbr{Q^{\pi_k,\tilP_k, \tilc_k}(s, a, h) - Q^{\roptpi, P, \tilc_k}(s, a, h)} + \sumk Q^{\roptpi, P, \lambda\hatQ_k+e_k}(s, a, h) \tag{definition of $\tilc_k$}
	\end{align*}
	Also note that $\lambda\sumk\inner{\optq}{Q^{\roptpi,P,\lambda\hatQ_k}}=\tilo{\lambda^2\Tmax^3K}=\tilo{S^2A\Tmax^3}$ by $\lambda\leq \sqrt{S^2A/K}$.
	Thus,
	\begin{align*}
		&\lambda\sumk\inner{\optq}{c\circ(\hatQ_k - Q^{\roptpi, P, \hatc_k})}\\
		&= \tilO{ \lambda\sumk\inner{\optq}{c\circ(\tilQ_k - Q^{\roptpi, P, \tilc_k})} + \lambda\sumk\inner{\optq}{Q^{\roptpi, P, e_k}} + S^2A\Tmax^3  }.
	\end{align*}
	Now by \pref{lem:value diff} and the definition of $\tilP_k$:
	\begin{align}
		&\sumk(\tilQ_k(s, a, h) - Q^{\roptpi, P, \tilc_k}(s, a, h)) \label{eq:tilQ-Q}\\
		&\leq \sumk \sum_{s'',h''}P_{s,a,h}(s'', h'')\sum_{s', a', h'}\optq_{(s'', h'')}(s', h')\rbr{\pi_k(a'|s', h') - \roptpi(a'|s', h')}\tilQ_k(s', a', h') \notag\\
		&=\tilO{ \frac{\Tmax}{\eta} + \Tmax G + \Tmax^2 }. \tag{\pref{lem:po}}
	\end{align}
	Thus, by $\lambda\Tmax\leq 1$, we have $\lambda\sumk\inner{\optq}{c\circ(\tilQ_k - Q^{\roptpi, P, \tilc_k})} = \tilo{ \frac{\Tmax}{\eta} + \Tmax G + \Tmax^2 }$.
	Putting everything together completes the proof.
\end{proof}

\subsection{\pfref{thm:SF-SC}}\label{app:proof_SF-SC}
\begin{proof}
	By \pref{lem:diff} with $\frn_k=n_k$, $\frN_k=N_k$, and $e_k(s, a, h)=0$, with probability at least $1-2\delta$:
	\begin{align*}
		\sum_{k=1}^{K-1}\inner{\pi_{k+1}(\cdot|s, h)}{d\tilQ_k(s, \cdot, h)} &= \tilO{\Tmax^2\sumk\sum_{s', a'}\frac{Sn_k(s', a')}{\Np_k(s', a')} + \lambda\eta \Tmax^4K }\\ 
		&= \tilO{ S^2A\Tmax^3 + \Tmax^2(S^2AK)^{1/4} }. \tag{definition of $\lambda$ and $\eta$, $n_k(s,a)=\bar{n}_k(s,a)$ under the event of \pref{lem:bound J}, and \pref{lem:n sum}}
	\end{align*}
	Thus, by \pref{thm:SF}, \pref{lem:SC}, definition of $\lambda$, and replacing $G$ by the bound above, we have with probability at least $1-28\delta$,
	\begin{align*} 
		\rR_K &= \tilO{ \sqrt{SA\sumk\sumit c^k_i} + \frac{S^2A}{\lambda} + \Tmax^3(S^2AK)^{1/4} + S^4A^{2.5}\Tmax^4 + \lambda\sumk\inner{\optq}{c\circ Q^{\roptpi, P, \hatc_k}}}\\
		&= \tilO{\sqrt{SA\sumk\sumit c^k_i} + \B S\sqrt{AK} + \Tmax^3(S^2AK)^{1/4} + S^4A^{2.5}\Tmax^4}. \tag{\pref{lem:qcQ}}
	\end{align*}
	Now by $\rR_k=\sumk\sumit c^k_i - K\cdot V^{\roptpi,P,c}(\rs^k_1)$ and \pref{lem:quad}, we have $\sumk\sumit c^k_i=\tilo{\B K}$.
	Plugging this back, we get $\rR_K=\tilo{\B S\sqrt{AK} + \Tmax^3(S^2AK)^{1/4} + S^4A^2\Tmax^4}$.
	Applying \pref{lem:approx} then completes the proof.
\end{proof}

\subsection{\pfref{thm:SA-F}}\label{app:proof_SA-F}
\begin{proof}
	First note that with probability at least $1-3\delta$,
	\begin{align*}
		\sumk\norm{de_k}_1 &\leq \sumk\sum_{s, a, h\leq H}\abr{\sqrt{\frac{\hatc_k(s, a, h)}{k}} - \sqrt{\frac{\hatc_{k+1}(s, a, h)}{k+1}} } + \beta'\sumk\sum_{s, a, h\leq H}\abr{d\hatQ_k(s, a, h)}\\
		&=\tilO{ S^3A^2\Tmax^2 + S^{1/2}A^{3/4}\Tmax K^{1/4} },
	\end{align*}
	where in the last inequality we apply
	\begin{align*}
		&\sumk\sum_{s, a, h\leq H}\abr{\sqrt{\frac{\hatc_k(s, a, h)}{k}} - \sqrt{\frac{\hatc_{k+1}(s, a, h)}{k+1}} }\\
		&\leq \sumk\sum_{s, a, h\leq H}\rbr{\frac{1}{\sqrt{k}} - \frac{1}{\sqrt{k+1}}} + \sumk\sum_{s, a, h\leq H}\frac{\sqrt{|\hatc_k(s, a, h) - \hatc_{k+1}(s, a, h)|}}{\sqrt{k+1}}\tag{add and subtract $\sqrt{\hatc_k(s,a,h)/(k+1)}$, and $|\sqrt{a}-\sqrt{b}|\leq\sqrt{|a-b|}$}\\
		&= \tilO{SA + \sqrt{ \sumk\sum_{s, a, h\leq H}\frac{1}{k+1} }\sqrt{ \sumk\sum_{s, a, h\leq H}\frac{m_k(s, a)}{\Mp_k(s, a)} }} = \tilO{SA},\tag{Cauchy-Schwarz inequality, \pref{lem:diff}, and \pref{lem:n sum}}
	\end{align*}
	and by \pref{lem:bound J},
	\begin{align*}
		&\beta'\sumk\sum_{s, a, h\leq H}\abr{d\hatQ_k(s, a, h)}\\ 
		&= \tilO{\beta'\sumk\sum_{s, a, h\leq H}\rbr{ \Tmax^2\sum_{s', a'}\frac{Sn_k(s', a')}{\Np_k(s', a')} + \Tmax\sum_{s', a'}\frac{m_k(s', a')}{\Mp_k(s', a')}  + \eta \Tmax^3 } }\tag{\pref{lem:diff}}\\
		&= \tilO{ \beta' S^3A^2\Tmax^3 + \beta'\eta SA\Tmax^3 K } = \tilO{ S^3A^2\Tmax^2 + S^{1/2}A^{3/4}\Tmax K^{1/4} }. \tag{\pref{lem:n sum}}
	\end{align*}
	Moreover, by \pref{lem:diff} with $\frn_k=m_k$, $\frN_k=M_k$, and $\lambda\leq\frac{1}{\Tmax}$, we have with probability at least $1-\delta$:
	\begin{align}
		&\sum_{k=1}^{K-1}\inner{\pi_{k+1}(\cdot|s, h)}{d\tilQ_k(s, \cdot, h)}\notag\\ 
		&= \tilO{\Tmax^2\sumk\sum_{s', a'}\frac{Sn_k(s', a')}{\Np_k(s', a')} + \lambda \Tmax^2\sumk\sum_{s', a'}\frac{m_k(s', a')}{\Mp_k(s', a')} + \lambda\eta \Tmax^4K + \Tmax\sumk\norm{de_k}_1 }\notag\\
		&= \tilO{ S^3A^2\Tmax^3 + \Tmax^2(S^2A^3K)^{1/4} },\label{eq:G full}
	\end{align}
	where the last step is by \pref{lem:n sum}, the definition of $\eta$ and $\lambda$, and the bound on $\sumk\norm{de_k}_1$.
	Moreover, by \pref{lem:SA-F} and definition of $e_k$, we have with probability at least $1-16\delta$:
	\begin{align*}
		&\sumk\sumib\rbr{c^k_i - \hatc_k(\rs^k_i, a^k_i)} - \sumk\inner{q_k}{e_k}= \tilO{\beta'\sumk\inner{q_k}{Q_k-\hatQ_k} + \frac{1}{\beta'} + \sqrt{S^3A^3\Tmax^3}}\\
		&= \tilO{S^3A^2\Tmax^3 + \sqrt{D\T K} },\tag{\pref{eq:Qk-hatQ}, \pref{eq:qq} similar to bounding $\xi_2$, and the definition of $\beta'$}\\
		&\sumk\inner{\optq}{e_k} = \tilO{ \sumk\sum_{s, a}\optq(s, a)\sqrt{\frac{c(s, a)}{k}} + \beta'\sumk\inner{\optq}{\hatQ_k}} \tag{\pref{lem:cost bound}}\\
		&\overset{\text{(i)}}{=} \tilO{ \sqrt{D\T K} + S^3A^2\Tmax^4 },\\
		&\sqrt{S^2A\sumk\inner{q_k}{e_k\circ Q^{\pi_k, P, e_k}}} = \tilO{ \sqrt{S^2A\Tmax^4} },\\
		&\lambda\sumk\inner{\optq}{Q^{\roptpi, P,e_k}} = \tilO{\lambda\Tmax^2\sqrt{K} + \lambda\beta'\Tmax^3K} = \tilO{ \sqrt{S^2A}\Tmax^3 },
	\end{align*}
	where (i) is by
	\begin{align*}
		\sumk\sum_{s, a}\optq(s, a)\sqrt{\frac{c(s, a)}{k}} &= \tilO{\sqrt{\sumk\sum_{s, a}\optq(s, a)c(s, a)}\sqrt{\sumk\sum_{s, a}\frac{\optq(s, a)}{k}}}= \tilO{ \sqrt{D\T K} },\tag{Cauchy-Schwarz inequality}
	\end{align*}
	definition of $\beta'$, and
	\begin{align*}
		&\beta'\sumk\inner{\optq}{\hatQ_k} = \beta'\sumk\inner{\optq}{\hatQ_k - Q^{\roptpi, P, \hatc_k}} + \beta'\sumk\inner{\optq}{Q^{\roptpi, P, \hatc_k}}\\ 
		&= \tilO{ \beta'\sumk\sum_{s, a, h}\optq(s, a, h) \rbr{Q^{\pi_k,\tilP_k, \tilc_k}(s, a, h) - Q^{\roptpi, P, \tilc_k}(s, a, h)} }\\ 
		&\qquad+ \tilO{ \beta'\sum_{s, a, h}\sumk \optq(s, a, h) Q^{\roptpi, P, \lambda\hatQ_k+e_k}(s, a, h) + \sqrt{D\T K}}, \tag{$\sum_{s,a,h}\optq(s,a,h)=\bigo{\T}$, and $\norm{Q^{\roptpi, P, \hatc_k}}_{\infty}=\bigo{D}$}\\
		&= \tilO{ \frac{\beta'\Tmax^2}{\eta} + \beta'\Tmax^2G + \beta'\Tmax^3 + (\lambda\beta'+{\beta'}^2)\Tmax^3K + \beta'\Tmax^2\sqrt{K} + \sqrt{D\T K} }, \tag{\pref{eq:tilQ-Q}}\\
		&= \tilO{ S^3A^2\Tmax^4 + \sqrt{D\T K}}.\tag{replace $G$ by \pref{eq:G full}}
	\end{align*}
	Thus, by \pref{thm:SF}, \pref{lem:qcQ sa}, and definition of $\eta,\lambda$, we have with probability at least $1-22\delta$,
	\begin{align*}
		\rR_K &= \tilO{ \sqrt{D\T K} + DS\sqrt{AK} + \Tmax^3(S^2A^3K)^{1/4} + S^4A^{2.5}\Tmax^4 }.
	\end{align*}
	Applying \pref{lem:approx} completes the proof.
\end{proof}

\subsection{\pfref{thm:SA-B}}\label{app:proof_SA-B}
\begin{proof}
	By \pref{lem:diff} with $\frn_k=m_k$, $\frN_k=M_k$, and $\lambda\leq\frac{1}{\Tmax}$, we have with probability at least $1-2\delta$:
	\begin{align}
		&\sum_{k=1}^{K-1}\inner{\pi_{k+1}(\cdot|s, h)}{d\tilQ_k(s, \cdot, h)} \label{eq:G}\\
		&= \tilO{\Tmax^2\sumk\sum_{s', a'}\frac{Sn_k(s', a')}{\Np_k(s', a')} + \lambda \Tmax^2\sumk\sum_{s', a'}\frac{m_k(s', a')}{\Mp_k(s', a')} + \lambda\eta \Tmax^4K + \Tmax\sumk\norm{de_k}_1 } \notag\\
		&= \tilO{ S^3A^2\Tmax^3 + \Tmax^2 SA^{5/4}K^{1/4} }, \tag{definition of $\eta$ and \pref{lem:n sum}}
	\end{align}
	where in the last step we apply
	\begin{align*}
		&\sumk\norm{de_k}_1 = \beta\sumk\sum_{s, a, h\leq H}\abr{d\hatQ_k(s, a, h)}\\ 
		&= \tilO{\beta\sumk\sum_{s, a, h\leq H}\rbr{ \Tmax^2\sum_{s', a'}\frac{Sn_k(s', a')}{\Np_k(s', a')} + SA\Tmax\frac{m_k(s, a)}{\Mp_k(s, a)}  + \eta \Tmax^3 } }\tag{\pref{lem:diff}}\\
		&= \tilO{ \beta S^3A^2\Tmax^3 + \beta\eta SA\Tmax^3 K } = \tilO{ S^3A^2\Tmax^2 + SA^{5/4}\Tmax K^{1/4} }. \tag{\pref{lem:n sum}}
	\end{align*}
	By \pref{lem:SA-B} and the definition of $e_k$, we have with probability at least $1-11\delta$:
	\begin{align*}
		&\sumk\sumib\rbr{c^k_i - \hatc_k(\rs^k_i, a^k_i)} - \sumk\inner{q_k}{e_k} = \tilO{\beta\sumk\inner{q_k}{Q_k - \hatQ_k} + \frac{SA}{\beta} + \sqrt{S^3A^3\Tmax^3}}\\
		&= \tilO{ \sqrt{SAD\T K} + S^{2.5}A^2\Tmax^3}, \tag{\pref{eq:Qk-hatQ}, \pref{eq:qq} similar to bounding $\xi_2$, and the definition of $\beta$}\\
		&\sumk\inner{\optq}{e_k} \leq \beta\sumk\inner{\optq}{\hatQ_k - Q^{\roptpi, P, \hatc_k}} + \beta\sumk\inner{\optq}{Q^{\roptpi, P, \hatc_k}}\\ 
		&= \tilO{ \beta\sumk\sum_{s, a, h}\optq(s, a, h) \rbr{Q^{\pi_k,\tilP_k, \tilc_k}(s, a, h) - Q^{\roptpi, P, \tilc_k}(s, a, h)} }\\ 
		&\qquad+ \tilO{ \beta\sum_{s, a, h}\sumk \optq(s, a, h) Q^{\roptpi, P, \lambda\hatQ_k+e_k}(s, a, h) + \sqrt{SAD\T K}}, \tag{$\sum_{s,a,h}\optq(s,a,h)=\bigo{\T}$, and $\norm{Q^{\roptpi, P, \hatc_k}}_{\infty}=\bigo{D}$}\\
		&= \tilO{ \frac{\beta\Tmax^2}{\eta} + \beta\Tmax^2G + \beta \Tmax^3 + (\lambda\beta+\beta^2)\Tmax^3K + \sqrt{SAD\T K} }, \tag{\pref{eq:tilQ-Q}}\\
		&= \tilO{ S^3A^2\Tmax^4 + \sqrt{SAD\T K}},  \tag{replace $G$ by \pref{eq:G}}\\
		&\sqrt{S^2A\sumk\inner{q_k}{e_k\circ Q^{\pi_k, P, e_k}}} = \tilO{ \sqrt{\beta^2S^2A\Tmax^4K} } = \tilO{ \sqrt{S^3A^2\Tmax^4} },\\
		&\lambda\sumk\inner{\optq}{Q^{\roptpi, P,e_k}} = \tilO{\lambda\beta \Tmax^3K} = \tilO{S^{3/2}A\Tmax^3}.
	\end{align*}
	Thus, by \pref{thm:SF}, definition of $\eta$, $\lambda$, and $\beta$, and \pref{lem:qcQ sa}, with probability at least $1-22\delta$,
	\begin{align*}
		\rR_K &= \tilO{ \sqrt{SAD\T K} + DS\sqrt{AK} + \Tmax^3SA^{5/4}K^{1/4} + S^4A^{2.5}\Tmax^4 }.
	\end{align*}
	Applying \pref{lem:approx} completes the proof.
\end{proof}

\subsection{Extra Lemmas for \pref{sec:sto}}

\begin{lemma}
	\label{lem:SC}
	Under stochastic costs, we have with probability at least $1-6\delta$:
	$$\sumk\sumib(c^k_i - \hatc_k(\rs^k_i, a^k_i)) = \tilO{\sqrt{SA\sumk\sumit c^k_i} + SA\Tmax}.$$
\end{lemma}
\begin{proof}
	First note that:
	\begin{align*}
		\sumk\sumib(c^k_i - \hatc_k(\rs^k_i, a^k_i)) = \sumk\sumib(c^k_i - c(s^k_i, a^k_i)) + \sumk\sumib(c(s^k_i, a^k_i) - \hatc_k(s^k_i, a^k_i)).
	\end{align*}
	For the first term, by \pref{lem:freedman} and \pref{lem:e2r}, we have with probability at least $1-2\delta$,
	\begin{align*}
		\sumk\sumib(c^k_i - c(s^k_i, a^k_i)) &= \tilO{\sqrt{\sumk\sumib \E[(c^k_i)^2|s^k_i, a^k_i]}} = \tilO{\sqrt{\sumk\sumib c(s^k_i, a^k_i) }}\\
		&=\tilO{\sqrt{\sumk\sumit c^k_i}}.
	\end{align*}
	For the second term, with probability at least $1-4\delta$,
	\begin{align*}
		&\sumk\sumib(c(s^k_i, a^k_i) - \hatc_k(\rs^k_i, a^k_i)) = \tilO{ \sumk\sumib\rbr{ \sqrt{\frac{c(s^k_i, a^k_i)}{\Np_k(s^k_i, a^k_i)}} + \frac{1}{\Np_k(s^k_i, a^k_i)} } }\tag{\pref{lem:cost bound} and $\hatc_k(s, a)\leq c(s, a)$}\\
		&= \tilO{\sumsa\sumk \rbr{\bar{n}_k(s, a)\sqrt{\frac{c(s, a)}{\Np_k(s, a)}} + \frac{\bar{n}_k(s, a)}{\Np_k(s, a)}} } \\
		&= \tilO{\sqrt{SA\sumk\sumib c(s^k_i, a^k_i)} + SA\Tmax}\tag{\pref{lem:n sum} and $J_k=\J_k$}\\ 
		&= \tilO{\sqrt{SA\sumk\sumib c^k_i} + SA\Tmax}. \tag{\pref{lem:e2r}}
	\end{align*}
	This completes the proof.
\end{proof}

\begin{lemma}
	\label{lem:SA-F}
	Under stochastic adversary with full information, with probability at least $1-8\delta$,
	\begin{align*}
		&\sumk\sumib(c^k_i - \hatc_k(\rs^k_i, a^k_i)) = 8\iota\cdot\sumk\sumsa q_k(s, a)\sqrt{\hatc_k(s,a)/k}\\
		&\qquad + \tilO{\sqrt{\sumk\sum_{s,a,h\leq H}q_k(s,a,h)Q_k(s,a,h)} + \sqrt{S^3A^3\Tmax^3} }.
	\end{align*}
\end{lemma}
\begin{proof}
	First note that by $c^k_i=c_k(s^k_i, a^k_i)$:
	\begin{align*}
		\sumk\sumib(c^k_i - \hatc_k(\rs^k_i, a^k_i)) = \sumk\sumib(c_k(s^k_i, a^k_i) - c(s^k_i, a^k_i)) + \sumk\sumib(c(s^k_i, a^k_i) - \hatc_k(s^k_i, a^k_i)).
	\end{align*}
	For the first term, with probability at least $1-\delta$,
	\begin{align*}
		&\sumk\sumib(c_k(s^k_i, a^k_i) - c(s^k_i, a^k_i)) = \sumk\sum_{s, a}\bar{n}_k(s, a)(c_k(s, a) - c(s, a))\\
		&\overset{\text{(i)}}{=}\tilO{ \sqrt{\sumk \E_k\sbr{\rbr{\sumsa \bar{n}_k(s, a)c_k(s, a)}^2 }} + \Tmax}\\
		&= \tilO{\sqrt{\sumk\E_{c_k}\sbr{ \sum_{s, a, h\leq H} q_k(s, a, h)Q^{\pi_k,P, c_k}(s, a, h) }} + \Tmax} \tag{$\E_k[\cdot]=\E_{c_k, \bar{n}_k}[\cdot]$, \pref{lem:var} and $c_k(s, a)\leq 1$}\\
		&= \tilO{\sqrt{\sumk\sum_{s,a,h\leq H}q_k(s,a,h)Q_k(s,a,h) + \sumk\inner{q_k}{Q^{\pi_k,P,c} - Q_k}} + \Tmax}, \tag{$\hatc_k(s,a,H+1)=c(s,a,H+1)$}
	\end{align*}
	where in (i) we apply \pref{lem:freedman}, $\E_k[\cdot]=\E_{c_k, \bar{n}_k}[\cdot]$, and 
	$$\E_{c_k}\sbr{\left.\rbr{\sum_{s,a}\bar{n}_k(s,a)(c_k(s,a)-c(s,a))}^2\right|\bar{n}_k} \leq \E_{c_k}\sbr{\left. \rbr{\sum_{s, a}\bar{n}_k(s, a)c_k(s,a)}^2\right|\bar{n}_k}.$$
	Now note that for $h\leq H$, by \pref{lem:value diff}, \pref{lem:cost bound}, and $\hatc_k(s,a,H+1)=c(s,a,H+1)$, we have with probability at least $1-2\delta$:
	\begin{align*}
		Q^{\pi_k,P,c}(s, a, h) - Q_k(s, a, h) &= \sum_{s', a', h'\leq H}q_{k,(s,a,h)}(s',a',h')(c(s',a',h') - \hatc_k(s',a',h'))\\
		&= \tilO{ \sum_{s', a'}q_{k,(s,a,h)}(s',a')\rbr{\sqrt{\frac{\hatc_k(s', a')}{\Mp_k(s', a')}} + \frac{1}{\Mp_k(s', a')}} }.
	\end{align*}
	Note that $q_k(s,a,h)q_{k,(s,a,h)}(s', a')=\tilo{\Tmax x_k(s,a,h)q_{k,(s,a,h)}(s', a')}=\tilo{\Tmax q_k(s', a')}$.
	Therefore, we have with probability at least $1-\delta$:
	\begin{align*}
		&\sumk\inner{q_k}{Q^{\pi_k,P,c} - Q_k} = \tilO{\Tmax\sumk\sum_{s,a,h\leq H}\sum_{s', a'}q_k(s', a')\rbr{\sqrt{\frac{\hatc_k(s', a')}{\Mp_k(s', a')}} + \frac{1}{\Mp_k(s', a')}} }\\
		&=\tilO{ SA\Tmax\sumk\sum_{s', a'}q_k(s', a')\rbr{\sqrt{\frac{\hatc_k(s', a')}{\Mp_k(s', a')}} + \frac{1}{\Mp_k(s', a')}} } \tag{$q_k(s',a',h')=\bigo{\Tmax x_k(s', a', h')}$}\\
		&=\tilO{SA\Tmax\rbr{ \sqrt{\sumk\sum_{s', a'}\frac{q_k(s', a')}{\Mp_k(s', a')} }\sqrt{\sumk\sum_{s', a'}q_k(s', a')\hatc_k(s', a') } + \sumk\sum_{s', a'}\frac{q_k(s', a')}{\Mp_k(s', a')} }} \tag{Cauchy-Schwarz inequality}\\
		&=\tilO{\sqrt{S^3A^3\Tmax^3 \sumk\sum_{s', a'}q_k(s', a')\hatc_k(s', a')} + S^2A^2\Tmax^2}. \tag{$q_k(s', a')\leq\Tmax x_k(s', a')$ and \pref{lem:n sum}}\\
		&=\tilO{\sumk\sum_{s', a'}q_k(s', a')\hatc_k(s', a') + S^3A^3\Tmax^3}. \tag{AM-GM inequality}
	\end{align*}
	Substituting these back, we have
	\begin{align}
		&\sumk\sumib(c_k(s^k_i, a^k_i) - c(s^k_i, a^k_i))\notag\\
		&= \tilO{\sqrt{\sumk\sum_{s,a,h\leq H}q_k(s,a,h)Q_k(s,a,h)} + \sqrt{S^3A^3\Tmax^3}}.\label{eq:cost diff}
	\end{align}
	For the second term, with probability at least $1-4\delta$,
	\begin{align*}
		&\sumk\sumib(c(s^k_i, a^k_i) - \hatc_k(\rs^k_i, a^k_i)) \leq \sumk\sumib\rbr{ 4\sqrt{\frac{\hatc_k(s^k_i, a^k_i)\iota}{\Mp_k(s^k_i, a^k_i)}} + \frac{34\iota}{\Mp_k(s^k_i, a^k_i)} } \tag{\pref{lem:cost bound}}\\
		&\leq \sumk\sumsa 8\cdot q_k(s, a)\iota\sqrt{\frac{\hatc_k(s, a)}{k}} + \sumk\sumsa\frac{68q_k(s, a)\iota}{k} + \tilO{\Tmax} \tag{\pref{lem:e2r}}\\
		&\leq 8\iota\cdot\sumk\sumsa q_k(s,a)\sqrt{\hatc_k(s,a)/k} + \tilO{\Tmax }.
	\end{align*}
	Putting everything together completes the proof.
\end{proof}

\begin{lemma}
	\label{lem:SA-B}
	Under stochastic adversary with bandit feedback, with probability at least $1-8\delta$,
	\begin{align*}
		\sumk\sumib(c^k_i - \hatc_k(\rs^k_i, a^k_i))=\tilO{\sqrt{SA\sumk\sumsa\sum_{h=1}^Hq_k(s, a, h)Q_k(s, a, h)} + \sqrt{S^3A^3\Tmax^3}}.
	\end{align*}
\end{lemma}
\begin{proof}
	First note that by $c^k_i=c_k(s^k_i, a^k_i)$:
	\begin{align*}
		\sumk\sumib(c^k_i - \hatc_k(\rs^k_i, a^k_i)) = \sumk\sumib(c_k(s^k_i, a^k_i) - c(s^k_i, a^k_i)) + \sumk\sumib(c(s^k_i, a^k_i) - \hatc_k(s^k_i, a^k_i)).
	\end{align*}
	For the first term, \pref{eq:cost diff} holds by the same arguments as in \pref{lem:SA-F} with probability at least $1-4\delta$.
	For the second term, we have with probability at least $1-4\delta$,
	\begin{align*}
		&\sumk\sumib(c(s^k_i, a^k_i) - \hatc_k(\rs^k_i, a^k_i)) = \tilO{ \sumk\sumib\rbr{ \sqrt{\frac{\hatc_k(s^k_i, a^k_i)}{\Mp_k(s^k_i, a^k_i)}} + \frac{1}{\Mp_k(s^k_i, a^k_i)} } }\tag{\pref{lem:cost bound}}\\
		&= \tilO{\sumk\sumsa\sum_{h=1}^H q_k(s, a, h)\sqrt{\frac{\hatc_k(s, a)}{\Mp_k(s, a)}} + \sumk\sumsa\frac{q_k(s, a)}{\Mp_k(s, a)} + \Tmax} \tag{\pref{lem:e2r}}\\
		&= \tilO{\sqrt{\sumk\sumsa\sum_{h=1}^H \frac{q^2_k(s, a, h)}{x_k(s, a, h)}\hatc_k(s, a)}\sqrt{\sumk\sumsa\sum_{h=1}^H\frac{x_k(s, a, h)}{\Mp_k(s, a)}} + SA\Tmax } \tag{Cauchy-Schwarz inequality, \pref{lem:n sum}, and $q_k(s, a)=\tilO{\Tmax x_k(s, a)}$}\\
		&= \tilO{\sqrt{SA\sumk\sumsa\sum_{h=1}^Hq_k(s, a, h)Q_k(s, a, h)} + SA\Tmax}. \tag{\pref{lem:n sum} and $\frac{q_k(s, a, h)}{x_k(s, a, h)}\hatc_k(s, a) \leq Q_k(s, a, h)$}
	\end{align*}
\end{proof}

\begin{lemma}
	\label{lem:qh}
	For $h\in[H+1]$, we have $\sum_{s, a}\optq(s, a, h)\leq (\frac{1}{2})^{h-1}\Tmax$.
\end{lemma}
\begin{proof}
	Denote by $p(s)$ the probability that the learner starts at state $s$ in layer $h$ and eventually reaches layer $h+1$ following $\roptpi$. 
	Clearly, $p(g)=0$, and
	\begin{align*}
		p(s) \leq 1-\gamma + \gamma P_{s, \optpi(s)}p \overset{\text{(i)}}{\leq} \E\sbr{\left. \sum_{t=1}^I (1 - \gamma)\gamma^{t-1} \right| \optpi, P, s_1=s} \leq \frac{1}{2},
	\end{align*}
	where (i) is by repeatedly applying the first inequality.
	By a recursive argument, we have the probability of reaching layer $h$ is upper bounded by $(\frac{1}{2})^{h-1}$.
	Then by $\sum_{s,a}\optq_{(s',h)}(s, a, h)\leq \Tmax$ for any $s'$, we have $\sum_{s, a}\optq(s, a, h)\leq (\frac{1}{2})^{h-1}\Tmax$.
\end{proof}

\begin{lemma}
	\label{lem:qcQ}
	Under stochastic costs, $\inner{\optq}{c\circ Q^{\roptpi, P, c}} \leq 2\B^2 + \frac{(H+1)\Tmax}{K}$.
\end{lemma}
\begin{proof}
	By \pref{lem:sda} and \pref{lem:qh}, we have:
	\begin{align*}
		\inner{\optq}{c\circ Q^{\roptpi, P, c}} &= \sumh\sum_{s, a}\optq(s, a, h)c(s, a)Q^{\roptpi,P,c}(s, a, h) + \sum_s\optq(s, H+1)c_f\\
		&\leq \sumh\sum_{s, a}\optq(s, a, h)c(s, a)\rbr{Q^{\optpi,P,c}(s, a) + \frac{c_f}{2^{H-h+1}} } + \frac{c_f\Tmax}{2^H}\\
		&\leq 2\B^2 + \sumh \frac{\Tmax}{2^{h-1}}\frac{c_f}{2^{H-h+1}} + \frac{c_f\Tmax}{2^H} \tag{$\sumh \optq(s, a, h)c(s, a)\leq \B$ and $Q^{\optpi,P,c}(s, a)\leq 1+\B$}\\
		&\leq 2\B^2 + (H+1)\frac{c_f\Tmax}{2^H} \leq 2\B^2 + \frac{(H+1)\Tmax}{K}.
	\end{align*}
\end{proof}

\begin{lemma}
	\label{lem:qcQ sa}
	For stochastic adversary, we have $\sumk\inner{\optq}{c\circ  Q^{\roptpi, P, c}}=\tilO{D^2K}$.
\end{lemma}
\begin{proof}
	$\sumk\inner{\optq}{c\circ  Q^{\roptpi, P, c}} = \tilO{DK\inner{\optq}{c}} = \tilO{D^2K}$.
\end{proof}

\begin{lemma}
	\label{lem:bound tilQ}
	$\eta\norm{\tilQ_k}_{\infty}\leq 1$ under all definitions of $\tilc_k$.
\end{lemma}
\begin{proof}
	It suffices to bound $\norm{\tilQ_k}_{\infty}$.
	By \pref{lem:sda}, $\hatQ_k(s,a,h)\leq \frac{H}{1-\gamma} + c_f = \chi$.
	Therefore, $e_k(s, a, h)\leq 8\iota + \chi/\Tmax$ under all feedback types. 
	This gives $\tilc_k(s,a,h)\leq (1+\lambda\hatQ_k(s, a, h)) + e_k(s, a, h) \leq 3(8\iota+\chi/\Tmax)$ for $h\leq H$ and $\tilc_k(s, a, H+1)\leq (1+\lambda\hatQ_k(s, a, H+1))c_f \leq 3c_f\chi/\Tmax$.
	\pref{lem:sda} then gives $\tilQ_k(s, a, h)\leq \frac{H}{1-\gamma}\cdot3(8\iota+\chi/\Tmax) + 3c_f\chi/\Tmax \leq 3\Tmax(8\iota+\chi/\Tmax)^2$, and the statement is proved by the definition of $\eta$.
\end{proof}

\begin{lemma}
	\label{lem:po stab}
	Under all definitions of $\tilc_k$, we have $\abr{d\pi_k(a|s, h)} = \tilo{\eta \Tmax\pi_k(a|s, h)}$ and $\norm{dQ^{\pi_k, P', c'}}_{\infty} =\tilo{ \eta \Tmax^3 }$ for $P'\in\Lambda_{\calM}$ and $c'\in\calC_{\calM}$.
\end{lemma}
\begin{proof}
	Note that:
	\begin{align*}
		&\pi_{k+1}(a|s, h) - \pi_k(a|s, h) = \frac{\pi_k(a|s, h)\exp(-\eta\tilQ_k(s, a, h))}{\sum_{a'}\pi_k(a'|s, h)\exp(-\eta\tilQ_k(s, a', h))} - \pi_k(a|s, h)\\
		&\leq \frac{\pi_k(a|s, h)}{\sum_{a'}\pi_k(a'|s, h)}\exp(\max_{a'}|\eta\tilQ_k(s, a', h)|) - \pi_k(a|s, h) = \tilO{\eta \Tmax\pi_k(a|s, h)}. \tag{\pref{lem:bound tilQ} and  $|e^x-1| \leq 2|x|$ for $x\in[-1, 1]$}
	\end{align*}
	The other direction can be proved similarly.
	Then by \pref{lem:value diff},
	\begin{align*}
		&\abr{Q^{\pi_{k+1}, P', c'}(s, a, h) - Q^{\pi_k, P', c'}(s, a, h)}\\
		&= \abr{\sum_{s'', h''}P_{s, a, h}(s'', h'')\sum_{s', a', h'}q_{\pi_k, P', (s'', h'')}(s', h')\rbr{d\pi_k(a'|s', h')}Q^{\pi_{k+1}, P', c'}(s', a', h')}\\ 
		&= \tilO{\eta \Tmax^3}.
	\end{align*}
	This completes the proof.
\end{proof}

\begin{lemma}
	\label{lem:po}
	Suppose $\pi_k(a|s,h)\propto\exp(\sum_{j<k}\tilQ_j(s, a, h))$.
	Then,
	\begin{align*}
		&\sumk\suma(\pi_k(a|s,h)-\optpi(a|s,h))\tilQ_k(s, a, h)\\ 
		&\leq \frac{\ln A}{\eta} + \inner{\pi_1(\cdot|s, h)}{\tilQ_1(s, \cdot, h)} + \sum_{k=1}^{K-1}\inner{\pi_{k+1}(\cdot|s, h)}{\tilQ_{k+1}(s, \cdot, h) - \tilQ_k(s, \cdot, h)}.
	\end{align*}
\end{lemma}
\begin{proof}
	First note that:
	\begin{equation}
		\label{eq:pi opt}
		\pi_{k+1}(\cdot|s, h) = \argmin_{\pi(\cdot|s, h)\in\Delta(A)}\eta\inner{\pi(\cdot|s, h)}{\tilQ_k(s, \cdot, h)} + \KL(\pi(\cdot|s, h), \pi_k(\cdot|s, h)),
	\end{equation}
	where $\KL(p, q)=\sum_a (p(a)\ln\frac{p(a)}{q(a)} - p(a) + q(a))$, and
	$$\pi_{k+1}(\cdot|s, h)\propto \pi'_{k+1}(a|s,h) \triangleq \pi_k(a|s, h)\exp(-\eta \tilQ_k(s, a, h)),$$
	where $\pi'_{k+1}$ is the solution of the unconstrained variant of \pref{eq:pi opt} (that is, replacing $\argmin_{\pi(\cdot|s, h)\in\Delta(A)}$ by $\argmin_{\pi(\cdot|s, h)\in\fR^{A}}$).
	It is easy to verify that:
	\begin{align}
		&\KL(\pi_k(\cdot|s,h),\pi_{k+1}(\cdot|s, h)) + \KL(\pi_{k+1}(\cdot|s, h),\pi_k(\cdot|s,h))\notag\\
		&= \inner{\pi_k(\cdot|s,h)}{\ln\frac{\pi_k(\cdot|s,h)}{\pi_{k+1}(\cdot|s,h)}} + \inner{\pi_{k+1}(\cdot|s,h)}{\ln\frac{\pi_{k+1}(\cdot|s,h)}{\pi_k(\cdot|s,h)}} \notag\\
		&= \inner{\pi_k(\cdot|s,h) - \pi_{k+1}(\cdot|s,h)}{\ln\frac{\pi_k(\cdot|s,h)}{\pi'_{k+1}(\cdot|s,h)}} \tag{$\pi_{k+1}(\cdot|s,h)\propto\pi'_{k+1}(\cdot|s,h)$}\\ 
		&= \inner{\pi_k(\cdot|s,h) - \pi_{k+1}(\cdot|s,h)}{\eta \tilQ_k(s, \cdot, h)} \geq 0. \label{eq:refine omd}
	\end{align}
	By the standard OMD analysis~\citep{hazan2019introduction} (note that $\KL$ is the Bregman divergence w.r.t the negative entropy regularizer),
	\begin{align*}
		&\sumk\inner{\pi_k(\cdot|s, h) - \optpi(\cdot|s, h)}{\tilQ_k(s, \cdot, h)}\\
		&= \frac{1}{\eta}\sumk \rbr{\KL(\optpi(\cdot|s, h), \pi_k(\cdot|s, h)) - \KL(\optpi(\cdot|s, h), \pi'_{k+1}(\cdot|s, h)) + \KL(\pi_k(\cdot|s, h), \pi'_{k+1}(\cdot|s, h)) }\\
		&= \frac{1}{\eta}\sumk \rbr{ \KL(\optpi(\cdot|s, h), \pi_k(\cdot|s, h)) - \KL(\optpi(\cdot|s, h), \pi_{k+1}(\cdot|s, h)) + \KL(\pi_k(\cdot|s, h), \pi_{k+1}(\cdot|s, h)) }\\
		&\leq \frac{\KL(\optpi(\cdot|s, h), \pi_1(\cdot|s, h))}{\eta} + \sumk \inner{\pi_k(\cdot|s,h) - \pi_{k+1}(\cdot|s,h)}{\tilQ_k(s, \cdot, h)} \tag{\pref{eq:refine omd}} \\
		&\leq \frac{\ln A}{\eta} + \sum_{k=1}^{K-1}\inner{\pi_{k+1}(\cdot|s, h)}{\tilQ_{k+1}(s, \cdot, h) - \tilQ_k(s, \cdot, h)} \\ 
		&\qquad + \inner{\pi_1(\cdot|s, h)}{\tilQ_1(s, \cdot, h)} - \inner{\pi_{K+1}(\cdot|s, h)}{\tilQ_K(s, \cdot, h)}. 
	\end{align*}
	This completes the proof.
\end{proof}

\begin{lemma}
	\label{lem:diff}
	Define $\frn_k(s, a)=\frN_{k+1}(s, a)-\frN_k(s, a)$.
	We have:
	\begin{align*}
		&\abr{d\hatc_k(s, a)} = \bigO{\frac{\frn_k(s, a)\iota}{\frNp_k(s, a)} },\\ 
		&\abr{d\hatQ_k(s, a, h)} = \bigO{ \Tmax^2\sum_{s', a'}\frac{Sn_k(s', a')\iota}{\Np_k(s', a')} + \Tmax\sum_{s', a'}\frac{\frn_k(s', a')\iota}{\frNp_k(s', a')}  + \eta \Tmax^3 },\\
		&\abr{d\tilc_k(s, a)}\\
		&=\bigO{\frac{\frn_k(s, a)\iota}{\frNp_k(s, a)} + \lambda \Tmax^2\sum_{s', a'}\frac{Sn_k(s', a')\iota}{\Np_k(s', a')} + \lambda \Tmax\sum_{s', a'}\frac{\frn_k(s', a')}{\frNp_k(s', a')} + \lambda\eta \Tmax^3 + \abr{de_k(s, a, h)} },\\
		&d\tilQ_k(s, a, h)= \bigO{\Tmax^2\sum_{s', a'}\frac{Sn_k(s', a')\iota}{\Np_k(s', a')} + \lambda \Tmax^2\sum_{s', a'}\frac{\frn_k(s', a')}{\frNp_k(s', a')} + \lambda\eta \Tmax^4 + \Tmax\norm{de_k}_1 }.
	\end{align*}
\end{lemma}
\begin{proof}
	\textbf{First statement:} Note that for all definitions of $\hatc_k$ used in this paper, we have $\norm{\hatc_k}_{\infty}\leq 1$.
	Then by the definition of $\hatc_k$ and $|\max\{0, a\} - \max\{0, b\}|\leq |a-b|$:
	\begin{align*}
		&\abr{\hatc_{k+1}(s, a) - \hatc_k(s, a)}\\
		&= \bigO{\abr{ \barc_{k+1}(s, a) - \barc_k(s, a) } + \abr{\sqrt{\frac{\barc_k(s, a)\iota}{\frNp_k(s, a)}} - \sqrt{\frac{\barc_{k+1}(s, a)\iota}{\frNp_{k+1}(s, a)}}} + \frac{\iota}{\frNp_k(s, a)} - \frac{\iota}{\frNp_{k+1}(s, a)} }.
	\end{align*}
	Note that:
	\begin{align*}
		&\abr{\bar{c}_{k+1}(s, a) - \bar{c}_k(s, a)} = \abr{\frac{C_{k+1}(s, a)}{\frNp_{k+1}(s, a)} - \frac{C_k(s, a)}{\frNp_k(s, a)} }\\ 
		&\leq \abr{\frac{C_{k+1}(s, a)-C_k(s, a)}{\frNp_{k+1}(s, a)}} + \frN_k(s, a)\abr{\frac{1}{\frNp_k(s, a)} - \frac{1}{\frNp_{k+1}(s, a)}} \tag{$C_k(s, a)\leq\frN_k(s,a)$}\\
		&\leq \frac{\frn_k(s, a)}{\frNp_{k+1}(s, a)} + \frac{\frN_k(s,a)\frn_k(s, a)}{\frNp_k(s, a)\frNp_{k+1}(s, a)} \leq \frac{2\frn_k(s, a)}{\frNp_k(s, a)},
	\end{align*}
	and by $|\sqrt{a}-\sqrt{b}|\leq\sqrt{|a-b|}$, $\frn_k(s, a)\in\fN$:
	\begin{align*}
		&\abr{\sqrt{\frac{\barc_k(s, a)\iota}{\frNp_k(s, a)}} - \sqrt{\frac{\barc_{k+1}(s, a)\iota}{\frNp_{k+1}(s, a)}}}\\
		&\leq \sqrt{\frac{|\barc_k(s, a)-\barc_{k+1}(s, a)|\iota}{\frNp_k(s, a)}} + \sqrt{\barc_{k+1}(s, a)\iota}\rbr{\frac{1}{\sqrt{\frNp_k(s, a)}} - \frac{1}{\sqrt{\frNp_{k+1}(s, a)}}}\\
		&\leq \frac{2\frn_k(s, a)\iota}{\frNp_k(s, a)} + \rbr{ \sqrt{\frac{\iota}{\frNp_k(s, a)}} - \sqrt{\frac{\iota}{\frNp_{k+1}(s, a)}} } = \bigO{\frac{\frn_k(s, a)\iota}{\frNp_k(s, a)}},
	\end{align*}
	where in the last inequality we apply
	\begin{align}
		\frac{1}{\sqrt{\frNp_k(s, a)}} - \frac{1}{\sqrt{\frNp_{k+1}(s, a)}} &= \rbr{\frac{1}{\frNp_k(s, a)} - \frac{1}{\frNp_{k+1}(s, a)}} / \rbr{ \frac{1}{\sqrt{\frNp_k(s, a)}} + \frac{1}{\sqrt{\frNp_{k+1}(s, a)} }}\notag\\
		&\leq \sqrt{\frNp_{k+1}(s, a)}\cdot\frac{\frn_k(s, a)}{\frNp_k(s, a)\frNp_{k+1}(s, a)} \leq \frac{\frn_k(s, a)}{\frNp_k(s, a)}.\label{eq:dsqrtn}
	\end{align}
	Thus, $\abr{d\hatc_k(s, a)} = \bigO{\frac{\frn_k(s, a)\iota}{\frNp_k(s, a)} }$.
	
	\textbf{Second statement:} Define $\Pi_k(P')=\argmin_{P''\in\calP_{k+1}}\sum_{s, a, h}\norm{P''_{s, a, h} - P'_{s, a, h}}_1$ for any $P'\in\calP_k$.
	By the definition of $\calP_k$, we have (note that $P'_{s,a,h}(s',h')=0$ for $h'\notin\{h, h+1\}$):
	\begin{align*}
		\norm{\Pi_k(P')_{s, a, h} - P'_{s, a, h}}_1 \leq 2\sum_{s'}\abr{\barP_{k, s, a}(s') - \barP_{k+1, s, a}(s')} + 2\sum_{s'}\abr{\epsilon_{k+1}(s, a, s') - \epsilon_k(s, a, s')}.
	\end{align*}
	Denote by $n_k(s, a, s')$ the number of visits to $(s, a, s')$ (before policy switch or goal state is reached) in episode $k$.
	Note that:
	\begin{align*}
		&\abr{\barP_{k, s, a}(s') - \barP_{k+1, s, a}(s')} = \abr{\frac{N_k(s, a, s') + n_k(s, a, s')}{\Np_{k+1}(s, a)} - \frac{N_k(s, a, s')}{\Np_k(s, a)}}\\
		&\leq N_k(s, a, s')\rbr{\frac{1}{\Np_k(s, a)} - \frac{1}{\Np_{k+1}(s, a)}} + \frac{n_k(s, a, s')}{\Np_{k+1}(s, a)} \leq \frac{2n_k(s, a)}{\Np_k(s, a)}.
	\end{align*}
	and by $|\sqrt{a} - \sqrt{b}|\leq\sqrt{|a-b|}$,
	\begin{align*}
		&\abr{\epsilon_k(s, a, s') - \epsilon_{k+1}(s, a, s')} = \bigO{\abr{ \sqrt{\frac{\barP_{k, s, a}(s')\iota}{\Np_k(s, a)}} - \sqrt{\frac{\barP_{k+1, s, a}(s')\iota}{\Np_{k+1}(s, a)}} } + d\rbr{\frac{-\iota}{\Np_k(s, a)}} } \\ 
		&=\bigO{ \sqrt{ \frac{\abr{\barP_{k, s, a}(s') - \barP_{k+1, s, a}(s')}\iota}{\Np_k(s, a)} } + \sqrt{\barP_{k+1, s, a}(s')\iota}d\rbr{\frac{-1}{\sqrt{\Np_k(s, a)}} } + d\rbr{\frac{-\iota}{\Np_k(s, a)}} }\\
		&= \bigO{\frac{n_k(s, a)\iota}{\Np_k(s, a)} + \sqrt{\barP_{k+1, s, a}(s')} d\rbr{ \frac{-\sqrt{\iota}}{\sqrt{\Np_k(s, a)}}  } }.
	\end{align*}
	Plugging these back, and by Cauchy-Schwarz inequality and \pref{eq:dsqrtn} with $\frN_k=N_k$, we have
	\begin{equation}
		\label{eq:proj}
		\norm{\Pi_k(P')_{s, a, h} - P'_{s, a, h}}_1 = \bigO{\frac{Sn_k(s, a)\iota}{\Np_k(s, a)} + d\rbr{ \frac{-\sqrt{S\iota}}{\sqrt{\Np_k(s, a)}} } } = \bigO{ \frac{Sn_k(s, a)\iota}{\Np_k(s, a)} }.
	\end{equation}
	Thus, for any policy $\pi'$ and cost function $c'\in\calC_{\calM}$ with $c'(s, a, h)\in[0, 1]$ for $h\leq H$, by \pref{lem:value diff} and \pref{eq:proj},
	\begin{align}
		&\abr{Q^{\pi', \Pi_k(P'), c'}(s, a, h) - Q^{\pi', P', c'}(s, a, h)} \notag\\ 
		&= \abr{ \sum_{s', a', h'}q_{\pi', P', (s, a, h)}(s', a', h')(\Pi_k(P')_{s',a',h'} - P'_{s',a',h'})V^{\pi', \Pi_k(P'), c'} } \notag\\
		&= \bigO{ \Tmax^2\sum_{s', a'} \frac{Sn_k(s', a')\iota}{\Np_k(s', a')} }. \label{eq:proj diff}
	\end{align}
	Now define $P'_k=\Pi_k(P_k)$.
	We have
	\begin{align*}
		&\hatQ_{k+1}(s, a, h) - \hatQ_k(s, a, h) = Q^{\pi_{k+1}, P_{k+1}, \hatc_{k+1}}(s, a, h) - Q^{\pi_k, P_k, \hatc_k}(s, a, h)\\
		&\leq Q^{\pi_{k+1}, P'_k, \hatc_{k+1}}(s, a, h) - Q^{\pi_{k+1}, P_k, \hatc_{k+1}}(s, a, h) + Q^{\pi_{k+1}, P_k, \hatc_{k+1}}(s, a, h) - Q^{\pi_k, P_k, \hatc_k}(s, a, h)\\ 
		&= \bigO{ \Tmax^2\sum_{s', a'} \frac{Sn_k(s', a')\iota}{\Np_k(s', a')} } + (Q^{\pi_{k+1}, P_k, \hatc_{k+1}}(s, a, h) - Q^{\pi_{k+1}, P_k, \hatc_k}(s, a, h)) \tag{\pref{eq:proj diff}}\\
		&\qquad + (Q^{\pi_{k+1}, P_k, \hatc_k}(s, a, h) - Q^{\pi_k, P_k, \hatc_k}(s, a, h))\\
		&=\bigO{ \Tmax^2\sum_{s', a'} \frac{Sn_k(s', a')\iota}{\Np_k(s', a')} + \Tmax\sum_{s', a'}\abr{\hatc_{k+1}(s', a') - \hatc_k(s', a')} + \eta \Tmax^3}\tag{\pref{lem:value diff} and \pref{lem:po stab}}\\
		&= \bigO{ \Tmax^2\sum_{s', a'}\frac{Sn_k(s', a')\iota}{\Np_k(s', a')} + \Tmax\sum_{s', a'}\frac{\frn_k(s', a')\iota}{\frNp_k(s', a')}  + \eta \Tmax^3 }.
	\end{align*}
	The other direction can be proved similarly.
	
	\textbf{Third statement:} Note that $|d\tilc_k(s, a, H+1)|=0$, and for $h\leq H$,
	\begin{align*}
		&\abr{\tilc_{k+1}(s, a, h) - \tilc_k(s, a, h)}\\
		&\leq \abr{d\hatc_k(s,a)} + \lambda\abr{\hatc_{k+1}(s, a)\hatQ_{k+1}(s, a, h) - \hatQ_k(s, a, h)\hatc_k(s, a)} + \abr{de_k(s, a, h)}\\
		&\leq \abr{d\hatc_k(s,a)} + \lambda\hatQ_{k+1}(s, a, h)\abr{d\hatc_k(s, a)} + \lambda\hatc_k(s, a)\abr{d\hatQ_k(s, a, h)} + \abr{de_k(s, a, h)}\\
		&= \bigO{\frac{\frn_k(s, a)\iota}{\frNp_k(s, a)} + \lambda \Tmax^2\sum_{s', a'}\frac{Sn_k(s', a')\iota}{\Np_k(s', a')} + \lambda \Tmax\sum_{s', a'}\frac{\frn_k(s', a')}{\frNp_k(s', a')} + \lambda\eta \Tmax^3 + \abr{de_k(s, a, h)} }.
	\end{align*}
	\textbf{Fourth statement:} Define $\tilP'_k=\Pi_k(\tilP_k)$.
	By \pref{eq:proj}, $\norm{\tilP'_{k, s, a, h} - \tilP_{k, s, a, h}}_1 = \bigO{ \frac{Sn_k(s, a)\iota}{\Np_k(s, a)} }$, and $\lambda\leq 1/\Tmax$, we have
	\begin{align*}
		&\tilQ_{k+1}(s, a, h) - \tilQ_k(s, a, h) \leq Q^{\pi_{k+1}, \tilP'_k, \tilc_{k+1}}(s, a, h) - Q^{\pi_k, \tilP_k, \tilc_k}(s, a, h)\\
		&= \rbr{Q^{\pi_{k+1}, \tilP'_k, \tilc_{k+1}}(s, a, h) - Q^{\pi_{k+1}, \tilP_k, \tilc_{k+1}}(s, a, h)}\\
		&\quad + \rbr{ Q^{\pi_{k+1}, \tilP_k, \tilc_{k+1}}(s, a, h) - Q^{\pi_{k+1}, \tilP_k, \tilc_k}(s, a, h) } + \rbr{Q^{\pi_{k+1}, \tilP_k, \tilc_k}(s, a, h) - Q^{\pi_k, \tilP_k, \tilc_k}(s, a, h)}\\
		&\overset{\text{(i)}}{\leq} \bigO{ \Tmax^2\sum_{s', a'} \frac{Sn_k(s', a')\iota}{\Np_k(s', a')} } + \sum_{s', a', h'}q_{\pi_{k+1},\tilP_k, (s,a,h)}(s', a', h')\abr{\tilc_{k+1}(s', a', h') - \tilc_k(s', a', h')}\\
		&= \bigO{\Tmax^2\sum_{s', a'}\frac{Sn_k(s', a')\iota}{\Np_k(s', a')} + \lambda \Tmax^2\sum_{s', a'}\frac{\frn_k(s', a')}{\frNp_k(s', a')} + \lambda\eta \Tmax^4 + \Tmax\norm{de_k}_1 },
	\end{align*}
	where in (i) we apply \pref{eq:proj diff}, \pref{lem:value diff}, and
	\begin{align*}
		&Q^{\pi_{k+1},\tilP_k,\tilc_k}(s, a, h) - Q^{\pi_k,\tilP_k,\tilc_k}(s, a, h)\\
		&= \sum_{s'',h''}\tilP_{k,s,a,h}(s'',h'')\sum_{s', a', h'}q_{\pi_{k+1}, \tilP_k,(s'',h'')}(s', h')\rbr{d\pi_k(a'|s', h')}Q^{\pi_k,\tilP_k,\tilc_k}(s', a', h') \tag{\pref{lem:value diff}}\\
		&\leq 0. \tag{\pref{eq:refine omd}}
	\end{align*}
	This completes the proof.
\end{proof}

\begin{lemma}
	\label{lem:var}
	For any cost function $c$ in $\rcalM$ such that $c((s, h), a)\geq 0$, we have:
	\begin{align*}
		\V_k[\inner{n_k}{c}] &= \sum_{s, a, h} q_k(s, a, h)(A^{\pi_k, P, c}(s, a, h)^2 + \fV(P_{s, a, h}, V^{\pi_k, P, c}))\\
		&\leq \E_k[\inner{n_k}{c}^2] \leq 2\inner{q_k}{c\circ Q^{\pi_k, P, c}}.
	\end{align*}
\end{lemma}
\begin{proof}
	Let $Q=Q^{\pi_k, P, c}$, $V=V^{\pi_k, P, c}$, $A=A^{\pi_k, P, c}$ and define $c(g, a)=0$.
	Then,
	\begin{align*} 
		&\V_k[\inner{n_k}{c}] = \E_k\sbr{ \rbr{\sumitp c(\rs^k_i, a^k_i) - V(\rs^k_1)}^2 }\\
		&= \E_k\sbr{ \rbr{ \sum_{i=2}^{J_k+1}c(\rs^k_i, a^k_i) + Q(\rs^k_1, a^k_1) - P_{\rs^k_1, a^k_1}V - V(\rs^k_1)  }^2 }\tag{$Q(\rs, a)=c(\rs,a)+P_{\rs,a}V$}\\
		&\overset{\text{(i)}}{=} \E_k\sbr{\rbr{Q(\rs^k_1, a^k_1) - V(\rs^k_1)}^2} + \E_k\sbr{\rbr{ \sum_{i=2}^{J_k+1} c(\rs^k_i, a^k_i) - P_{\rs^k_1, a^k_1}V }^2} \\
		&\overset{\text{(ii)}}{=} \E_k\sbr{\rbr{Q(\rs^k_1, a^k_1) - V(\rs^k_1)}^2}  + \E_k\sbr{\rbr{\sum_{i=2}^{J_k+1} c(\rs^k_i, a^k_i) - V(\rs^k_2)}^2} + \E_k\sbr{\rbr{ V(\rs^k_2) - P_{\rs^k_1, a^k_1}V }^2}\\
		&= \E_k\sbr{ \sum_{i=1}^{J_k+1}\sbr{\rbr{Q(\rs^k_i, a^k_i) - V(\rs^k_i)}^2 + \rbr{ V(\rs^k_{i+1}) - P_{\rs^k_i, a^k_i}V }^2 } } \tag{recursive argument}\\ 
		&= \sum_{s, a, h}q_k(s, a, h)\rbr{A^2(s, a, h) + \fV(P_{s, a, h}, V) }, 
	\end{align*}
	where (i) is by $Q(\rs^k_1, a_1) - V(\rs^k_1)\in\sigma(\rs^k_1, a^k_1)$ (the $\sigma$-algebra of events defined on $(\rs^k_1, a^k_1)$) and 
	$$\E_k\sbr{\left. \sum_{i=2}^{J_k+1} c(\rs^k_i, a^k_i) - P_{\rs^k_1, a^k_1}V\right| \rs^k_1, a^k_1}=0;$$ (ii) is by $V(\rs^k_2) - P_{\rs^k_1, a^k_1}V\in\sigma(\rs^k_1, a^k_1, \rs^k_2)$ and
	$$ \E_k\sbr{\left.\sum_{i=2}^{J_k+1} c(\rs^k_i, a^k_i) - V(\rs^k_2)\right| \rs^k_1, a^k_1, \rs^k_2}=0.$$
	Moreover, by $(\sum_{i=1}^na_i)^2\leq 2a_i(\sum_{i'=i}^na_{i'})$ for any $n\geq 1$ and $P(J_k=\infty)=0$,
	\begin{align*}
		\V_k[\inner{n_k}{c}] \leq \E_k[\inner{n_k}{c}^2] &= \E_k\sbr{\rbr{\sum_{i=1}^{J_k+1} c(\rs^k_i, a^k_i)}^2} \leq 2\E_k\sbr{ \sum_{i=1}^{J_k+1}c(\rs^k_i, a^k_i)\sum_{i'=i}^{J_k+1}c(\rs^k_{i'}, a^k_{i'}) }\\
		&=2\E_k\sbr{ \sum_{i=1}^{\infty}\Ind\{J_k+1\geq i\}c(\rs^k_i, a^k_i)\sum_{i'=i}^{J_k+1}c(\rs^k_{i'}, a^k_{i'}) }\\
		&\overset{\text{(i)}}{=} 2\E_k\sbr{\sum_{i=1}^{J_k+1}c(\rs^k_i, a^k_i)Q(\rs^k_i, a^k_i)} = 2\inner{q_k}{c\circ Q},
	\end{align*}
	where in (i) we apply $Q(\rs^k_i, a^k_i) = \E[\sum_{i'=i}^{J_k+1}c(\rs^k_{i'}, a^k_{i'})|\rs^k_1,a^k_1,\ldots,\rs^k_i, a^k_i]$ and $\{J_k+1\geq i\}\in\sigma(\rs^k_1, a^k_1,\ldots, \rs^k_i, a^k_i)$.
\end{proof}

\begin{lemma}
	\label{lem:bound q}
	For every $k \in [K]$ it holds that
	$q_k(s, a, h)\leq \E_k[\bar{n}_k(s, a, h)] + \tilO{1/K}$.
\end{lemma}
\begin{proof}
    By definition of $n_k(s,a,h)$, $x_k(s,a,h)$, and $y_k(s,a,h)$ we have:
    \begin{align*}
        \Pr \rbr{n_k(s,a,h) > n}
        &=
        \Pr \rbr{n_k(s,a,h) > n \mid n_k(s,a,h) > n-1} \Pr \rbr{n_k(s,a,h) > n-1}
        \\
        &=
        \Pr \rbr{\text{return to $(s,a,h)$}} \Pr \rbr{n_k(s,a,h) > n-1}
        \\
        &=
        y_k(s,a,h) \Pr \rbr{n_k(s,a,h) > n-1}
        \\
        &= \dots
        = y^n_k(s,a,h) \Pr \rbr{n_k(s,a,h) > 0}
        = y^n_k(s,a,h) x_k(s,a,h).
    \end{align*}
    Now, since $q_k(s,a,h)$ is the expected number of visits to $(s,a,h)$,
    \begin{align*}
        q_k(s, a, h) 
        &=
        \E_k [n_k(s,a,h)]
        =
        \sum_{n=0}^{\infty} \Pr \rbr{n_k(s,a,h) > n}
        =
        x_k(s,a,h) \sum_{n=0}^{\infty} y^n_k(s,a,h)
        \\
        & =
        x_k(s,a,h) \sum_{n=0}^{L-1} y^n_k(s,a,h) + x_k(s,a,h) \sum_{n=L}^{\infty} y^n_k(s,a,h).
    \end{align*}
    To finish we bound each of the sums separately.
    By definition of $\bar{n}_k(s,a,h)$:
    \begin{align*}
        x_k(s,a,h) \sum_{n=0}^{L-1} y^n_k(s,a,h)
        &=
        \sum_{n=0}^{L-1} \Pr \rbr{n_k(s,a,h) > n}
        =
        \sum_{n=0}^{L-1} \Pr \rbr{\min \{ L , n_k(s,a,h) \} > n}
        \\
        & \le
        \sum_{n=0}^{\infty} \Pr \rbr{\min \{ L , n_k(s,a,h) \} > n}
        \\
        &=
        \sum_{n=0}^{\infty} \Pr \rbr{\bar{n}_k(s,a,h) > n}
        =
        \E_k [\bar{n}_k(s,a,h)].
    \end{align*}
	In each step there's a probability of at most $\gamma$ to stay in layer $h$. So $y_k(s,a,h) \le \gamma$, which implies:
	\begin{align*}
	    x_k(s,a,h) \sum_{n=L}^{\infty} y^n_k(s,a,h)
	    &\le
	    \sum_{n=L}^{\infty} \gamma^n
	    =
	    \frac{\gamma^L}{1 - \gamma}
	    \le
	    \frac{\gamma^{\frac{8H}{1 - \gamma} \ln (2 \Tmax K / \delta)}}{1 - \gamma}
	    \le
	    \frac{e^{-8H\ln (2 \Tmax K / \delta)}}{1 - \gamma}
	    \\
	    &\le
	    2 \Tmax \rbr{\frac{\delta}{2 \Tmax K}}^{8 \log_2(c_fK)}
	    =
	    \tilO{1/K},
	\end{align*}
	where the second inequality uses $\gamma^{\frac{1}{1 - \gamma}} \le e^{-1}$.
\end{proof}

\begin{lemma}
	\label{lem:q-qk}
	Consider a sequence of cost functions $\{c_k\}_{k=1}^K$ and transition functions $\{P_k\}_{k=1}^K$ such that $c_k\in\calC_{\calM}$ and $P_k\in\calP_k$.
	Also define $\hatq_k=q_{\pi_k, P_k}$.
	Then with probability at least $1-8\delta$,
	\begin{align*}
		\sumk\abr{\inner{q_k - \hatq_k}{c_k}} = \tilO{\sqrt{S^2A\sumk\inner{q_k}{c_k\circ Q^{\pi_k, P, c_k}}} + S^{2.5}A^{1.5}\Tmax^3}.
	\end{align*}
\end{lemma}
\begin{proof}
	Define $v_{k, s, a, h}(\rs')=V^{\pi_k, P, c_k}(\rs') - P_{s, a, h}V^{\pi_k, P, c_k}$ for $\rs'\in\rcalS_+$.
	Note that with probability at least $1-4\delta$:
	\begin{align*}
		&\sumk\abr{\inner{q_k-\hatq_k}{c_k}} = \sumk\abr{\sum_{s, a, h}q_k(s, a, h)(P_{s, a, h}-P_{k, s, a, h})V^{\pi_k, P_k,c_k}} \tag{\pref{lem:value diff}}\\
		&= \sumk\abr{\sum_{s, a, h}q_k(s, a, h)(P_{s, a, h}-P_{k, s, a, h})V^{\pi_k, P, c_k}} + \tilO{S^{2.5}A^{1.5}\Tmax^3} \tag{\pref{lem:conf bound} and \pref{lem:double trans diff}}.
	\end{align*}
	Below we bound the first term.
	We continue with:
	\begin{align*}
		&= \sumk\abr{\sum_{s, a}\sum_{h=1}^Hq_k(s, a, h)(P_{s, a, h}-P_{k, s, a, h})v_{k, s, a, h}}\tag{$P_{s,a,H+1}=P_{k,s,a,H+1}$}\\
		&= \tilO{\sumk\sum_{s, a, h\leq H, \rs'}q_k(s, a, h)\sqrt{\frac{P_{s, a, h}(\rs')v_{k, s, a, h}^2(\rs')}{\Np_k(s, a)}} + S\Tmax\sumk\sum_{s, a}\frac{q_k(s, a)}{\Np_k(s, a)}}. \tag{\pref{lem:conf bound}}
	\end{align*}
	By \pref{lem:bound q}, we have $q_k(s, a, h)\leq \E_k[\bar{n}_k(s, a, h)] + \tilo{1/K}$.
	Therefore, we continue with
	\begin{align*}
		&= \tilO{\sumk\sum_{s, a, h\leq H, \rs'}\E_k[\bar{n}_k(s, a, h)]\sqrt{\frac{P_{s, a, h}(\rs')v_{k, s, a, h}^2(\rs')}{\Np_k(s, a)}} + S^2A\Tmax } \tag{\pref{lem:n sum}}\\
		&= \tilO{\sumk\sum_{s, a, h\leq H, \rs'}\bar{n}_k(s, a, h)\sqrt{\frac{P_{s, a, h}(\rs')v_{k, s, a, h}^2(\rs')}{\Np_k(s, a)}} + S^2A\Tmax^2} \tag{\pref{lem:e2r}}\\
		&= \tilO{\sumk\sum_{s, a, h\leq H, \rs'}\bar{n}_k(s, a, h)\sqrt{\frac{P_{s, a, h}(\rs')v_{k, s, a, h}^2(\rs')}{\Np_{k+1}(s, a)}} }\\
		&\qquad + \tilO{S\Tmax^2\sumsa\sumk\rbr{ \frac{1}{\sqrt{\Np_k(s, a)}} - \frac{1}{\sqrt{\Np_{k+1}(s, a)}} }  + S^2A\Tmax^2}\\
		&= \tilO{\sqrt{\sumk\sum_{s, a, \rs'} \frac{\bar{n}_k(s, a)}{\Np_{k+1}(s, a)} }\sqrt{\sumk\sum_{s, a, h\leq H, \rs'}\bar{n}_k(s, a, h)P_{s, a, h}(\rs')v_{k, s, a, h}^2(\rs')} + S^2A\Tmax^2}\tag{Cauchy-Schwarz inequality}\\
		&= \tilO{\sqrt{S^2A}\sqrt{\sumk\sum_{s, a, h\leq H, \rs'}q_k(s, a, h)P_{s, a, h}(\rs')v_{k, s, a, h}^2(\rs') +  SA\Tmax^3}  + S^2A\Tmax^2} \tag{\pref{lem:e2r}}\\
		&= \tilO{ \sqrt{S^2A\sumk\V_k[\inner{n_k}{c_k}]} + S^2A\Tmax^2 } \tag{$\sum_{\rs'}P_{s,a,h}(\rs')v_{k, s, a, h}^2(\rs')=\fV(P_{s,a,h}, V^{\pi_k,P,c_k})$ and \pref{lem:var}}\\ 
		&= \tilO{ \sqrt{S^2A\sumk\inner{q_k}{c_k\circ Q^{\pi_k, P, c_k}}} + S^2A\Tmax^2}. \tag{\pref{lem:var}}
	\end{align*}
	Substituting these back completes the proof.
\end{proof}

\begin{lemma}
	\label{lem:double trans diff}
	Consider a sequence of cost functions $\{c_k\}_{k=1}^K$ and transition functions $\{P_k\}_{k=1}^K$ such that $c_k\in\calC_{\calM}$ and $P_k\in\calP_k$. 
	Then, we have with probability at least $1-4\delta$:
	\[
	    \sumk\sum_{s, a, h, s', h'}q_k(s, a, h)\epsilon^{\star}_k(s, a, h, s', h')\abr{V^{\pi_k, P, c_k}(s', h') - V^{\pi_k, P_k, c_k}(s', h')} =\tilO{S^{2.5}A^{1.5}\Tmax^3}.
	\]
\end{lemma}
\begin{proof}
	Below $\lesssim$ is equivalent to $\tilO{\cdot}$.
	Also denote $z=(s, a, h, s', h')$ and $\tilde z=(\tils,\tila,\tilh,\tils',\tilh')$.
	By \pref{lem:value diff} we have with probability at least $1-2\delta$:
	\begin{align*}
	    \abr{V^{\pi_k, P, c_k}(s', h') - V^{\pi_k, P_k, c_k}(s', h')}
	    &\lesssim
	    \sum_{\tils,\tila,\tilh} q_{k,(s',h')}(\tils,\tila,\tilh)\abr{P_{\tils,\tila,\tilh} V^{\pi_k, P, c_k} - P_{k,\tils,\tila,\tilh} V^{\pi_k, P, c_k}}
	    \\
	    &\lesssim
	    \Tmax \sum_{\tils,\tila,\tilh} q_{k,(s',h')}(\tils,\tila,\tilh) \norm{P_{\tils,\tila,\tilh} - P_{k,\tils,\tila,\tilh}}_1
	    \\
	    &\lesssim
	   \Tmax \sum_{\tils,\tila,\tilh, \tils',\tilh'} q_{k,(s',h')}(\tils,\tila,\tilh) \epsilon^{\star}_k(\tils,\tila,\tilh, \tils', \tilh'),
	\end{align*}
	where the second inequality is by \pref{lem:sda}, and the third is by \pref{lem:conf bound}.
	Thus, using \pref{lem:conf bound} and the Cauchy-Schwarz inequality, we get:
	\begin{align*}
	    &\sumk\sum_{s, a, h, s', h'}q_k(s, a, h)\epsilon^{\star}_k(s, a, h, s', h')\abr{V^{\pi_k, P, c_k}(s', h') - V^{\pi_k, P_k, c_k}(s', h')}
	    \\
	    &\lesssim
	    \Tmax \sumk \sum_{z}q_k(s, a, h)\epsilon^{\star}_k(s, a, h, s', h') \sum_{\tilde z} q_{k,(s',h')}(\tils,\tila,\tilh) \epsilon^{\star}_k(\tils,\tila,\tilh, \tils', \tilh')
	    \\
	    &\lesssim
	    \Tmax
	    \sumk \sum_{z} q_k(s, a, h)\sqrt{\frac{P_{s, a, h}(s', h')}{\Np_k(s, a)}} \sum_{\tilde z} q_{k, (s', h')}(\tils, \tila, \tilh) \sqrt{\frac{P_{\tils, \tila, \tilh}(\tils', \tilh')}{\Np_k(\tils, \tila)}}
	    \\
	    &\lesssim
	    \Tmax \sqrt{\sum_{k,z,\tilde z} \frac{q_k(s, a, h)P_{\tils, \tila, \tilh}(\tils', \tilh')q_{k,(s', h')}(\tils, \tila, \tilh)}{\Np_k(s, a)}} \sqrt{\sum_{k,z,\tilde z} \frac{q_k(s, a, h)P_{s, a, h}(s', h')q_{k,(s', h')}(\tils, \tila, \tilh)}{\Np_k(\tils, \tila)}}.
	\end{align*}
	Note that we ignore some lower order terms in the calculation above.
	To finish the proof we bound each of the terms separately.
	For the first term we have with probability at least $1-\delta$:
	\begin{align*}
	    &\sum_{k,z,\tilde z} \frac{q_k(s, a, h)P_{\tils, \tila, \tilh}(\tils', \tilh')q_{k,(s', h')}(\tils, \tila, \tilh)}{\Np_k(s, a)}
	    \\
	    &=
	    \sum_{k,s,a} \frac{\rbr{\sum_{h} q_k(s,a,h)} \sum_{s',h',\tils,\tila,\tilh} q_{k,(s', h')} (\tils, \tila, \tilh) \sum_{\tils',\tilh'} P_{\tils, \tila, \tilh}(\tils', \tilh') }{\Np_k(s, a)}
	    \\
	    & \lesssim
	    \Tmax S \sum_{k,s,a} \frac{ q_k(s,a)  }{\Np_k(s, a)}
	    \lesssim
	    \Tmax^2 S^2 A,
	\end{align*}
	where the last inequality is by \pref{lem:n sum}.
	For the second term we have with probability at least $1-\delta$:
    \begin{align*}
        &\sum_{k,z,\tilde z} \frac{q_k(s, a, h)P_{s, a, h}(s', h')q_{k,(s', h')}(\tils, \tila, \tilh)}{\Np_k(\tils, \tila)}
        \\
        & \lesssim
        S \sum_{k,s,a,h,\tils,\tila,\tilh} \frac{q_k(s, a, h) \sum_{s',h'} P_{s, a, h}(s', h')q_{k,(s', h')}(\tils, \tila, \tilh) }{\Np_k(\tils, \tila)}
        \\
        & \lesssim
        \Tmax S \sum_{k,s,a,h,\tils,\tila,\tilh} \frac{x_k(s, a, h) \sum_{s',h'} P_{s, a, h}(s', h')q_{k,(s', h')}(\tils, \tila, \tilh) }{\Np_k(\tils, \tila)}
        \\
        & \lesssim
        \Tmax S \sum_{k,s,a,h,\tils,\tila,\tilh} \frac{q_k(\tils, \tila, \tilh) }{\Np_k(\tils, \tila)}
        \lesssim
        \Tmax S^2 A \sum_{k,\tils,\tila} \frac{q_k(\tils, \tila) }{\Np_k(\tils, \tila)}
        \lesssim
        S^3 A^2\Tmax^2,
    \end{align*}
    where the second inequality follows by $q_k(s,a,h) \lesssim \Tmax x_k(s,a,h)$, the third by $x_k(s, a, h)\sum_{s', h'}P_{s, a, h}(s', h')q_{k, (s', h')}(\tils, \tila, \tilh) \leq q_k(\tils, \tila, \tilh)$, and the last one by \pref{lem:n sum}.
\end{proof}

\begin{lemma}[Extended Value Difference]
	\label{lem:value diff}
	For any policies $\pi, \pi'$, transitions $P, P'$, and cost functions $c, c'$ in $\rcalM$, we have:
	\begin{align*}
		&Q^{\pi, P, c}(s, a, h) - Q^{\pi', P', c'}(s, a, h)\\ 
		&= \sum_{s'', h''}P'_{s,a,h}(s'', h'')\sum_{s', h'}q_{\pi', P', (s'', h'')}(s', h')\sum_{a'}\rbr{\pi(a'|s', h') - \pi'(a'|s', h')}Q^{\pi, P, c}(s', a', h') \\
		&\qquad + \sum_{s', a', h'}q_{\pi', P', (s, a, h)}(s', a', h')\rbr{Q^{\pi, P, c}(s', a', h') - c'(s', a', h') - P'_{s', a', h'}V^{\pi, P, c}}.
	\end{align*}
	and
	\begin{align*}
		&V^{\pi, P, c}(s, h) - V^{\pi', P', c'}(s, h)\\
		&= \sum_{s', h'}q_{\pi', P', (s, h)}(s', h')\sum_{a'}\rbr{\pi(a'|s', h') - \pi'(a'|s', h')}Q^{\pi, P, c}(s', a', h')\\ 
		&\qquad +  \sum_{s', a', h'}q_{\pi', P', (s, h)}(s', a', h')\rbr{Q^{\pi, P, c}(s', a', h') - c'(s', a', h') - P'_{s', a', h'}V^{\pi, P, c}}.
	\end{align*}	
\end{lemma}
\begin{proof}
	We first prove the second statement, note that:
	\begin{align*}
		&V^{\pi, P, c}(s, h) - V^{\pi', P', c'}(s, h) = \sum_{a'}\rbr{\pi(a'|s, h) - \pi'(a'|s, h)}Q^{\pi, P, c}(s, a', h)\\
		&\qquad + \sum_{a'}\pi'(a'|s, h)(Q^{\pi,P,c}(s, a', h) - Q^{\pi',P',c'}(s, a', h))\\
		&= \sum_{a'}\rbr{\pi(a'|s, h) - \pi'(a'|s, h)}Q^{\pi, P, c}(s, a', h)\\ 
		&\qquad + \sum_{a'}\pi'(a'|s, h)\rbr{Q^{\pi,P,c}(s, a', h) - c'(s, a', h) -  P'_{s, a', h}V^{\pi,P,c}}\\
		&\qquad + \sum_{a'}\pi'(a'|s, h)P'_{s, a', h}(V^{\pi, P, c} - V^{\pi', P', c'}).
	\end{align*}
	Applying the equality above recursively and by the definition of $q_{\pi',P',(s,h)}$, we prove the second statement.
	For the first statement, note that:
	\begin{align*}
		&Q^{\pi, P, c}(s, a, h) - Q^{\pi', P', c'}(s, a, h)\\
		&= \rbr{Q^{\pi, P, c}(s, a, h) - c'(s, a, h) - P'_{s,a,h}V^{\pi,P,c}} + P'_{s,a,h}(V^{\pi,P,c} - V^{\pi',P',c'}).
	\end{align*}
	Applying the second statement and the definition of $q_{\pi',P',(s,a,h)}$ completes the proof.
\end{proof}

\begin{lemma}{\citep[Lemma 6]{rosenberg2020adversarial}}
	\label{lem:hitting}
	Let $\pi$ be a policy with expected hitting time at most $\tau$ starting from any state.
	Then for any $\delta\in(0, 1)$, with probability at least $1-\delta$, $\pi$ takes no more than $4\tau\ln\frac{2}{\delta}$ steps to reach the goal state.
\end{lemma}


\begin{lemma}
	\label{lem:n sum}
	For any $z_k:\SA\rightarrow [0, 1]$, with probability at least $1-\delta$,
	\begin{align*}
		\sumk\sumsa\frac{\bar{n}_k(s, a)\sqrt{z_k(s, a)}}{\sqrt{\Np_k(s, a)}} &= \tilO{SA\Tmax + \sqrt{SA\sum_k \sumsa \bar{n}_k(s, a)z_k(s, a) }}\\ 
		&= \tilO{SA\Tmax + \sqrt{SA\sum_k\sumsa q_k(s,a)z_k(s, a)}},\\
		\sumk\sumsa\frac{q_k(s, a)\sqrt{z_k(s, a)}}{\sqrt{\Np_k(s, a)}} &= \tilO{SA\Tmax + \sqrt{SA\sum_k \sumsa \bar{n}_k(s,a)z_k(s,a) }}\\ 
		&= \tilO{SA\Tmax + \sqrt{SA\sum_k\sumsa q_k(s,a)z_k(s,a) }},\\
		\sumk\sumsa\frac{\bar{n}_k(s, a)}{\Np_k(s, a)} &= \tilO{SA\Tmax}, \qquad \sumk\sumsa\frac{q_k(s, a)}{\Np_k(s, a)} = \tilO{SA\Tmax},\\
		\sumk\sumsa\frac{m_k(s, a)}{\Mp_k(s, a)}&=\tilO{SA}, \qquad \sumk\sumsa\frac{x_k(s, a)}{\Mp_k(s, a)}=\tilO{SA}.
	\end{align*}
\end{lemma}
\begin{proof}
	\textbf{First statement:} Since $z_k(s,a) \le 1$ and $\bar{n}_k(s,a) \le L = \tilO{\Tmax}$ we have:
	\begin{align*}
		\sumk\frac{\bar{n}_k(s, a)\sqrt{z_k(s, a)}}{\sqrt{\Np_k(s, a)}} &\leq \sumk\frac{\bar{n}_k(s, a)\sqrt{z_k(s, a)}}{\sqrt{\Np_{k+1}(s, a)}} + \sumk L\rbr{\frac{1}{\sqrt{\Np_k(s, a)}} - \frac{1}{\sqrt{\Np_{k+1}(s, a)}} }
		\\
		&\leq \sumk\frac{\bar{n}_k(s, a)\sqrt{z_k(s, a)}}{\sqrt{\Np_{k+1}(s, a)}} + \tilO{\Tmax}.
	\end{align*}
	By Cauchy-Schwarz inequality this implies:
	\begin{align*}
	    \sum_{s,a} \sumk\frac{\bar{n}_k(s, a)\sqrt{z_k(s, a)}}{\sqrt{\Np_k(s, a)}} &= \tilO{ \sqrt{\sumk\sumsa \frac{\bar{n}_k(s,a)}{\Np_{k+1}(s,a)}}\sqrt{\sumk\sumsa \bar{n}_k(s,a)z_k(s,a)} + SA\Tmax}\\
	    &=\tilO{\sqrt{SA\sum_k\sum_{s,a}\bar{n}_k(s, a)z_k(s, a)} + SA\Tmax}.
	\end{align*}
	Finally, $\sum_{s,a} \sum_k\bar{n}_k(s, a)z_k(s, a)= \tilO{\sum_{s,a} \sum_kq_k(s, a)z_k(s, a) + SA\Tmax}$ with high probability by \pref{lem:e2r}.
	
	\textbf{Second statement:} By \pref{lem:bound q} we have:
	\begin{align*}
	    \sumk\sumsa\frac{q_k(s, a)\sqrt{z_k(s, a)}}{\sqrt{\Np_k(s, a)}}
	    &\le
	    \sumk\sumsa\frac{\E_k[\bar{n}_k(s, a)]\sqrt{z_k(s, a)}}{\sqrt{\Np_k(s, a)}} + \tilO{1/K} \sumk\sumsa\frac{\sqrt{z_k(s, a)}}{\sqrt{\Np_k(s, a)}}
	    \\
	    &\le \sumk\sumsa\frac{\E_k[\bar{n}_k(s, a)]\sqrt{z_k(s, a)}}{\sqrt{\Np_k(s, a)}} + \tilO{SA}
	    \\
	    &= \tilO{\sumk\sumsa\frac{\bar{n}_k(s, a)\sqrt{z_k(s, a)}}{\sqrt{\Np_k(s, a)}} + \Tmax S A},
	\end{align*}
	where the last relation holds with high probability by \pref{lem:e2r}.
	Now the statement follows by the first statement.
	
	\textbf{Third and forth statements:} Similarly to the first statement,
	\begin{align*}
	    \sumk\frac{\bar{n}_k(s, a)}{\Np_k(s, a)}
	    &\le
	    \sumk\frac{\bar{n}_k(s, a)}{\Np_{k+1}(s, a)} + \sumk L \rbr{\frac{1}{\Np_{k}(s, a)}-\frac{1}{\Np_{k+1}(s, a)}}
	    \\
	    & \le
	    \sumk \frac{\bar{n}_k(s, a)}{\max\{1, \sum_{i \le k} \bar{n}_i(s, a)\}} + \tilO{\Tmax}
	    =
	    \tilO{\Tmax}.
	\end{align*}
	Summing over $(s,a)$ proves the third statement.
	The forth statement is then proved similarly to the second statement.
	
	\textbf{Fifth and sixth statements:} Similarly to the third statement,
	\begin{align*}
	    \sumk\frac{m_k(s, a)}{\Mp_k(s, a)}
	    &\le
	    \sumk\frac{m_k(s, a)}{\Mp_{k+1}(s, a)} + \sumk \rbr{\frac{1}{\Mp_{k}(s, a)}-\frac{1}{\Mp_{k+1}(s, a)}}
	    \\
	    & \le
	    \sumk \frac{m_k(s, a)}{\max\{1, \sum_{i \le k} m_i(s, a)\}} + 1
	    =
	    \tilO{1}.
	\end{align*}
	Summing over $(s,a)$ proves the fifth statement.
	The sixth statement is again obtained with high probability by \pref{lem:e2r}.
\end{proof}

%% file: app-adv.tex

\paragraph{Extra Notations}
Define $\tilq_k=q_{\pi_k,\tilP_k}$, $Q_k=Q^{\pi_k, P, c_k}$, $V_k=V^{\pi_k,P,c_k}$, and $A_k=A^{\pi_k,P,c_k}$.


\subsection{\pfref{thm:full}}
\label{app:proof_adv_full}
In this part, define $\tilP_k=\Gamma(\pi_k,\calP_k,\tilc_k)$ and $P_k=\Gamma(\pi_k,\calP_k,c_k)$, such that $\tilQ_k=Q^{\pi_k,\tilP_k,\tilc_k}$, $\tilV_k=V^{\pi_k,\tilP_k,\tilc_k}$, and $\hatQ_k=Q^{\pi_k,P_k,c_k}$.
We first provide bounds on some important quantities.
\begin{lemma}
	\label{lem:bound}
	$\tilc_k\in\calC_{\calM}$, $\eta\norm{\tilA_k-B_k}_{\infty}\leq1$, and $\eta\norm{B_k}_{\infty}\leq\frac{1}{2H'}$.
\end{lemma}
\begin{proof}
	For the first statement, by $P_k\in\Lambda_{\calM}$, we have $\hatQ_k(s, a, h)\leq\frac{H}{1-\gamma}+c_f=\chi$.
	Therefore, $\lambda\hatQ_k(s, a, h)\leq 1$ and $\tilc_k\in\calC_{\calM}$.
	For the second statement, by $\tilc_k(s,a,h)\leq 2$ for $h\leq H$, we have $|\tilA_k(s, a, h)| \leq |\tilQ_k(s,a,h)|+|\tilV_k(s,h)|\leq 4(\frac{H}{1-\gamma}+c_f)=4\chi$ for $h\leq H$.
	Therefore, $\norm{b_k}_{\infty}\leq 32\eta\chi^2$, and by \pref{lem:bound B}, we have $\norm{B_k}_{\infty}\leq \frac{15H\norm{b_k}_{\infty}}{1-\gamma}\leq 960\eta H\Tmax\chi^2$.
	Thus by the definition of $\eta$, we have $\eta\norm{B_k}_{\infty}\leq\frac{1}{2H'}$ and $\eta\norm{\tilA_k-B_k}_{\infty}\leq\eta(\norm{\tilA_k}_{\infty}+\norm{B_k}_{\infty})\leq 1$.
\end{proof}
We are now ready to prove \pref{thm:full}.
\begin{proof}[\pfref{thm:full}]
	With probability at least $1-10\delta$, we decompose the regret as follows:
	\begin{align*}
		&\rR_K = \sumk \inner{n_k - q_k}{c_k} + \inner{q_k - \optq}{c_k} \overset{\text{(i)}}{\leq} \tilO{\sqrt{\sumk\inner{q_k}{c_k\circ Q_k}} + SA\Tmax} \\
		&\qquad + \sumk \inner{q_k-\tilq_k}{\tilc_k} + \sumk\inner{\tilq_k-\optq}{\tilc_k} -\lambda\sumk\inner{q_k}{c_k\circ\hatQ_k} +\lambda\sumk\inner{\optq}{c_k\circ\hatQ_k}\\
		&\overset{\text{(ii)}}{=} \tilO{\sqrt{S^2A\sumk\inner{q_k}{c_k\circ Q_k}} + S^{2.5}A^{1.5}\Tmax^3} + \sumk\inner{\tilq_k-\optq}{\tilc_k} - \lambda\sumk\inner{q_k}{c_k\circ Q_k}\\
		&+ \lambda\sumk\inner{q_k}{c_k\circ(Q_k-\hatQ_k)} + \lambda\sumk\inner{\optq}{c_k\circ Q^{\roptpi, P, c_k}} + \lambda\sumk\inner{\optq}{c_k\circ(\hatQ_k - Q^{\roptpi, P, c_k})},
	\end{align*}
	where in (i) we apply \pref{lem:bound J} and \pref{lem:freedman} to have
	\begin{align*}
		\sumk\inner{n_k-q_k}{c_k} &= \sumk\inner{\bar{n}_k - q_k}{c_k} = \tilO{\sqrt{\sumk\E_k[\inner{\bar{n}_k}{c_k}^2}] + SA\Tmax} \\
		&= \tilO{ \sqrt{\sumk\inner{q_k}{c_k\circ Q_k}} + SA\Tmax }, \tag{\pref{lem:var}}
	\end{align*}
	and in (ii) we apply \pref{lem:q-qk} and $\tilc_k\in\calC_{\calM}$ for $h\leq H$ to have
	\begin{align*}
		\sumk\inner{q_k - \tilq_k}{\tilc_k} &= \tilO{ \sqrt{S^2A\sumk\inner{q_k}{\tilc_k\circ Q^{\pi_k, P, \tilc_k}}} + S^{2.5}A^{1.5}\Tmax^3}\\ 
		&= \tilO{ \sqrt{S^2A\sumk\inner{q_k}{c_k\circ Q^{\pi_k, P, c_k}} } + S^{2.5}A^{1.5}\Tmax^3} \tag{$\tilc_k(s, a,h)\leq 2c_k(s,a,h)$}. 
	\end{align*}
	Define $\lambda'=\sqrt{\frac{S^2A}{D\T K}}$.
	By \pref{lem:tilq-optq}, \pref{lem:hatQ-optQ}, \pref{lem:q Q-hatQ}, and definition of $\lambda,\eta$, with probability at least $1-9\delta$:
	\begin{align*}
		\rR_K &\leq \tilO{ \sqrt{S^2A\sumk\inner{q_k}{c_k\circ Q_k}} + \frac{\T}{\eta} + S^4A^2\Tmax^5}\\
		&\qquad + 24\eta\sumk\inner{q_k}{A^2_k} -\lambda\sumk\inner{q_k}{c_k\circ Q_k} + \lambda\sumk\inner{\optq}{c_k\circ Q^{\roptpi, P, c_k}}\\
		&= \tilO{ \frac{S^2A}{\lambda'} } + \lambda'\sumk\inner{q_k}{c_k\circ Q_k} + \tilO{\frac{\T}{\eta} + S^4A^2\Tmax^5}\\
		&\qquad + 48\eta\sumk\inner{q_k}{c_k\circ Q_k} - \lambda\sumk\inner{q_k}{c_k\circ Q_k} + \bigO{\lambda D\T K}\tag{AM-GM inequality, \pref{lem:var}, and \pref{lem:qcQ adv}}\\
		&= \tilO{\T\sqrt{DK} + \sqrt{S^2AD\T K} + S^4A^2\Tmax^5 }. \tag{$K=\tilo{S^2A\Tmax^2}$ when $\lambda<48\eta+\lambda'$}
	\end{align*}
	Applying \pref{lem:approx} completes the proof.
\end{proof}

\begin{lemma}
	\label{lem:tilq-optq}
	With probability at least $1-6\delta$,
	$$\sumk\inner{\tilq_k - \optq}{\tilc_k}= 24\eta\inner{q_k}{A_k^2} + \tilO{\frac{\T}{\eta} + S^4A^2\Tmax^{3.5}}.$$
\end{lemma}
\begin{proof}
	Note that by \pref{lem:expw} and \pref{lem:bound}:
	\begin{align*}
		&\sumk\sum_{s, h}\optq(s, h)\suma(\pi_k(a|s, h)-\roptpi(a|s, h))\rbr{\tilA_k(s, a, h) - B_k(s, a, h)}\\
		&\leq \sum_{s, h}\optq(s, h)\rbr{ \frac{\ln A}{\eta} + \eta\sumk\suma\pi_k(a|s, h)\rbr{\tilA_k(s, a, h) - B_k(s, a, h)}^2}\\
		&\leq \tilO{\frac{\T}{\eta}} + 2\eta\sum_{s, h}\optq(s, h)\rbr{\sumk\suma\pi_k(a|s, h)\tilA_k(s, a, h)^2 + \sumk\suma\pi_k(a|s, h)B_k(s, a, h)^2}\\
		&= \tilO{\frac{\T}{\eta}} + \sumk\inner{\optq}{b_k} + \frac{1}{H'}\sum_{s, h}\optq(s, h)\sumk\suma\pi_k(a|s, h)B_k(s, a, h). \tag{\pref{lem:bound}}
	\end{align*}
	Define $\hatq'_k=q_{\pi_k,P'_k}$, where $P'_k$ is the optimistic transition defined in $B_k$.
	We have
	\begin{align*}
		&\sumk\inner{\tilq_k-\optq}{\tilc_k} = \sumk\sum_{s, h}\optq(s, h)\suma(\pi_k(a|s, h)-\roptpi(a|s, h))\rbr{\tilA_k(s, a, h) - B_k(s, a, h)} \\
		&\qquad + \sumk\sum_{s, a, h}\optq(s, a, h)\rbr{ \tilQ_k(s, a, h) - \tilc_k(s, a, h) - P_{s, a, h}\tilV_k }\\
		&\qquad + \sumk\sum_{s, h}\optq(s, h)\suma(\pi_k(a|s, h)-\roptpi(a|s, h))B_k(s, a, h) \tag{shifting argument and \pref{lem:value diff}}\\
		&\overset{\text{(i)}}{\leq} \tilO{\frac{\T}{\eta}} + 3\sumk\inner{\hatq'_k}{b_k} + \tilO{\Tmax}\\
		&= \tilO{\frac{\T}{\eta}} + 6\eta\sumk\inner{q_k}{\tilA_k^2} + 3\sumk\inner{\hatq'_k-q_k}{b_k},
	\end{align*}
	where in (i) we apply \pref{lem:dilated bonus}, $b_k(s, a, h)=\tilo{1}$, and the definition of $\tilP_k$ so that
	\begin{align*}
		&\sumk\sum_{s, a, h}\optq(s, a, h)\rbr{ \tilQ_k(s, a, h) - \tilc_k(s, a, h) - P_{s, a, h}\tilV_k } \leq 0.
	\end{align*}
	For the second term, by $(a+b+c)^2\leq 2a^2+2(b+c)^2\leq 2a^2+4b^2+4c^2$,
	\begin{align*}
		&\eta\sumk\sum_{s, a, h}q_k(s, a, h)\tilA_k(s, a, h)^2\leq 2\eta\sumk\sum_{s, a, h}q_k(s, a, h)A^{\pi_k, P, \tilc_k}(s, a, h)^2\\
		&\qquad + 4\eta\sumk\sum_{s, a, h}q_k(s, a, h)\rbr{Q^{\pi_k, \tilP_k,\tilc_k}(s, a, h) - Q^{\pi_k, P,\tilc_k}(s, a, h)}^2\\
		&\qquad + 4\eta\sumk\sum_{s, h}q_k(s, h)\rbr{V^{\pi_k, \tilP_k,\tilc_k}(s, h) - V^{\pi_k, P,\tilc_k}(s, h)}^2\\
		&\leq 2\eta\sumk\sum_{s, a, h}q_k(s, a, h)A^{\pi_k, P, \tilc_k}(s, a, h)^2\\ 
		&\qquad + 8\eta\sumk\sum_{s, a, h}q_k(s, a, h)\rbr{Q^{\pi_k, \tilP_k,\tilc_k}(s, a, h) - Q^{\pi_k, P,\tilc_k}(s, a, h)}^2,
	\end{align*}
	where in the last step we apply Cauchy-Schwarz inequality to obtain
	\begin{align*}
		\rbr{V^{\pi_k, \tilP_k, \tilc_k}(s, h) - V^{\pi_k, P, \tilc_k}(s, h)}^2 &=\rbr{\sum_a\pi_k(a|s, h)(Q^{\pi_k, \tilP_k, \tilc_k}(s, a, h) - Q^{\pi_k, P, \tilc_k}(s, a, h))}^2\\
		&\leq \sum_a\pi_k(a|s, h)\rbr{Q^{\pi_k, \tilP_k, \tilc_k}(s, a, h) - Q^{\pi_k, P, \tilc_k}(s, a, h)}^2.
	\end{align*}
	Note that with probability at least $1-2\delta$,
	\begin{align}
		&\abr{Q^{\pi_k, \tilP_k, \tilc_k}(s, a, h) - Q^{\pi_k, P, \tilc_k}(s, a, h)}\label{eq:Q-hatQ}\\
		&= \tilO{\Tmax\sum_{s', a', h'\leq H}q_{\pi_k, P, (s, a, h)}(s', a', h')\norm{P_{s', a', h'} - \tilP_{k, s', a', h'}}_1}\tag{\pref{lem:value diff}, H\"older's inequality, and $V^{\pi_k,\tilP_k,\tilc_k}=\tilo{\Tmax}$}\\
		&= \tilO{ \Tmax S\sum_{s', a'}\frac{q_{\pi_k, P, (s, a, h)}(s', a')}{\sqrt{\Np_k(s', a')}} }. \tag{\pref{lem:conf bound}}
	\end{align}
	Therefore, with probability at least $1-\delta$,
	\begin{align*}
		&\eta\sumk\sum_{s, a, h}q_k(s, a, h)\rbr{Q^{\pi_k, \tilP_k,\tilc_k}(s, a, h) - Q^{\pi_k, P,\tilc_k}(s, a, h)}^2\\ 
		&\leq \eta\sumk\sum_{s, a, h\leq H}q_k(s, a, h)\Tmax^3S^2\sum_{s', a'}\frac{q_{\pi_k, P, (s, a, h)}(s', a')}{\N^+_k(s', a')}\tag{Cauchy-Schwarz inequality}\\
		&\overset{\text{(i)}}{=} \tilO{\eta \Tmax^4S^3A\sumk\sum_{s', a'}\frac{q_{\pi_k, P}(s', a')}{\N^+_k(s', a')}} \overset{\text{(ii)}}{=} \tilO{\eta \Tmax^5S^4A^2},
	\end{align*}
	where in (i) we apply $q_k(s,a,h)\leq2\Tmax x_k(s,a,h)$ and $x_k(s, a, h)q_{\pi_k, P, (s, a, h)}(s', a')\leq q_{\pi_k, P}(s', a')$, and in (ii) we apply \pref{lem:n sum}.
	Plugging these back, we get:
	\begin{align*}
		&\eta\sumk\sum_{s, a, h}q_k(s, a, h)\tilA_k(s, a, h)^2 \leq 2\eta\sumk\inner{q_k}{(A^{\pi_k, P, \tilc_k})^2} + \tilO{\eta \Tmax^5S^4A^2}\\
		&\leq 4\eta\sumk\inner{q_k}{A_k^2} + 4\eta\lambda^2\sumk\inner{q_k}{(A^{\pi_k, P, \hatQ_k})^2} + \tilO{\eta \Tmax^5S^4A^2} \tag{$(a+b)^2\leq2a^2+2b^2$}\\
		&\leq 4\eta\inner{q_k}{A_k^2} + \tilO{\eta \Tmax^5S^4A^2 + \eta\lambda^2\Tmax^5K} = 4\eta\inner{q_k}{A_k^2} + \tilO{S^4A^2\Tmax^3}.
	\end{align*}
	For the third term, with probability at least $1-3\delta$,
	\begin{align*}
		&\sumk\inner{\hatq'_k-q_k}{b_k} \leq \sumk\sum_{s, a, h\leq H}q_k(s, a, h)\norm{ \hatP'_{k, s, a, h} - P_{s, a, h} }_1\norm{V^{\pi_k, \hatP'_k, b_k}}_{\infty} \tag{\pref{lem:value diff} and H\"older's inequality}\\
		&= \tilO{ \eta S\Tmax^3\sumk\sum_{s, a, h\leq H}\frac{q_k(s, a, h)}{\sqrt{\N^+_k(s, a)}} } \tag{$b_k(s, a, h)=\tilo{\eta\Tmax^2}$ and \pref{lem:conf bound}} \\
		&= \tilO{\eta S\Tmax^3\sqrt{SA\Tmax K} + \eta S^2A\Tmax^4} = \tilO{S^2A\Tmax^{3.5}}.\tag{\pref{lem:n sum}}
	\end{align*}
	Putting everything together completes the proof.
\end{proof}

\begin{lemma}
	\label{lem:hatQ-optQ}
	$\lambda\sumk\inner{\optq}{c_k\circ(\hatQ_k - Q^{\roptpi, P, c_k})} = \tilO{S^2A\Tmax^5}$.
\end{lemma}
\begin{proof}
	Define $q'_{s,a,h}(s', h')=\sum_{s'',h''}P_{s,a,h}(s'',h'')\optq_{(s'',h'')}(s', h')$.
	We have for $h\leq H$:
	\begin{align*}
		&\lambda\sumk (\hatQ_k(s, a, h) - Q^{\roptpi, P, c_k}(s, a, h)) \leq \lambda\sumk (Q^{\pi_k,\tilP_k,c_k}(s, a, h) - Q^{\roptpi, P, c_k}(s, a, h))\\
		&\leq \lambda\sumk \rbr{Q^{\pi_k, \tilP_k, \tilc_k}(s, a, h) - Q^{\roptpi, P, \tilc_k}(s, a, h)} + \lambda^2\sumk Q^{\roptpi, P, \hatQ_k}(s, a, h)\\
		&\leq \lambda\sumk \rbr{Q^{\pi_k, \tilP_k, \tilc_k}(s, a, h) - Q^{\roptpi, P, \tilc_k}(s, a, h)} + \tilO{\lambda^2\Tmax^2K}\\
		&\leq \lambda\sum_{s', h'}q'_{s, a, h}(s', h')\sumk\sum_{a'}\rbr{\pi_k(a'|s', h')-\roptpi(a'|s', h')}\rbr{\tilA_k(s', a', h') - B_k(s', a', h')}\\
		&\qquad + \lambda\sum_{s', a', h'}\optq_{(s, a, h)}(s', a', h')\sumk\rbr{ Q^{\pi_k, \tilP_k, \tilc_k}(s', a', h') - \tilc_k(s', a', h') - P_{s', a', h'}V^{\pi_k, \tilP_k, \tilc_k}}\\
		&\qquad + \lambda\sum_{s', h'}q'_{s, a, h}(s', h')\sumk\sum_{a'}\rbr{\pi_k(a'|s', h')-\roptpi(a'|s', h')}B_k(s', a', h') + \tilO{\lambda^2\Tmax^2K}\tag{\pref{lem:value diff}}\\
		&= \tilO{\lambda\sum_{s', h'}q'_{s, a, h}(s', h')\rbr{ \frac{\T}{\eta} + \eta\sumk\suma\pi_k(a|s, h)\Tmax^2} + \lambda\eta\Tmax^4 K + \lambda^2\Tmax^2K } \tag{$\norm{\tilA_k}_{\infty}=\tilo{\Tmax}$, definition of $\tilP_k$, and $B_k(s, a, h)=\tilo{\eta\Tmax^3}$}\\
		&= \tilO{\frac{\lambda\Tmax^2}{\eta} + \lambda\eta\Tmax^4K + \lambda^2\Tmax^2K}.
	\end{align*}
	Plugging this back and by the definition of $\lambda,\eta$:
	\begin{align*}
		\lambda\sumk\inner{\optq}{\hatQ_k - Q^{\roptpi, P, c_k}} = \tilO{\frac{\lambda\Tmax^3}{\eta} + \lambda\eta\Tmax^5K + \lambda^2\Tmax^3K} = \tilO{S^2A\Tmax^5}.
	\end{align*}
	This completes the proof.
\end{proof}

\begin{lemma}
	\label{lem:q Q-hatQ}
	With probability at least $1-3\delta$,
	$\lambda\sumk\inner{q_k}{c_k\circ(Q_k-\hatQ_k)}=\tilO{ S^{3.5}A^2\Tmax^3 }$.
\end{lemma}
\begin{proof}
	By similar arguments as in \pref{eq:Q-hatQ} with $\tilP_k$ replaced by $P_k$ and $\tilc_k$ replaced by $c_k$, with probability at least $1-3\delta$:
	\begin{align*}
		&\lambda\sumk\inner{q_k}{Q_k-\hatQ_k} = \lambda\Tmax S\sumk\sum_{s, a, h\leq H}q_k(s, a, h) \sum_{s', a'}\frac{q_{\pi_k,P,(s, a, h)}(s', a')}{\sqrt{\N^+_k(s', a')}}\\
		&= \tilO{\lambda S^2A\Tmax^2 \sumk\sum_{s', a'}\frac{q_k(s', a')}{\sqrt{\Np_k(s', a')}}} \tag{$q_k(s, a, h)=\bigo{\Tmax x_k(s, a, h)}$ and $x_k(s, a, h)q_{\pi_k, P, (s, a, h)}(s', a')\leq q_k(s', a')$}\\
		&= \tilO{ \lambda S^2A\Tmax^2\sqrt{SA\Tmax K} + \lambda S^3A^2\Tmax^3 } = \tilO{ S^{3.5}A^2\Tmax^3 }. \tag{\pref{lem:n sum}}
	\end{align*}
\end{proof}

\begin{lemma}
	\label{lem:qcQ adv}
	$\sumk\inner{\optq}{c_k\circ Q^{\roptpi,P,c_k}}=\tilO{D\T K} + \Tmax$.
\end{lemma}
\begin{proof}
	By \pref{lem:sda}, for $h\leq H$, $\sumk c_k(s,a,h)Q^{\roptpi,P,c_k}(s,a,h)\leq\sumk(Q^{\optpi,P,c_k}(s,a) + c_f)=\tilo{DK}$.
	Therefore,
	\begin{align*}
		&\sumk\inner{\optq}{c_k\circ Q^{\roptpi,P,c_k}}\\
		&= \sumk\sum_{s,a,h\leq H}\optq(s,a,h)c_k(s,a,h)Q^{\roptpi,P,c_k}(s,a,h) + \sumk\sum_{s,a}\optq(s,a,H+1)c_f\\
		&\leq \tilO{D\T K} + \Tmax,
	\end{align*}
	where the last step is by $\sum_{s,a,h\leq H}\optq(s,a,h)=\bigo{\T}$, $\sumk Q^{\roptpi, P, c_k}(s, a, h)=\bigo{DK}$, and \pref{lem:qh}.
\end{proof}

\subsection{\pfref{thm:bandit}}\label{app:proof_adv_bandit}
Here we denote by $\hatq'_k$ the occupancy measure w.r.t policy $\pi_k$ and the optimistic transition defined in $B_k$.
Also define $\barQ_k(s, a, h)=\E_k[\sum_{i=1}^{\min\{J_k, L\}+1} c(s^k_i, a^k_i, h^k_i)|\pi_k, P, s^k_1=s, a^k_1=a, h^k_1=h]$, such that $\E_k[G_{k, s, a, h}]=x_k(s, a, h)\barQ_k(s, a, h)$. 
\begin{proof}
	With probability at least $1-2\delta$, we decompose the regret as follows:
	\begin{align*}
		\rR_K &= \sumk\inner{n_k - \optq}{c_k} = \sumk\inner{\bar{n}_k-q_k}{c_k} + \sumk\inner{q_k-\optq}{c_k} \tag{\pref{lem:bound J}}\\
		&= \tilO{\Tmax\sqrt{K} + SA\Tmax} + \sumk\sum_{s, h}\optq(s, h)\suma(\pi_k(a|s, h)-\roptpi(a|s, h))Q_k(s, a, h) \tag{\pref{lem:freedman} and \pref{lem:value diff}}\\
		&= \tilO{\Tmax\sqrt{K} + SA\Tmax} + \underbrace{\sumk\sum_{s, h}\optq(s, h)\suma(\pi_k(a|s, h)-\roptpi(a|s, h))\tilQ_k(s, a, h)}_{\reg}\\
		&\qquad + \underbrace{\sumk\sum_{s, h}\optq(s, h)\suma\pi_k(a|s, h)(Q_k(s, a, h)-\tilQ_k(s, a, h))}_{\bias_1}\\
		&\qquad + \underbrace{\sumk\sum_{s, h}\optq(s, h)\suma\roptpi(a|s, h)(\tilQ_k(s, a, h)-Q_k(s, a, h))}_{\bias_2}.
	\end{align*}
	Therefore, by \pref{lem:reg sum}, we have $\rR_K = \tilo{\sqrt{S^2A\Tmax^5K} + S^{5.5}A^{3.5}\Tmax^5}$ with probability at least $1-25\delta$.
	Applying \pref{lem:approx} completes the proof.
\end{proof}

\begin{lemma}
	\label{lem:bound bandit}
	$\norm{\tilQ_k}_{\infty}\leq L'/\theta$, $\norm{b_k}_{\infty}\leq 5L'$, $\norm{B_k}_{\infty}\leq 150H\Tmax L'$, and $\eta\norm{\tilQ_k - B_k}_{\infty}\leq 1$.
\end{lemma}
\begin{proof}
	The first statement is by the definition of $\tilQ_k$ and $G_{k,s,a,h}\leq L'$.
	For the second statement, $b_k\leq 5L'$ by definition.
	For the third statement, by \pref{lem:bound B}, we have $\norm{B_k}_{\infty}\leq \frac{15H\norm{b_k}_{\infty}}{1-\gamma}\leq 150H\Tmax L'$.
	For the fourth statement, we have $\eta\norm{\tilQ_k-B_k}_{\infty}\leq \eta(\norm{\tilQ_k}_{\infty}+\norm{B_k}_{\infty}) \leq 1/2 + \eta 150H\Tmax L' \leq 1$.
\end{proof}

\begin{lemma}
	\label{lem:Q-barQ}
	$Q_k(s, a, h) - \bar{Q}_k(s, a, h) = \tilO{1/K}$.
\end{lemma}
\begin{proof}
	Note that:
	\begin{align*}
		Q_k(s, a, h) - \bar{Q}_k(s, a, h) &= \E_k\sbr{\left.\sum_{i=\bar{J}_k+2}^{J_k+1}c(s^k_i, a^k_i, h^k_i) \right|\pi_k, P, s^k_1=s, a^k_1=a, h^k_1=h}\\
		&= \tilO{\frac{\Tmax}{\Tmax K}} = \tilO{1/K}. \tag{\pref{lem:hitting}}
	\end{align*}
\end{proof}

\begin{lemma}
	\label{lem:reg sum}
	With probability at least $1-25\delta$,
	\begin{align*}
		\reg+\bias_1+\bias_2 &= \tilO{\sqrt{S^2A\Tmax^5K} + S^{5.5}A^{3.5}\Tmax^5}.
	\end{align*}
\end{lemma}
\begin{proof}
	Define $\xi_B=\sumk\sum_{s, h\leq H}\optq(s, h)\suma(\pi_k(a|s, h)-\roptpi(a|s, h))B_k(s, a, h)$.
	By \pref{lem:bound bandit} and \pref{lem:expw}, with probability at least $1-\delta$,
	\begin{align*}
		\reg&=\sumk\sum_{s, h\leq H}\optq(s, h)\suma(\pi_k(a|s, h)-\roptpi(a|s, h))\rbr{\tilQ_k(s, a, h) -B_k(s, a, h)} + \xi_B\\
		&\leq \frac{\T\ln A}{\eta} + \eta\sum_{s, h\leq H}\optq(s, h)\sumk\suma\pi_k(a|s, h)\rbr{\tilQ_k(s, a, h) -B_k(s, a, h)}^2 + \xi_B\\
		&\leq \frac{\T\ln A}{\eta} + 2\eta\sum_{s, h \leq H}\optq(s, h)\sumk\suma\pi_k(a|s, h)\tilQ_k^2(s, a, h)\\
		&\qquad + \frac{1}{H'}\sum_{s, h\leq H}\optq(s, h)\sumk\suma\pi_k(a|s, h)B_k(s, a, h) + \xi_B\tag{$(a+b)^2\leq 2a^2+2b^2$ and $\eta\norm{B_k}_{\infty}\leq\frac{1}{H'}$}\\
		&\leq \tilO{\frac{\T}{\eta}} + \sum_{s, h\leq H}\optq(s, h)\sumk\suma\pi_k(a|s, h)\frac{2\theta L'}{\ux_k(s, a, h)+\theta}\\
		&\qquad + \frac{1}{H'}\sum_{s, h\leq H}\optq(s, h)\sumk\suma\pi_k(a|s, h)B_k(s, a, h) + \xi_B,
	\end{align*}
	where in the last inequality we apply:
	\begin{align*}
		&2\eta\sum_{s, h\leq H}\optq(s, h)\sumk\suma\pi_k(a|s, h)\tilQ_k^2(s, a, h)\\ 
		&= 2\eta\sum_{s, h\leq H}\optq(s, h)\sumk\suma\pi_k(a|s, h)\frac{G_{k, s, a, h}^2}{(\ux_k(s, a, h)+\theta)^2}\\
		&\leq  L'\sum_{s, h\leq H}\optq(s, h)\sumk\suma\frac{\theta\pi_k(a|s, h)}{\ux_k(s, a, h)+\theta}\frac{m_k(s, a, h)}{\ux_k(s, a, h)+\theta} \tag{$2\eta=\theta/L'$ and $G_{k,s,a,h}\leq L'm_k(s,a,h)$}\\
		&\leq L'\sum_{s, h\leq H}\optq(s, h)\rbr{ 2\sumk\suma\frac{\theta\pi_k(a|s, h)}{\ux_k(s, a, h)+\theta} + \tilO{\frac{1}{\theta}} } \tag{\pref{lem:e2r} and $\frac{x_k(s,a,h)}{\ux_k(s,a,h)+\theta}\leq 1$}\\
		&\leq \sum_{s, h\leq H}\optq(s, h)\sumk\suma\frac{2\theta L'\pi_k(a|s, h)}{\ux_k(s, a, h)+\theta} + \tilO{\frac{\T L'}{\theta}}.
	\end{align*}
	Therefore, by \pref{lem:bias1}, \pref{lem:bias2}, and \pref{lem:dilated bonus}, with probability at least $1-24\delta$,
	\begin{align*}
		&\reg + \bias_1 + \bias_2 \leq \tilO{\frac{\T}{\eta}} + 3\sumk\inner{\hatq'_k}{b_k} + \tilO{\Tmax^2} \tag{$\theta=2\eta L'$ and $\norm{b_k}_{\infty}=\tilo{\Tmax}$ by \pref{lem:bound bandit}}\\
		&\leq \tilO{\frac{\T}{\eta} + \sumk\sum_{s, a, h\leq H}\hatq'_k(s, a, h)\frac{L'(\ux_k(s, a, h)-\lx_k(s, a, h)) + \theta L'}{\ux_k(s, a, h) + \theta} + \Tmax^2 }\\
		&\leq \tilO{\frac{\T}{\eta} + L'\sumk\sum_{s, a, h\leq H}(\ux_k(s, a, h) - \lx_k(s, a, h))\hatq'_{k, (s, a, h)}(s, a, h) + \theta \Tmax L'SAK + \Tmax^2 }. \tag{$\hatq'_k(s,a,h)\leq\ux_k(s,a,h)\hatq'_{k,(s,a,h)}(s, a, h)$ and $\hatq'_k(s, a, h)= \tilo{\Tmax\ux_k(s, a, h)}$}\\
		&= \tilO{\sqrt{S^2A\Tmax^5K} + S^{5.5}A^{3.5}\Tmax^5}. \tag{\pref{lem:stab} and definition of $\eta, \theta$}
	\end{align*}
	This completes the proof.
\end{proof}

\begin{lemma}
	\label{lem:stab}
	With probability at least $1-22\delta$,
	$$\sumk\sum_{s, a, h\leq H}(\ux_k(s, a, h) - \lx_k(s, a, h))\hatq'_{k, (s, a, h)}(s, a, h)=\tilo{\sqrt{S^2A\Tmax^3K} + S^{5.5}A^{3.5}\Tmax^4}.$$
\end{lemma}
\begin{proof}
	For any $z\in\SA\times[H]$, denote by $\uq^z_k$ / $\lq^z_k$ the occupancy measure w.r.t the policy and transition defined in $\ux_k(z)$ / $\lx_k(z)$ (transition at $(s, a, h)$ can be randomly pick as long as $P_{\uq^z_k}, P_{\lq^z_k}\in\calP_k$).
	For a fixed tuple $z=(s, a, h)$,
	\begin{align*}
		&\ux_k(s, a, h)\hatq'_{k, (s, a, h)}(s, a, h) = \uq^z_k(s, a, h) + \ux_k(s, a, h)(\hatq'_{k, (s, a, h)}(s, a, h) - \uq^z_{k,(s, a, h)}(s, a, h))\\
		&\leq \uq^z_k(s, a, h) + 2\ux_k(s, a, h)\sum_{s', a', h'}\uq^z_{k, (s, a, h)}(s', a', h')\sum_{s'', h''}\epsilon^{\star}_k(s', a', h', s'', h'')\hatq'_{k,(s'', h'')}(s, a, h) \tag{\pref{lem:value diff} and \pref{lem:conf bound}}\\
		&\leq \uq^z_k(s, a, h) + 2\sum_{s', a', h'}\uq^z_k(s', a', h')\sum_{s'', h''}\epsilon^{\star}_k(s', a', h', s'', h'')\hatq'_{k,(s'', h'')}(s, a, h) \tag{$\ux_k(s, a, h)\uq^z_{k,(s, a, h)}(s', a', h')\leq \uq^z_k(s', a', h')$}\\
		&= \uq^z_k(s, a, h) + 2\sum_{s', a', h'}q_k(s', a', h')\sum_{s'', h''}\epsilon^{\star}_k(s', a', h', s'', h'')\hatq'_{k,(s'', h'')}(s, a, h) \\
		&\qquad + 2\sum_{s', a', h'}(\uq^z_k(s', a', h')-q_k(s', a', h'))\sum_{s'', h''}\epsilon^{\star}_k(s', a', h', s'', h'')\hatq'_{k,(s'', h'')}(s, a, h).
	\end{align*}
	Therefore, with probability at least $1-7\delta$,
	\begin{align*}
		&\sumk\sum_{s, a, h\leq H}\ux_k(s, a, h)\hatq'_{k, (s, a, h)}(s, a, h) \overset{\text{(i)}}{\leq} \sumk\sum_{s, a, h\leq H}\uq^z_k(s, a, h)\\
		&\qquad + \tilO{\Tmax\sumk \sum_{s', a', h'}q_k(s', a', h')\sum_{s'', h''}\epsilon^{\star}_k(s', a', h', s'', h'') + S^{5.5}A^{3.5}\Tmax^4}\\
		&\overset{\text{(ii)}}{\leq} \sumk\sum_{s, a, h\leq H}\uq^z_k(s, a, h) + \tilO{\sqrt{S^2A\Tmax^3 K} + S^{5.5}A^{3.5}\Tmax^4},
	\end{align*}
	where in (i) we apply $\sum_{s,a,h\leq H}\hatq'_{k,(s'', h'')}(s, a, h)=\tilo{\Tmax}$ and ($z=(s, a, h)$ iterates over $\SA\times[H]$):
	\begin{align*}
		&\sum_{k, z}\sum_{s', a', h'}(\uq^z_k(s', a', h')-q_k(s', a', h'))\sum_{s'', h''}\epsilon^{\star}_k(s', a', h', s'', h'')\hatq'_{k,(s'', h'')}(z)\\
		&\leq \sum_{k, z}\sum_{\substack{\tils, \tila, \tilh\\ \tils', \tilh'}}\sum_{\substack{s', a', h'\\ s'', h''}}q_k(\tils, \tila, \tilh)\epsilon^{\star}_k(\tils,\tila,\tilh, \tils',\tilh')\uq^z_{k,(\tils',\tilh')}(s', a', h')\epsilon^{\star}_k(s', a', h', s'', h'')\hatq'_{k,(s'', h'')}(z) \tag{\pref{lem:value diff} and \pref{lem:conf bound}}\\
		&\leq \sum_{k, z}\sum_{\substack{s', a', h'\\s'', h''}}\sum_{\substack{\tils, \tila, \tilh\\\tils', \tilh'}}q_k(\tils, \tila, \tilh)\epsilon^{\star}_k(\tils,\tila,\tilh, \tils',\tilh')q_{k,(\tils',\tilh')}(s', a', h')\epsilon^{\star}_k(s', a', h', s'', h'')\hatq'_{k,(s'', h'')}(z)\\
		&\qquad + \tilO{S^{2.5}A^{1.5}\Tmax^3\sum_{z}\sum_{\substack{s', a', h'\\ s'', h''}}\Tmax }\tag{$\epsilon^{\star}_k(s', a', h', s'', h'')\hatq'_{k,(s'', h'')}(z)=\tilo{\Tmax}$ and \pref{lem:double trans diff}}\\ 
		&= \tilO{ S^{5.5}A^{3.5}\Tmax^4}.\\
	\end{align*}
	and in (ii) we apply:
	\begin{align*}
		&\Tmax\sumk \sum_{s', a', h'}q_k(s', a', h')\sum_{s'', h''}\epsilon^{\star}_k(s', a', h', s'', h'')\\ 
		&= \tilO{\Tmax\sumk \sum_{s', a', h'\leq H}q_k(s', a', h')\sum_{s'',h''}\rbr{\sqrt{\frac{P_{s', a', h'}(s'',h'')}{\Np_k(s', a')}} + \frac{1}{\Np_k(s', a')} } }\tag{definition of $\epsilon^{\star}_k$}\\
		&= \tilO{\Tmax\sqrt{S}\sumk \sum_{s', a'}\frac{q_k(s', a')}{\sqrt{\Np_k(s', a')}} + \Tmax S\sumk\sum_{s', a'}\frac{q_k(s', a')}{\Np_k(s', a')} }\\ 
		&= \tilO{\sqrt{S^2A\Tmax^3K} + S^2A\Tmax^2}. \tag{\pref{lem:n sum}}
	\end{align*}
	By similar arguments, we also have with probability at least $1-7\delta$,
	\begin{align*}
		&-\sumk\sum_{s, a, h\leq H}\lx_k(s, a, h)\hatq'_{k, (s, a, h)}(s, a, h)\\ 
		&\leq -\sumk\sum_{s, a, h\leq H}\lq^z_k(s, a, h) + \tilO{\sqrt{S^2A\Tmax^3 K} + S^{5.5}A^{3.5}\Tmax^4}.
	\end{align*}
	Therefore,	 with probability at least $1-7\delta$,
	\begin{align*}
		&\sumk\sum_{s, a, h\leq H}(\ux_k(s, a, h) - \lx_k(s, a, h))\hatq'_{k, (s, a, h)}(s, a, h)\\
		&= \tilO{\sumk\sum_{s, a, h\leq H}(\uq^{(s, a, h)}_k(s, a, h) - \lq^{(s, a, h)}_k(s, a, h)) + \sqrt{S^2A\Tmax^3K} + S^{5.5}A^{3.5}\Tmax^4}\\
		&= \tilO{\sqrt{S^2A\Tmax^3K} + S^{5.5}A^{3.5}\Tmax^4},
	\end{align*}
	where in the last inequality we apply (similarly for $\sumk (q_k(s, a, h)-\lq^{(s, a, h)}_k(s, a, h))$):
	\begin{align*}
		&\sumk\sum_{s, a, h\leq H}(\uq^{(s, a, h)}_k(s, a, h) - q_k(s, a, h))\\
		&\leq \sum_{k, s, a, h\leq H}\sum_{s', a', h'}q_k(s', a', h')\sum_{s'', h''}\epsilon^{\star}_k(s', a', h', s'', h'')\uq_{k, (s'', h'')}^{(s, a, h)}(s, a, h) \tag{\pref{lem:value diff} and \pref{lem:conf bound}}\\
		&\leq \sum_{k, s, a, h\leq H}\sum_{s', a', h'}q_k(s', a', h')\sum_{s'', h''}\epsilon^{\star}_k(s', a', h', s'', h'')q_{k, (s'',h'')}(s, a, h) + \tilO{\sum_{s, a, h}S^{2.5}A^{1.5}\Tmax^3} \tag{\pref{lem:double trans diff}}\\
		&= \tilO{ \Tmax\sum_{k, s', a', h'\leq H}q_k(s', a', h')\sum_{s'',h''}\rbr{\sqrt{\frac{P_{s',a',h'}(s'',h'')}{\Np_k(s', a')}} + \frac{1}{\Np_k(s', a')}} + S^{3.5}A^{2.5}\Tmax^3 } \tag{definition of $\epsilon_k^{\star}$} \\
		&= \tilO{\sqrt{S^2A\Tmax^3K} + S^{3.5}A^{2.5}\Tmax^3}. \tag{Cauchy-Schwarz inequality and \pref{lem:n sum}}
	\end{align*}
	This completes the proof.
\end{proof}

\begin{lemma}
	\label{lem:bias1}
	With probability at least $1-\delta$,
	\begin{align*}
		\bias_1 \leq \sumk\sum_{s, h\leq H}\optq(s, h)\suma\pi_k(a|s, h)\frac{L'(\ux_k(s, a, h)- \lx_k(s, a, h)) + 2\theta L'}{\ux_k(s, a, h)+\theta} + \tilO{\frac{\T L'}{\theta}}.
	\end{align*}
\end{lemma}
\begin{proof}
	Note that:
	\begin{align*}
		\bias_1 &= \sumk\sum_{s, h\leq H}\optq(s, h)\suma\pi_k(a|s, h)(Q_k(s, a, h)-\E_k[\tilQ_k(s, a, h)])\\
		&\qquad + \sumk\sum_{s, h\leq H}\optq(s, h)\suma\pi_k(a|s, h)(\E_k[\tilQ_k(s, a, h)]-\tilQ_k(s, a, h)).
	\end{align*}
	For the first term,
	\begin{align*}
		&\sumk\sum_{s, h\leq H}\optq(s, h)\suma\pi_k(a|s, h)(Q_k(s, a, h)-\E_k[\tilQ_k(s, a, h)])\\
		&\leq \sumk\sum_{s, h\leq H}\optq(s, h)\suma\pi_k(a|s, h)\barQ_k(s, a, h)\rbr{1-\frac{x_k(s, a, h)}{\ux_k(s, a, h)+\theta}} + \tilO{\T} \tag{\pref{lem:Q-barQ} and $\E_k[G_{k,s,a,h}]=x_k(s,a,h)\bar{Q}_k(s,a,h)$}\\
		&\leq \sumk\sum_{s, h\leq H}\optq(s, h)\suma\pi_k(a|s, h)\frac{L'(\ux_k(s, a, h)- \lx_k(s, a, h) + \theta)}{\ux_k(s, a, h)+\theta} + \tilO{\T}.
	\end{align*}
	For the second term, first note that $G_{k, s, a, h}\leq L'm_k(s, a, h)$ and
	\begin{align*}
		&\V_k\sbr{\inner{\pi_k(\cdot|s, h)}{\tilQ_k(s, \cdot, h)}} \leq \E_k\sbr{\inner{\pi_k(\cdot|s, h)}{\tilQ_k(s, \cdot, h)}^2}\\
		&\leq \suma\pi_k(a|s, h)\frac{\E_k[G_{k, s, a, h}^2]}{(\ux_k(s, a, h) + \theta)^2} \leq \suma\pi_k(a|s, h)\frac{{L'}^2}{\ux_k(s, a, h) + \theta}. \tag{Cauchy-Schwarz inequality and $\E_k[G_{k,s,a,h}]\leq L'x_k(s, a, h)$}
	\end{align*}
	Therefore, by \pref{lem:freedman}, with probability at least $1-\delta$,
	\begin{align*}
		&\sumk\sum_{s, h\leq H}\optq(s, h)\suma\pi_k(a|s, h)(\E_k[\tilQ_k(s, a, h)]- \tilQ_k(s, a, h))\\
		&= \tilO{\sum_{s, h\leq H}\optq(s, h)\rbr{\sqrt{ \sumk\suma\pi_k(a|s, h)\frac{{L'}^2}{\ux_k(s, a, h)+\theta} } + \frac{L'}{\theta} }}\\
		&\leq \sum_{s, h\leq H}\optq(s, h)\sumk\suma\pi_k(a|s, h)\frac{\theta L'}{\ux_k(s, a, h)+\theta} + \tilO{\frac{\T L'}{\theta}}. \tag{AM-GM inequality}
	\end{align*}
	Summing these two terms, we have:
	\begin{align*}
		\bias_1 \leq \sumk\sum_{s, h\leq H}\optq(s, h)\suma\pi_k(a|s, h)\frac{L'(\ux_k(s, a, h)- \lx_k(s, a, h)) + 2\theta L'}{\ux_k(s, a, h)+\theta} + \tilO{\frac{\T L'}{\theta}}.
	\end{align*}
	This completes the proof.
\end{proof}

\begin{lemma}
	\label{lem:bias2}
	With probability at least $1-\delta$, $\bias_2=\tilO{\T L'/\theta}$.
\end{lemma}
\begin{proof}
	By \pref{lem:GIX} with $Z_k(s, a, h)=G_{k, s, a, h}/L'$ and $\E_k[G_{k,s,a,h}]=x_k(s,a,h)\bar{Q}_k(s,a,h)$, we have with probability at least $1-\delta$:
	\begin{align*}
		\sumk\tilQ_k(s, a, h) - \barQ_k(s, a, h) = \tilO{L'/\theta},
	\end{align*}
	for any $(s,a)\in\SA,h\leq H$.
	Therefore,
	\begin{align*}
		\bias_2=\sumk\sum_{s, h\leq H}\optq(s, h)\suma\optpi(a|s, h)(\tilQ_k(s, a, h)-Q_k(s, a, h)) = \tilO{\T L'/\theta}.
	\end{align*}
\end{proof}

\begin{lemma}
	\label{lem:GIX}
	For any random variable $Z_k(s, a, h)$ depending on interaction before episode $k$ such that $Z_k(s, a, h)\in [0, 1]$, $\E_k[Z_k(s, a, h)]=z_k(s, a, h)\leq x_k(s, a, h)$, we have with probability at least $1-\delta$:
	\begin{align*}
		\sumk \rbr{ \frac{Z_k(s, a, h)}{\ux_k(s, a, h)+\theta} - \frac{z_k(s, a, h)}{\ux_k(s, a, h)} } \leq \frac{\ln\frac{1}{\delta}}{2\theta}.
	\end{align*}
\end{lemma}
\begin{proof}
	The statement is clearly true when $x_k(s,a,h)=0$.
	When $x_k(s,a,h)>0$, we also have $\ux_k(s,a,h)>0$.
	By $\frac{z}{1+z/2}\leq\ln(1+z)$ for $z\geq 0$, we have:
	\begin{align*}
		\frac{2\theta Z_k(s, a, h)}{\ux_k(s, a, h) + \theta} &\leq \frac{2\theta Z_k(s, a, h)}{\ux_k(s, a, h) + \theta Z_k(s, a, h)} = \frac{2\theta Z_k(s, a, h)/\ux_k(s, a, h)}{1 + \theta Z_k(s, a, h)/\ux_k(s, a, h)}\\
		&\leq \ln\rbr{1 + 2\theta Z_k(s, a, h)/\ux_k(s, a, h)}.
	\end{align*}
	This gives
	\begin{align*}
		\E_k\sbr{\exp\rbr{ \frac{2\theta Z_k(s, a, h)}{\ux_k(s, a, h)+\theta} }} &\leq \E_k\sbr{1 + \frac{2\theta Z_k(s, a, h)}{\ux_k(s, a, h)}  } = 1 + \frac{2\theta z_k(s, a, h)}{\ux_k(s, a, h)}\\
		&\leq \exp(2\theta z_k(s, a, h)/\ux_k(s, a, h)). \tag{$1+z\leq e^z$}
	\end{align*}
	Therefore, by Markov inequality,
	\begin{align*}
		&P\rbr{\sumk \frac{2\theta Z_k(s, a, h)}{\ux_k(s, a, h)+\theta} - \frac{2\theta z_k(s, a, h)}{\ux_k(s, a, h)} > \ln\frac{1}{\delta} }\\
		&\leq \delta\cdot \E\sbr{ \exp\rbr{ \sumk \frac{2\theta Z_k(s, a, h)}{\ux_k(s, a, h)+\theta} - \frac{2\theta z_k(s, a, h)}{\ux_k(s, a, h)} } } \leq \delta.
	\end{align*}
	Thus, with probability at least $1-\delta$, $\sumk \frac{Z_k(s, a, h)}{\ux_k(s, a, h)+\theta} - \frac{z_k(s, a, h)}{\ux_k(s, a, h)} \leq \frac{\ln\frac{1}{\delta}}{2\theta}$.
\end{proof}

\subsection{Dilated Bonus in SDA}
Below we present lemmas related to dilated bonus in $\rcalM$.
We first show that a form of dilated value function is well-defined.

\begin{lemma}
	\label{lem:bound B}
	For some policy $\pi$ in $\rcalM$, transition $P\in\Lambda_{\calM}$, and bonus function $b: \SA\times[H]\rightarrow[0, \rho]$ for some $\rho>0$, define $B(s, a, h) = b(s, a, h) + \rbr{1+\frac{1}{H'}}P_{s, a, h}B$, $B(s, h)=\sum_a\pi(a|s,h)B(s, a, h)$ and $B(g)=B(s, a, H+1)=0$.
	Then, $\max_{s, a}B(s, a, h)\leq \frac{15 \rho(H-h+1)}{1-\gamma}$.
\end{lemma}
\begin{proof}
	Define $\gamma'=(1+\frac{1}{H'})\gamma$ and recall that $H'=\frac{8(H+1)\ln(2K)}{1-\gamma}$.
	Now note that $\frac{1}{1-\gamma'}\leq\frac{1+\frac{1}{H}}{1-\gamma}$ by simple algebra.
	Finally, define
	$\bar{b}(s, a, h)=\rbr{1+\frac{1}{H'}}\inner{P_{s, a, h}(\cdot, h+1)}{B(\cdot, h+1)}$ for $h \le H$,
	and $P'_{s,a,h}(s')=(1+\frac{1}{H'})P_{s, a, h}(s', h)$.
	
	We prove that $B$ is well defined and the statement holds by induction on $h=H+1,\ldots,1$.
	The base case is true by definition $B(s, a, H+1)=0$.
	For $h\leq H$ we have:
	\begin{align*}
		B(s, a, h) &= b(s, a, h) + \rbr{1+\frac{1}{H'}}\rbr{ \inner{P_{s, a, h}(\cdot, h)}{B(\cdot, h)} + \inner{P_{s, a, h}(\cdot, h+1)}{B(\cdot, h+1)} }\\
		&= b(s, a, h) + \bar{b}(s, a, h) + P'_{s, a, h}B(\cdot, h).
	\end{align*}
	Therefore, $B(\cdot, \cdot, h)$ can be treated as the action-value function in an SSP with cost $(b+\bar{b})(\cdot,\cdot,h)$ and transition function $P'$ (thus well defined).
	By $\sum_{s'}P'_{s,a,h}(s',h)\leq\gamma'$, we have the expected hitting time of any policy starting from any state in an SSP with transition $P'$ is upper bounded by $\frac{1}{1-\gamma'}\leq\frac{1+\frac{1}{H}}{1-\gamma}$.
	Let $R(h)=\max_{s, a}B(s, a, h)$ and note that $R(H+1)=0$.
	Since $b(s,a,h)\leq\rho$ and $\bar{b}(s, a, h) \le \rbr{1+\frac{1}{H'}}(1-\gamma) R(h+1)$ by $\sum_{s'}P_{s,a,h}(s',h+1)\leq 1-\gamma$, we have:
	\begin{align*}
	    R(h) & \leq \frac{\rho + \rbr{1+\frac{1}{H'}}(1-\gamma)R(h+1) }{1-\gamma'} \leq \frac{\rho}{1-\gamma'} + \rbr{1+\frac{1}{H'}}\rbr{1+\frac{1}{H}}R(h + 1)
	    \\
	    & \le \frac{\rho \rbr{1+\frac{1}{H}}}{1 - \gamma} + \rbr{1+\frac{1}{H'}}\rbr{1+\frac{1}{H}}R(h + 1),
	\end{align*}
	where the two last inequalities follow because $\frac{1}{1-\gamma'}\leq\frac{1+\frac{1}{H}}{1-\gamma}$.
	The proof is now finished by solving the recursion and obtaining:
	\[
	    R(h)
	    \le
	    \frac{\rho}{1 - \gamma} \sum_{i=0}^{H-h} \rbr{1+\frac{1}{H'}}^i \rbr{1+\frac{1}{H}}^{i+1},
	\]
	which implies that $R(h) \le \frac{15 \rho (H-h+1)}{1 - \gamma}$ since $(1+\frac{1}{H})^{H+1} (1+\frac{1}{H'})^H \le 2e^2 \le 15$.
\end{proof}

\begin{lemma}
	\label{lem:dilated bonus}
	Let $\pi$ be a policy in $\rcalM$ and $b$ be a non-negative cost function in $\rcalM$ such that $b(s, a, H+1)=0$ and $b(s,a,h) \leq \rho$.
	Moreover, let $\hatP\in\Lambda_{\calM}$ be an optimistic transition so that
	\begin{align*}
		B(s, a, h) &= b(s, a, h) + \rbr{1+\frac{1}{H'}}\hatP_{s, a, h}B \geq b(s, a, h) + \rbr{1+\frac{1}{H'}}P_{s, a, h}B,
	\end{align*}
	where $B(s, h)=\suma\pi(a|s, h)B(s, a, h)$ and $B(g)=B(s, H+1)=0$.
	Then,
	\begin{align*}
		&\sum_{s, h}\optq(s, h)\suma(\pi(a|s, h) - \roptpi(a|s, h))B(s, a, h) + \frac{1}{H'}\sum_{s, h}\optq(s, h)B(s, h)\\ 
		&\qquad + \sum_{s, a, h}\optq(s, a, h)b(s, a, h) \leq 3V^{\pi, \hatP, b}(\sinit, 1) + \tilO{\frac{H\rho}{K(1-\gamma)}}.
	\end{align*}
\end{lemma}
\begin{proof}
	By the optimism property of $\hatP$, we have:
	\begin{align}
		&\sum_{s, h}\optq(s, h)\suma\rbr{\pi(a| s, h)-\roptpi(a|s, h)}B(s, a, h) \nonumber\\
		&\qquad + \frac{1}{H'}\sum_{s, h}\optq(s, h)\suma\pi(a| s, h)B(s, a, h) + \sum_{s, a, h}\optq(s, a, h)b(s, a, h) \nonumber\\
		&\leq \rbr{1 + \frac{1}{H'}}\sum_{s, h}\optq(s, h)\suma\pi(a| s, h)B(s, a, h) + \sum_{s, a, h}\optq(s, a, h)b(s, a, h) \nonumber\\
		&\qquad - \sum_{s, a, h}\optq(s, a, h)\rbr{b(s, a, h) + \rbr{1+\frac{1}{H'}}\sum_{s', h'}P_{s,a,h}(s', h')B(s', h') } \nonumber\\
		&= \rbr{1 + \frac{1}{H'}}\sum_{s', h'}\rbr{\optq(s', h') - \sum_{s, a, h}\optq(s, a, h)P_{s,a,h}(s', h')}B(s', h') \nonumber\\
		&= \rbr{1 + \frac{1}{H'}}B(\sinit, 1). \label{eq:boung-with-B}
	\end{align}
	The last relation is by $\optq(s, h) - \sum_{s', a', h'}\optq(s', a', h')P_{s',a',h'}(s, h) = \Ind\{(s, h)=(\sinit, 1)\}$ (see \citep[Appendix B.1]{rosenberg2020adversarial}).
	
	Let $J$ be the number of steps until the goal state $g$ is reached in $\rcalM$, and $n=\frac{8H}{1-\gamma}\ln(2K)$.
	Now note that for any policy, the expected hitting time in an SSP with transition $\hatP$ is upper bounded by $\frac{H}{1-\gamma}+1$ by $\hatP\in\Lambda_{\calM}$.
	Therefore, by \pref{lem:hitting}, $P(J \geq n)\leq\frac{1}{K}$, and
	\begin{align*}
		B(s, h) &= \E\sbr{\left.\sum_{t=1}^J\rbr{1+\frac{1}{H'}}^{t-1}b(s_t, a_t, h_t)\right| \pi, \hatP, (s_1, h_1)=(s, h)}\\ 
		&= \E\sbr{ \left.\sum_{t=1}^n\rbr{1+\frac{1}{H'}}^{t-1}b(s_t, a_t, h_t) + \rbr{1 + \frac{1}{H'}}^nB(s_{t+1}, h_{t+1}) \right| \pi, \hatP, (s_1, h_1)=(s, h)}\\
		&\leq \rbr{1 + \frac{1}{H'}}^{n-1}V^{\pi, \hatP, b}(s, h) + \tilO{\frac{H\rho}{K(1-\gamma)}}. \tag{\pref{lem:bound B}}
	\end{align*}
	Plugging this back into \pref{eq:boung-with-B} and by $(1+1/H')^n\leq e < 3$, we get the desired result.
\end{proof}

\subsection{Computations of $B_k$}
\label{app:compute B}
We study an operator on value function, from which $B_k$ can be computed as a fixed point.
For any policy $\pi$, cost function $c$, transition confidence set $\calP\subseteq\Lambda_{\calM}$, and interest factor $\rho\geq 0$, we define the dilated Bellman operator $\calT_{\rho}$ that maps any value function $V:\rcalS_+\rightarrow\fR_+$ to another value function $\calT_{\rho} V: \rcalS_+\rightarrow\fR_+$, such that:
\begin{align}
	&(\calT_{\rho} V)(s, h) = \sum_a\pi(a|s, h)\rbr{ c(s, a, h) +  (1+\rho)\max_{P\in\calP}P_{s,a,h}V },\notag\\
	&(\calT_{\rho} V)(g) = 0,\; (\calT_{\rho} V)(s, H+1)=\max_ac(s,a,H+1).\label{eq:EVI}
\end{align}
In this work, we have $\calP\in\{\calP_k\}_{k=1}^K$, and $\calP_k=\bigcap_{s, a, h}\calP_{k,s,a,h}$, where $\calP_{k,s,a,h}$ is a convex set that specifies constraints on $((s, h), a)$.
In other words, $\calP_k$ is a product of constraints on each $((s, h), a)$ (note that $\Lambda_{\calM}$ can also be decomposed into shared constraints on $P_{s,a,H+1}$ and independent constraints on each $s, a, h\leq H$).
Thus, there exists $P'\in\calP$ that satisfies $P'=\argmax_{P\in\calP}P_{s,a,h}V$ in \pref{eq:EVI} for all $((s, h), a)$ simultaneously.
Moreover, finding such $P'$ can be done by linear programming for each $((s, h), a)$ independently.
Now we show that iteratively applying $\calT_{\rho}$ to some initial value function converges to a fixed point sufficiently fast.
\begin{lemma}
	\label{lem:EVI}
	Define value function $V^0: \rcalS_+\rightarrow\fR_+$ such that $V^0(s, h)=V^0(g)=0$ for any $(s,h)\in\calS\times[H]$ and $V^0(s, H+1)=\max_ac(s,a,H+1)$.
	Then for any $\rho\geq 0$ such that $\gamma'=(1+\rho)\gamma<1$, the limit $V_{\rho}=\lim_{n\rightarrow\infty}\calT^n_{\rho} V^0$ exists. 
	Moreover, when $n\geq Hl$ with $l=\ceil{\frac{\ln\frac{1}{\epsilon}}{1-\gamma'}}$ for some $\epsilon>0$, we have $\norm{\calT^n_{\rho}V^0-V_{\rho}}_{\infty}\leq H\rbr{\frac{(1+\rho)(1-\gamma)}{1-\gamma'}}^{H-1}\kappa\epsilon$, where $\kappa=\sum_{j=0}^{H-1}(\frac{(1+\rho)(1-\gamma)}{1-\gamma'})^j\frac{\norm{c}_{\infty}}{1-\gamma'}$.
\end{lemma}
\begin{proof}
	Define a sequence of value functions $\{V^i\}_{i=0}^{\infty}$ such that $V^{i+1}=\calT_{\rho} V^i$.
	We first show that $\norm{V^i(\cdot, h)}_{\infty}\leq \sum_{j=0}^{H-h}(\frac{(1+\rho)(1-\gamma)}{1-\gamma'})^j\frac{\norm{c}_{\infty}}{1-\gamma'}$ for $i\geq 0$ and $h\leq H$.
	We prove this by induction on $i$.
	Note that this is clearly true when $i=0$.
	For $i>0$, by $\calP\subseteq\Lambda_{\calM}$ and \pref{eq:EVI}, we have:
	\begin{align*}
		V^i(s, h) &= (\calT_{\rho}V^{i-1})(s, h) \leq \norm{c}_{\infty} + \gamma'\norm{V^{i-1}(\cdot,h)}_{\infty} + (1+\rho)(1-\gamma)\norm{V^{i-1}(\cdot,h+1)}_{\infty}\\
		&\leq \norm{c}_{\infty} + \gamma'\sum_{j=0}^{H-h}\rbr{\frac{(1+\rho)(1-\gamma)}{1-\gamma'}}^j\frac{\norm{c}_{\infty}}{1-\gamma'} + \sum_{j=1}^{H-h}\rbr{\frac{(1+\rho)(1-\gamma)}{1-\gamma'}}^j\norm{c}_{\infty}\\
		&\leq \sum_{j=0}^{H-h}\rbr{\frac{(1+\rho)(1-\gamma)}{1-\gamma'}}^j\frac{\norm{c}_{\infty}}{1-\gamma'}.
	\end{align*}
	Therefore, $\norm{V^i}_{\infty}\leq \kappa$.
	We now show that $\{V^i\}_i$ converges to a fixed point.
	Specifically, we show that for some $\epsilon>0$ and any $i, j\in\fN$, when $n\geq (H-h+1)l$, we have $\norm{(\calT^n_{\rho}V^i)(\cdot, h) - (\calT^n_{\rho}V^j)(\cdot, h)}_{\infty}\leq (H-h+1)(\frac{(1+\rho)(1-\gamma)}{1-\gamma'})^{H-h}\kappa\epsilon$ (note that $\frac{(1+\rho)(1-\gamma)}{1-\gamma'}>1$).
	Therefore, when $n\geq Hl$, we have $\norm{\calT_{\rho}^nV^i - \calT_{\rho}^nV^j}_{\infty}\leq H(\frac{(1+\rho)(1-\gamma)}{1-\gamma'})^{H-1}\kappa\epsilon$.
	Setting $\epsilon\rightarrow 0$, the statement above implies that for any $\rs\in\rcalS$, $\{V^i(\rs)\}_{i=1}^{\infty}$ is a Cauchy sequence and thus converges.
	Moreover, letting $j\rightarrow\infty$ implies that $\{V^i\}_i$ converges to $V_{\rho}$ with the rate shown above.
	We prove the statement above by induction on $h=H,\ldots,1$.
	First note that for any $s\in\calS, h\in[H]$:
	\begin{align}
		&\abr{(\calT_{\rho} V^i)(s, h) - (\calT_{\rho} V^j)(s, h)} = (1+\rho)\abr{ \sum_a\pi(a|s, h)\rbr{\max_{P\in\calP}P_{s,a,h}V^i - \max_{P\in\calP}P_{s,a,h}V^j} }\notag\\
		&\leq (1+\rho)\sum_a\pi(a|s,h)\max_{P\in\calP}\abr{P_{s,a,h}(V^i-V^j)}\notag\\
		&\leq \gamma'\norm{V^i(\cdot, h)-V^j(\cdot, h)}_{\infty} + (1+\rho)(1-\gamma)\norm{V^i(\cdot, h+1)-V^j(\cdot, h+1)}_{\infty},\label{eq:iter}
	\end{align}
	where the last inequality is by $\sum_{s'}P_{s,a,h}(s',h)\leq\gamma$, $\sum_{s'}P_{s,a,h}(s',h+1)\leq 1-\gamma$, and $P_{s,a,h}(s',h')=0$ for $h'\notin\{h,h+1\}$, for any $P\in\Lambda_{\calM}$.
	Now for the base case $h=H$, \pref{eq:iter} implies $\norm{(\calT_{\rho}V^i)(\cdot, H) - (\calT_{\rho}V^j)(\cdot, H)}_{\infty}\leq\gamma'\norm{V^i(\cdot, H) - V^j(\cdot, H)}_{\infty}$. 
	Thus for $n\geq l$, $\norm{(\calT^n_{\rho}V^i)(\cdot,H)-(\calT^n_{\rho}V^j)(\cdot,H)}_{\infty}\leq {\gamma'}^n\cdot\kappa\leq\kappa\epsilon$.
	For the induction step $h<H$, if $n\geq (H-h+1)l$, then \pref{eq:iter} implies:
	\begin{align*}
		&\abr{(\calT^n_{\rho}V^i)(s, h) - (\calT^n_{\rho}V^j)(s, h)}\\ 
		&\leq {\gamma'}^l\norm{(\calT^{n-l}_{\rho}V^i)(s, h) - (\calT^{n-l}_{\rho}V^j)(s, h)}_{\infty} + (1-\gamma')\rbr{\frac{(1+\rho)(1-\gamma)}{1-\gamma'}}^{H-h}\sum_{i=0}^{l-1}{\gamma'}^i(H-h)\kappa\epsilon\tag{by the induction assumption}\\
		&\leq (H-h+1)\rbr{\frac{(1+\rho)(1-\gamma)}{1-\gamma'}}^{H-h}\kappa\epsilon.
	\end{align*}
	This completes the proof of the statement above.
\end{proof}
Now note that $B_k$ is a fixed point of $\calT_{\rho}$ with $\pi=\pi_k$, $\calP=\calP_k$, $c=b_k$, and $\rho=1/H'$.
Thus, $B_k$ can be approximated efficiently.

\subsection{Computation of $\ux_k$ and $\lx_k$}
\label{app:compute ux}
Note that $\ux_k(s, a, h)$ can be computed by solving the following linear program (it is straightforward to verify that the constraints on $\pi_q$ and $P_q$ are linear):
\begin{align*}
	&\max_{q\in\fR_{\geq 0}^{\SA\times[H]\times\rcalS_+}} \sum_{\rs'\in\rcalS_+}q(s, a, h, \rs')\\
	\text{s.t. }& \sum_{a'\in\calA, \rs'\in\calS_+}q(s', a', h', \rs')\\ 
	&\qquad - \sum_{(s'', h'')\in\rcalS, a''\in\calA}q(s'', a'', h'', (s', h')) = \Ind\{(s', h')=(\sinit, 1)\},\forall (s', h')\\
	&\pi_q = \pi_k,\quad P_q \in \bigcap_{(s',a',h')\in(\SA\times[H])\setminus \{(s,a,h)\}}\calP_{k,s',a',h'},\quad P_{q,s,a,h}(g)=1
\end{align*}
That is, we try to compute the occupancy measure that maximizes the number of visits to $(s, a, h)$ in an augmented MDP, where the transition lies in $\calP_k$ except that taking action $a$ at state $(s, h)$ directly transits to the goal state (so that the number of visits to $(s, a, h)$ is at most $1$ and the occupancy measure at $(s, a, h)$ is the probability of visiting $(s, a, h)$).
The computation of $\lx_k(s,a,h)$ is similar.
Thus, both $\ux_k$ and $\lx_k$ can be computed efficiently (in a weakly polynomial time).

%% file: app-pf.tex

In this section, we discuss the achievable regret guarantee without knowing some of the parameters assumed to be known.
For simplicity, we only describe the high level ideas.
We first describe the general ideas of dealing with each parameter being unknown, which are applicable under all types of feedback.
\begin{itemize}
	\item \textbf{Unknown $D$ and unknown fast policy}: we can simply follow the ideas in \citep{chen2021finding} to estimate $D$ and fast policy.
	For unknown fast policy, we maintain an instance of Bernstein-SSP~\citep{cohen2020near} $\calB_f$.
	When we need to switch to the fast policy, we simply involve $\calB_f$ as if this is a new episode for this algorithm, follow its decision until reaching $g$, and always feed cost $1$ for all state-action pairs.
	Following the arguments in \citep[Lemma 1]{chen2021finding}, the scheme above only incurs constant extra regret.
	For unknown $D$, we maintain an estimate of it and update the algorithm's parameters whenever the estimate is updated.
	Specifically, we separate the state space into known states and unknown states.
	A state is known if the number of visits to it is more than some threshold, and it is unknown otherwise.
	Whenever the learner visits an unknown state, it involves a Bernstein-SSP instance to approximate the behavior of fast policy until reaching $g$.
	When an unknown state $s$ becomes known, we update the diameter estimate by incorporating an estimate of $T^{\pi^f}(s)$, and then updates the algorithm's parameters with respect to the new estimate.
	In terms of regret, this approach does not affect the transition estimation error, but brings an extra $\sqrt{S}$ factor in the regret from policy optimization due to at most $S$ updates to the algorithm's parameters.
	\item \textbf{Unknown $\B$}: We can estimate $\B$ following the procedure in \citep[Appendix C]{cohen2021minimax}.
	The main idea is pretty similar to the unknown $D$ case: we again maintain an estimate of $\B$ and separate states into known states and unknown states based on how many times a state has been visited.
	The learner updates algorithm's parameters whenever the estimate of $\B$ is updated.
	Similarly, this approach brings an extra $\sqrt{S}$ factor in the regret from policy optimization.
	\item \textbf{Unknown $\T$}: We can replace $\T$ in parameters by $\B/\cmin$ in stochastic costs setting and $D/\cmin$ in other settings since $\T\leq \B/\cmin$ (or $\T\leq D/\cmin$).
	How to estimate $D$ or $\B$ is discussed above.
	\item \textbf{Unknown $\Tmax$}: Similar to \citep{chen2021finding}, we simply replace $\Tmax$ in parameters by $K^p$ for some $p\in(0, \frac{1}{2})$.
\end{itemize}
Next, we describe under each setting, what regret guarantee we can achieve with each parameter being unknown by applying the corresponding method above.

\paragraph{Stochastic Costs} In this setting, we need the knowledge of $D$, $\B$ and $\Tmax$.
\begin{itemize}
	\item \textbf{Unknown $D$}: Since the regret from policy optimization is a lower order term, the dominating term of the final regret remains to be $\tilo{\B S\sqrt{AK}}$.
	\item \textbf{Unknown $\B$}: Since the regret from policy optimization is a lower order term, the dominating term of the final regret remains to be $\tilo{\B S\sqrt{AK}}$.
	\item \textbf{Unknown $\Tmax$}: We replace $\Tmax$ in parameters by $K^{1/12}$.
	If $K^{1/12}\leq \Tmax$, then clearly the regret is of order $\tilo{LK}=\tilo{\Tmax^{13}}$.
	Otherwise, by \pref{thm:SF-SC} we have $R_K=\tilo{\B S\sqrt{AK} + S^4A^{2.5}K^{1/3}}$.
\end{itemize}

\paragraph{Stochastic Adversary} In this setting, we need the knowledge of $D$, $\T$, and $\Tmax$.
We consider the following cases:
\begin{itemize}
	\item \textbf{Unknown $D$}: Since the regret from policy optimization is a lower order term, the dominating term of the final regret remains to be $\tilo{\sqrt{D\T K} + DS\sqrt{AK}}$ in the full information setting, and $\tilo{\sqrt{D\T SAK} + DS\sqrt{AK}}$ in the bandit feedback setting.
	\item \textbf{Unknown $\T$}: Ignoring the lower order terms, we have $R_K=\tilo{D\sqrt{K/\cmin} + DS\sqrt{AK}}$ in the full information setting by \pref{thm:SA-F}, and $R_K=\tilo{D\sqrt{SAK/\cmin} + DS\sqrt{AK}}$ in the bandit feedback setting by \pref{thm:SA-B}.
	\item \textbf{Unknown $\Tmax$}: We replace $\Tmax$ in parameters by $K^{1/13}$.
	If $K^{1/13}\leq \Tmax$, then $R_K$ is of order $\tilo{LK}=\tilo{\Tmax^{14}}$.
	Otherwise, we have $R_K = \tilo{\sqrt{D\T K} + DS\sqrt{AK} + (S^2A^3)^{1/4}K^{25/52} + S^4A^{2.5}K^{4/13} }$ in the full information setting by \pref{thm:SA-F}, and $R_K = \tilo{\sqrt{SAD\T K} + DS\sqrt{AK} + SA^{5/4}K^{25/52} + S^4A^{2.5}K^{4/13} }$ in the bandit feedback setting by \pref{thm:SA-B}.
\end{itemize}

\paragraph{Adversarial Costs, Full Information} In this setting, we need the knowledge of $D$, $\T$, and $\Tmax$.
We consider the following cases:
\begin{itemize}
	\item \textbf{Unknown $D$}: With an extra $\sqrt{S}$ factor in the policy optimization term, we have $R_K = \tilo{\T\sqrt{SDK} + \sqrt{S^2AD\T K}}$ ignoring the lower order terms.
	\item \textbf{Unknown $\T$}: Ignoring the lower order terms, we have $R_K = \tilo{\frac{D^{1.5}}{\cmin}\sqrt{K} + D\sqrt{S^2AK/\cmin}}$.
	\item \textbf{Unknown $\Tmax$}: We replace $\Tmax$ in parameters by $K^{1/11}$.
	If $K^{1/11}\leq \Tmax$, then clearly the regret is of order $\tilo{LK}=\tilo{\Tmax^{12}}$.
	Otherwise, by \pref{thm:full} we have $R_K = \tilo{\T\sqrt{DK} + \sqrt{S^2AD\T K} + S^4A^2K^{5/11}}$.
\end{itemize}

\paragraph{Adversarial Costs, Bandit Feedback} In this setting, we need the knowledge of $D$ and $\Tmax$. 
We consider the following cases:
\begin{itemize}
	\item \textbf{Unknown $D$}: Tracing the proof of \pref{thm:bandit}, the regret from policy optimization is of order $\tilo{\sqrt{SA\Tmax^4K}}$.
	With an extra $\sqrt{S}$ factor in the policy optimization term, we still have $R_K = \tilo{ \sqrt{S^2A\Tmax^5K} }$ ignoring the lower order terms.
	\item \textbf{Unknown $\Tmax$}: We replace $\Tmax$ in parameters by $K^p$ for any $p\in(0, \frac{1}{5})$.
	If $K^p\leq \Tmax$, then clearly the regret is of order $\tilo{LK}=\tilo{\Tmax^{1+1/p}}$.
	Otherwise, by \pref{thm:bandit} we have $R_K = \tilo{\sqrt{S^2AK^{1+5p}} + S^{5.5}A^{3.5}K^{5p}}$.
\end{itemize}

%% file: auxlm.tex

\begin{lemma}
	\label{lem:quad}
	If $x\leq (a\sqrt{x}+b)\ln^p(cx)$ for some $a, b, c>0$ and absolute constant $p\geq 0$, then $x = \tilo{a^2 + b}$.
	Specifically, $x \leq a\sqrt{x} + b$ implies $x\leq (a+\sqrt{b})^2\leq 2a^2+2b$.
\end{lemma}

\begin{lemma}\citep[Lemma A.4]{luo2021policy}
	\label{lem:expw}
	Let $\eta>0$, $\pi_k\in\Delta(A)$, and $\ell_k\in\fR^A$ satisfy the following for all $k\in[K]$ and $a\in\calA$:
	\begin{align*}
		\pi_1(a) = \frac{1}{A}, \quad \pi_{k+1}(a)\propto \pi_k(a)\exp(-\eta\ell_k(a)), \quad \abr{\eta\ell_k(a)}\leq 1.
	\end{align*}
	Then for any $\optpi\in\Delta(A)$, $\sumk\inner{\pi_k-\optpi}{\ell_k} \leq \frac{\ln A}{\eta} + \eta\sumk\suma\pi_k(a)\ell_k^2(a)$.
\end{lemma}

\begin{lemma}\citep[Lemma 38]{chen2021improved}
	\label{lem:freedman}
	Let $\{X_i\}_{i=1}^{\infty}$ be a martingale difference sequence adapted to the filtration $\{\calF_i\}_{i=0}^{\infty}$ and $|X_i|\leq B$ for some $B>0$.
	Then with probability at least $1-\delta$, for all $n\geq 1$ simultaneously,
	\begin{align*}
		\abr{\sum_{i=1}^nX_i}\leq 3\sqrt{\sum_{i=1}^n\E[X_i^2|\calF_{i-1}]\ln\frac{4B^2n^3}{\delta} } + 2B\ln\frac{4B^2n^3}{\delta}.
	\end{align*}
\end{lemma}

\begin{lemma}\citep[Theorem D.3]{cohen2020near}
	\label{lem:bernstein}
	Let $\{X_n\}_{n=1}^{\infty}$ be a sequence of i.i.d random variables with expectation $\mu$ and $X_n\in[0, B]$ almost surely.
	Then with probability at least $1-\delta$, for any $n\geq 1$:
	\begin{align*}
		\abr{\sum_{i=1}^n(X_i-\mu)} \leq \min\cbr{2\sqrt{B\mu n\ln\frac{2n}{\delta}} + B\ln\frac{2n}{\delta}, 2\sqrt{B\sum_{i=1}^nX_i\ln\frac{2n}{\delta}} + 7B\ln\frac{2n}{\delta}}.
	\end{align*}
\end{lemma}

\begin{lemma}{\citep[Lemma D.4]{cohen2020near} and \citep[Lemma C.2]{cohen2021minimax}}
	\label{lem:e2r}
	Let $\{X_i\}_{i=1}^{\infty}$ be a sequence of random variables w.r.t to the filtration $\{\calF_i\}_{i=0}^{\infty}$ and $X_i\in[0,B]$ almost surely.
	Then with probability at least $1-\delta$, for all $n\geq 1$ simultaneously:
	\begin{align*}
		\sum_{i=1}^n\E[X_i|\calF_{i-1}] &\leq 2\sum_{i=1}^n X_i + 4B\ln\frac{4n}{\delta},\\
		\sum_{i=1}^n X_i &\leq 2\sum_{i=1}^n\E[X_i|\calF_{i-1}] + 8B\ln\frac{4n}{\delta}.
	\end{align*}
\end{lemma}